%% file: paper.tex
\documentclass{article}

\usepackage[T1]{fontenc}

\usepackage{microtype}
\usepackage{graphicx}
\usepackage{subcaption}
\usepackage{booktabs} % for professional tables
\usepackage{enumitem}
\usepackage{multirow}
\usepackage{microtype}
\usepackage{comment}
\usepackage[table]{xcolor}
\definecolor{ourscolor}{RGB}{230,243,255}  % light blue

\usepackage[most]{tcolorbox}

  \newtcolorbox{summarybox}{
    colback=ourscolor,
    colframe=mydarkblue,
    boxrule=1pt,
    arc=3pt,
    left=6pt,
    right=6pt,
    top=6pt,
    bottom=6pt,
  }

\usepackage{hyperref}

\DeclareMathOperator{\diag}{diag}
 \DeclareMathOperator{\softmax}{softmax}

\usepackage[preprint]{icml2026}

\usepackage{amsmath}
\usepackage{amssymb}
\usepackage{mathtools}
\usepackage{amsthm}
\usepackage{subcaption}

\newtheorem{hyp}{Hypothesis}

\usepackage[capitalize,noabbrev]{cleveref}

\theoremstyle{plain}
\newtheorem{theorem}{Theorem}[section]

\newtheorem{lemma}[theorem]{Lemma}

\theoremstyle{definition}

\theoremstyle{remark}
\newtheorem{remark}[theorem]{Remark}

\usepackage[textsize=tiny]{todonotes}

\icmltitlerunning{Too Sharp, Too Sure: When Calibration Follows Curvature}

\begin{document}

\twocolumn[
  % \icmltitle{When Confidence follows Curvature: A Story of Separability}
  \icmltitle{Too Sharp, Too Sure: When Calibration Follows Curvature}

  % It is OKAY to include author information, even for blind submissions: the
  % style file will automatically remove it for you unless you've provided
  % the [accepted] option to the icml2026 package.
  \icmlsetsymbol{equal}{*}

    \begin{icmlauthorlist}
    \icmlauthor{Alessandro Morosini}{mit}
    \icmlauthor{Matea Gjika}{mit}
    \icmlauthor{Tomaso Poggio}{mit}
    \icmlauthor{Pierfrancesco Beneventano}{mit}
    \end{icmlauthorlist}
    
    \icmlaffiliation{mit}{Massachusetts Institute of Technology, Cambridge, MA, USA}
    \icmlcorrespondingauthor{Alessandro Morosini}{morosini@mit.edu}
    \icmlcorrespondingauthor{Pierfrancesco Beneventano}{pierb@mit.edu}

  % You may provide any keywords that you find helpful for describing your paper
  \icmlkeywords{Calibration, copy them}

  \vskip 0.3in
]

% This command actually creates the footnote in the first column listing the
% affiliations and the copyright notice.
% If you have no special notice, KEEP empty braces:
\printAffiliationsAndNotice{}  % no special notice
% Or, if applicable, use the standard equal contribution text:
% \printAffiliationsAndNotice{\icmlEqualContribution}

\begin{abstract}
    Modern neural networks can achieve high accuracy while remaining poorly calibrated, producing confidence estimates that do not match empirical correctness. Yet calibration is often treated as a post-hoc attribute. We take a different perspective: we study calibration as a \emph{training-time} phenomenon on small vision tasks, and ask whether calibrated solutions can be obtained reliably by intervening on the training procedure. We identify a tight coupling between calibration, curvature, and margins during training of deep networks under multiple gradient-based methods. Empirically, Expected Calibration Error (ECE) closely tracks curvature-based sharpness throughout optimization. Mathematically, we show that both ECE and Gauss--Newton curvature are controlled, up to problem-specific constants, by the same margin-dependent exponential tail functional along the trajectory. Guided by this mechanism, we introduce a margin-aware training objective that explicitly targets robust-margin tails and local smoothness, yielding improved out-of-sample calibration across optimizers without sacrificing accuracy.
\end{abstract}

% \onecolumn
% \input{rebuttal}
% \twocolumn

\begin{figure}[t]
  \centering
  \includegraphics[width=\columnwidth]{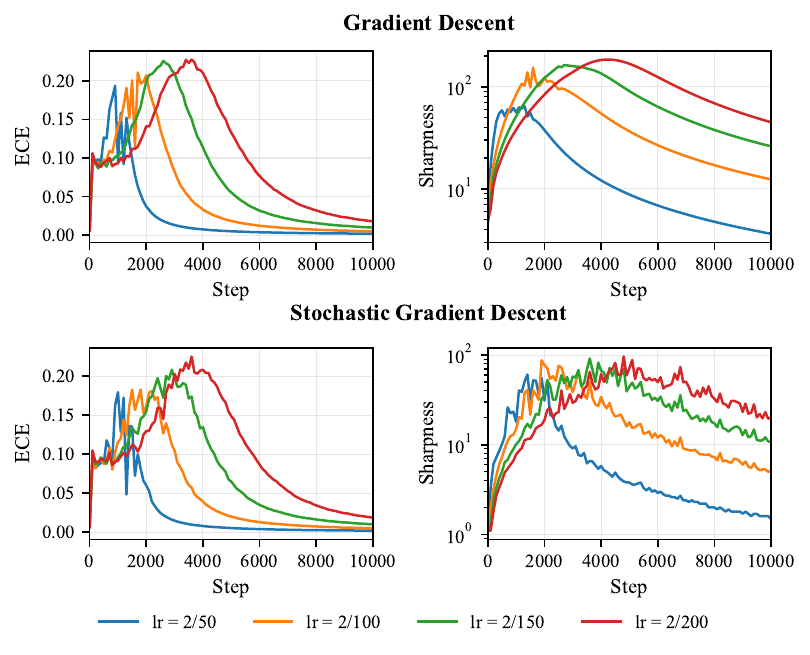}
  \caption{Training dynamics for Gradient Descent and Stochastic Gradient Descent across learning rates on CIFAR-10. \textbf{Expected Calibration Error closely tracks sharpness throughout training}: both rise as the model enters the edge of stability regime, peak around the same time, and decay together as training progresses.}
  \label{fig:pizza}
\end{figure}

\input{sections/intro}

\input{sections/related_work}
\input{sections/correlation}  
\input{sections/sam_vs_muon}
\input{sections/math}

\input{sections/calmo}
\input{sections/conclusion}

\newpage
\bibliography{references}
\bibliographystyle{icml2026}

%%%%%%%%%%%%%%%%%%%%%%%%%%%%%%%%%%%%%%%%%%%%%%%%%%%%%%%%%%%%%%%%%%%%%%%%%%%%%%%
% APPENDIX
%%%%%%%%%%%%%%%%%%%%%%%%%%%%%%%%%%%%%%%%%%%%%%%%%%%%%%%%%%%%%%%%%%%%%%%%%%%%%%%
\newpage
\appendix
\onecolumn

% \tableofcontents
% \pier{fix this if ICML allows it.}

\newpage

\renewcommand{\thefigure}{A\arabic{figure}}
\renewcommand{\thetable}{A\arabic{table}}
\setcounter{figure}{0}
\setcounter{table}{0}

\input{sections/appendix/further_related_work}
\input{sections/appendix/logistic_proof}

\input{sections/appendix/all_plots}
\input{sections/appendix/calmo_vs_ce}
\input{sections/appendix/mse_extension} 
\newpage

\end{document}

%% file: sections/intro.tex
\section{Introduction}

% Calibration matters in deployment; deep nets are miscalibrated.
% \paragraph{Motivations.} 
Neural networks are now routinely used in settings where a model's stated uncertainty matters as much as its accuracy, for example, in risk-sensitive domains such as healthcare or autonomous driving. In these contexts, we would like predicted probabilities to reflect empirical correctness: among predictions made with confidence $p$, approximately a fraction $p$ should be correct. However, modern deep networks are often \emph{miscalibrated}, frequently exhibiting overconfidence even when wrong \citep{guo2017calibration}.

% Most calibration work is post-hoc; what happens during training?
A widely adopted response to overconfidence in neural networks is \emph{post-hoc} calibration: models are trained for accuracy and their predicted probabilities are adjusted afterward. While effective in many regimes, this framing treats calibration as a post-training concern, rather than a property that \emph{emerges during training}. 
Recent work, however, suggests that calibration may be influenced by training dynamics. For example, Sharpness-Aware Minimization (SAM), which biases optimization toward flatter regions of the loss landscape, has been observed to generally reduce overconfidence \cite{tan2025samcalibration}. Since sharpness is a geometric property shaped along the optimization trajectory, these results hint at a link between calibration and loss geometry during training. Yet this connection remains poorly understood, motivating the following question:
\begin{center}
\emph{How does calibration evolve throughout optimization, \\ 
and what aspects of the training govern it?}
\end{center}
Importantly, answering this question is non-trivial. Calibration is defined in terms of the model’s \emph{predictive confidence distribution}, whereas most training-time analyses characterize optimization through \emph{loss-landscape geometry}, using quantities such as curvature or sharpness to reason about stability and generalization. Bridging these viewpoints is challenging, particularly early in training when predictions are still rapidly evolving. While several recent studies have examined relationships between sharpness, flat minima, and calibration at convergence, the resulting picture is mixed: curvature proxies do not reliably predict calibration across architectures, regularization schemes, or optimizers \citep{mason}. % However, examining geometric properties and calibration at the final solution obscures training-time dynamics that can be exploited to manage miscalibration. 
Crucially, these analyses focus on
converged solutions. By contrast, we study how calibration and loss geometry co-evolve during training—a perspective that not only clarifies their relationship but also reveals dynamics that can be exploited to improve calibration.
After formalizing a training-time connection between calibration and loss geometry, we turn to a prescriptive goal:
\begin{center}
\emph{Can we intervene on the training procedure to reliably obtain calibrated solutions?}
\end{center}

\paragraph{Contributions.}
We study calibration \emph{during} optimization by jointly tracking calibration metrics, such as Expected Calibration Error (ECE), and curvature-based sharpness proxies, such as Gauss--Newton (GN) sharpness, along the training trajectory (Section \ref{sec:sharpness-calibration}). We conduct this analysis throughout training rather than only at convergence. We observe across multiple gradient-based optimization methods that
% \textit{we establish that this co-evolution persists across different gradient-based methods, even though ECE and Curvature are not causally linked suggesting our main message that:}
\begin{summarybox}
\textbf{Contribution 1.}
Calibration error and curvature-based measures exhibit a strong and consistent temporal correlation throughout training.
\end{summarybox}
%
% Explanation: a margin-based unifier links calibration and curvature and explains train--test divergence.
%
% Further relevant scientific contributions are:
% \begin{itemize}[leftmargin=1em,itemsep=0.15em,topsep=0.15em,parsep=0pt]
    % \item Muon leads to less calibrated solution in our framework (see Section \ref{?}). \item \pier{others Ale? What's on SAM?}
% \end{itemize}
% \ale{I would move these observations to the discussion.}
%
% Causal test: compare flat-minima objectives vs directional curvature suppression along the trajectory.
% \paragraph{What we do: Engineering.}
% We leverage the coupling between calibration and curvature to probe whether sharpness \emph{causally} influences calibration. 
Next, we probe whether the coupling between calibration and curvature is causal. 
We compare optimizers designed to minimize sharpness (i.e., favoring flatter minima) with methods that instead suppress steep descent directions along the trajectory. Despite both affecting curvature, we find that
\begin{summarybox}
    \textbf{Contribution 2:} 
    Directional interventions yield consistently better in-sample calibration than flat-minima methods. 
\end{summarybox}
This clarifies the distinction between ``being flat'' and ``training in stable directions'' and its relation to confidence.

We then provide a unifying explanation through the lens of the \emph{margin} in Section \ref{sec:math}. Intuitively, both confidence and curvature are shaped by how strongly the model separates the correct class from its nearest competitor. We formalize this connection by showing mathematically that 
\begin{summarybox}
    \textbf{Contribution 3:} 
    A single margin-based functional controls both calibration error and Gauss--Newton sharpness, up to problem-dependent constants. 
\end{summarybox}
This perspective also clarifies an often-observed phenomenon: training and test calibration can diverge even when accuracy improves \citep{carrell2022calibration, wu2025benefits}. Once most examples achieve large positive margins, a relatively small set of near-boundary or negative-margin points can dominate the margin functional, making calibration highly sensitive to how optimization shapes this tail.

Based on the mathematical connection we establish between margins, curvature, and calibration, we design a new robust-margin-aware loss. This yields a principled training-time handle on calibration, which we confirm in Section~\ref{sec:CalMO}:
\begin{summarybox}
    \textbf{Contribution 4:} 
    We propose CalMO (CALibration with Margin Objective), a training objective that yields better-calibrated models without sacrificing accuracy.
\end{summarybox}
% With the limitation is that leads to better calibration in the long run, not in the early part of training.
Our margin-based view also yields concrete diagnostics for optimizer behavior. In particular, we observe that \textbf{Muon induces unusually large training margins, leading to near-zero training ECE, but severe test-time overconfidence}. Consistent with this finding, Muon is also the optimizer that benefits most from CalMO, as robust-margin control directly targets this failure mode.

%% file: sections/related_work.tex
\section{Preliminaries and Related Work}
\paragraph{Calibration.}
A model is calibrated if its predicted confidence values reflect empirical correctness frequencies: among all predictions made with confidence $p$, approximately a fraction $p$ should be correct \cite{Niculescu-Mizil2005, degroot1983forecasters}. Since calibration is a distributional property, it is typically assessed empirically from finite samples. The most widely used measure is ECE~\cite{Naeini2015}, which compares accuracy and confidence after binning predictions by confidence.
Predictions are grouped into $M$ bins according to their confidence
(i.e., the maximum predicted class probability); we denote by $B_m$
the set of predictions whose confidence falls in bin $m$.
For each bin, we compute the empirical accuracy $\mathrm{acc}(B_m)$ and the average predicted confidence $\mathrm{conf}(B_m)$. ECE is then defined as
\begin{equation}
\mathrm{ECE}
=
\sum_{m=1}^M \frac{|B_m|}{n}
\bigl|\mathrm{acc}(B_m) - \mathrm{conf}(B_m)\bigr|.
\end{equation}
Calibration is also evaluated through other summary statistics such as the Maximum Calibration Error \cite{Naeini2015}, reliability diagrams \cite{guo2017calibration}, kernel-based metrics
such as the Kernel Calibration Error (KCE) \cite{kumar2018trainable}, or similar related metrics. Further details, including the multiclass extension, are provided in Appendix~\ref{app:ece_defs}.

\paragraph{Mitigating miscalibration.}
Existing approaches to mitigating calibration error fall into two broad categories: post-hoc methods and intrinsic training-time methods. 
Post-hoc methods modify model outputs without adjusting the parameters, and include techniques such as Platt scaling \cite{platt2000probabilistic} and temperature scaling \cite{guo2017calibration}. Intrinsic methods incorporate calibration objectives directly into training by using regularization to penalize overconfident predictions~\cite{pereyra2017regularizing}, by incorporating differential proxies for calibration into the loss~\cite{kumar2018trainable,bohdal2023metacalibration}, or by applying label smoothing~\cite{muller2019label}. A more detailed overview of these methods is provided in Appendix~\ref{appendix:further_related_work}.

An important perspective links calibration to adversarial robustness. Points with small robust margin have been shown to be more likely miscalibrated; motivated by this, R-AdaLS \citep{qin2021improving} bins data points by robust margin and applies stronger label smoothing to low-margin samples. Moreover, state-of-the-art calibration losses can be unified as penalties on logit distances \citep{liu2022devil}. These results suggest that calibration errors are tied to local margin geometry, rather than solely to global confidence statistics. 

\paragraph{Curvature along the trajectory.}
A complementary line of work treats curvature not as a static attribute of the final solution, but as a dynamical quantity that governs optimization stability throughout training. Early work established a link between ``wide valleys'' and generalization \citep{hochreiter1997flatminima}, a picture later reinforced by the sharp-minima account of the large-batch generalization gap \citep{keskar2016largebatch}, motivating the study of Hessian-based sharpness along the training trajectory.
% Along this line of work, \citet{jastrzebski_catastrophic_2021} showed that SGD implicitly regularizes the Gauss-Newton matrix we track. 
More recent analyses emphasize that curvature generally increases as gradient methods approach an \emph{edge-of-stability} (EoS) regime, where the top curvature direction becomes commensurate with the inverse step size and the dynamics become oscillatory or unstable \citep{xing_walk_2018, jastrzebski_three_2018, jastrzebski_relation_2019, cohen2021edge, cohen2022adaptiveEoS, andreyev_edge_2024}, before often decreasing toward the end of training. 

In parallel, neighborhood-based objectives reshape the training trajectory by explicitly penalizing worst-case loss increases under small weight perturbations (e.g., SAM), thereby suppressing sensitivity to sharp directions \citep{foret2021sam, zhou2024sharpness}; such procedures have also been observed to improve confidence estimates under cross-entropy \citep{tan2025samcalibration}. At the same time, evidence suggests that the sharpness--calibration relationship can be fragile across architectures and regularization schemes \citep{mason}, pointing to the importance of \emph{how} curvature directions are traversed, not only where optimization converges. Our trajectory-level study aligns with this viewpoint by jointly tracking calibration and curvature throughout training and by contrasting convergence-to-flatness with explicit suppression of unstable high-curvature directions.

%% file: sections/correlation.tex
\section{The Coupling Between Calibration and Sharpness}
\label{sec:sharpness-calibration}

We track sharpness and ECE throughout training to study how loss landscape geometry relates to calibration. Following \citet{cohen2021edge}, we train an MLP (2 hidden layers, 200 units, tanh activation) on CIFAR-10 under cross-entropy (CE) loss. This small-scale setup enables frequent computation of GN sharpness, which serves as a proxy for the top Hessian eigenvalue $\lambda_{\max}$.
% For out-of-sample evaluation, we use ResNet-18
Models are trained using gradient descent (GD), stochastic gradient descent (SGD), 
% Adam \cite{adam2014method}, 
AdamW \cite{adam2014method,loshchilov2017decoupled}, Muon \cite{jordan2024muon}, and SAM \cite{foret2021sam}. We monitor GN sharpness and batch sharpness \citep{andreyev_edge_2024} as proxies for loss-landscape geometry, alongside ECE, KCE, loss, and accuracy. % We report results from cross-entropy (CE) in Table~\ref{tab:ece_sharpness_corr}. 

 % \begin{table}[t]
 %    \centering
 %    \footnotesize
 %    \setlength{\tabcolsep}{3pt}
 %    \begin{tabular}{lcccccc}
 %    \toprule
 %    & GD & SGD & Adam & AdamW & Muon & SAM \\
 %    \midrule
 %    Train & $.83 {\scriptstyle \pm .08}$ & $.84 {\scriptstyle \pm .07}$ & $.70 {\scriptstyle \pm .08}$ & $.72 {\scriptstyle \pm .07}$ & $.63 {\scriptstyle \pm .26}$ & $.92 {\scriptstyle \pm .04}$ \\
 %    Test & $.96 {\scriptstyle \pm .02}$ & $.97 {\scriptstyle \pm .01}$ & $.17 {\scriptstyle \pm .17}$ & $.15 {\scriptstyle \pm .15}$ & $-.10 {\scriptstyle \pm .28}$ & $.98 {\scriptstyle \pm .01}$ \\
 %    \bottomrule
 %    \end{tabular}
 %    \caption{Pearson correlation between ECE and sharpness, mean $\pm$ std over 4 learning rates.}
 %    \label{tab:ece_sharpness_corr}
 %  \end{table}
\subsection{Calibration Temporally Correlates with Sharpness}

  \begin{table}[t]
      \centering
      \footnotesize
      \setlength{\tabcolsep}{3pt}
      \begin{tabular}{llccccc}
      \toprule
      & & GD & SGD & AdamW & Muon & SAM \\
      \midrule
      \multirow{2}{*}{Train}
      & ECE & $.83 {\scriptstyle \pm .08}$ & $.84 {\scriptstyle \pm .07}$ & $.72 {\scriptstyle \pm .07}$ &
   $.63 {\scriptstyle \pm .26}$ & $.92 {\scriptstyle \pm .04}$ \\
      & KCE & $.83 {\scriptstyle \pm .08}$ & $.84 {\scriptstyle \pm .07}$ & $.70 {\scriptstyle \pm .08}$ &
   $.61 {\scriptstyle \pm .28}$ & $.91 {\scriptstyle \pm .04}$ \\
      \midrule
      \multirow{2}{*}{Test}
      & ECE & $.96 {\scriptstyle \pm .02}$ & $.97 {\scriptstyle \pm .01}$ & $.15 {\scriptstyle \pm .15}$ &
   $-.10 {\scriptstyle \pm .28}$ & $.98 {\scriptstyle \pm .01}$ \\
      & KCE & $.97 {\scriptstyle \pm .01}$ & $.97 {\scriptstyle \pm .01}$ & $.21 {\scriptstyle \pm .22}$ &
   $-.16 {\scriptstyle \pm .31}$ & $.98 {\scriptstyle \pm .01}$ \\
      \bottomrule
      \end{tabular}
      \vspace{0.2cm}
      \caption{Pearson correlation between calibration metrics (ECE and KCE) and GN sharpness, mean $\pm$ std over 4
  learning rates. }
      \label{tab:ece_sharpness_corr}
      \vspace{-0.6cm}
  \end{table}

Figure~\ref{fig:pizza} shows training dynamics for GD and SGD. Across all CE experiments\footnote{We observe a similar temporal correlation on models trained with mean-squared error (MSE) loss.
There, however, both GN sharpness and ECE increase and then plateau at high values. This behavior reflects the fact that MSE is not a proper scoring rule and induces systematic underconfidence; we therefore focus on CE in the main text and defer a detailed discussion of MSE to Appendix~\ref{sec:mse_ece}.
}, 
training ECE and GN sharpness follow the same trajectory: both quantities are small at initialization, increase as training enters an EoS regime, and decrease again later in training. This holds across optimizers and learning rates
(see Figure~\ref{fig:sam_vs_muon} and Appendix~\ref{sec:correlation_plots} for additional results). Similar observations extend to CIFAR-100 (Appendix~\ref{sec:correlation_plots}). Table~\ref{tab:ece_sharpness_corr} quantifies this effect, showing strong Pearson correlations between ECE and (batch) sharpness throughout training; KCE closely matches ECE across all settings, confirming the coupling is not a binning artifact.

% \pier{Do claim in the caption of Table 1 or in the appendix related how you picked the hyperparameters. In the appendix show the plots and table when ablating over $\eta$ and batch size, here pick the best in terms of performance}

% To quantify this empirical observation, we compute the correlation between training ECE and sharpness over the course of training for each setting. A summary is reported in Table~\ref{tab:ece_sharpness_corr}% UPDATE THE CORRELATIONS, SOME ARE PLACEHOLDERS

% As we explain later in Section \ref{section:math}, we expect calibration and curvature to converge to zero at the same time when all training data are well classified. It is thus not surprising that towards the very end of training ECE and Curvature dynamics couple. This pronounced, time-resolved co-evolution of sharpness and ECE observed well before convergence, across optimizers and loss functions is instead surprising and had not been observed nor explained before to the knowledge of the authors. In particular, the strong correlations in Table~\ref{tab:ece_sharpness_corr} arise in regimes where the model is far from interpolation and calibration is nontrivial.

% These observations suggest that calibration is governed not merely by convergence to a flat minimum, but by the training trajectory itself. Specifically, models that traverse training regions of lower sharpness tend to exhibit improved calibration throughout training, not just asymptotically. In the next section, we provide a mathematical perspective on this phenomenon. 

 As we show in Section~\ref{sec:math}, both calibration and sharpness converge to zero once all training points are correctly classified, so their coupling at the end of training is expected. What is surprising is the strong correlation during training, well before convergence, when the model is far from interpolation and calibration is nontrivial. To our knowledge, this has not been observed or explained before. These results suggest that calibration does not depend on training-metrics at convergence, but on the trajectory itself: models that stay in lower-sharpness regions remain better calibrated throughout training, not just asymptotically. We formalize this connection in Section~\ref{sec:math}. 

%% file: sections/sam_vs_muon.tex
\subsection{Converging to Flat Minima or Following Flat Directions?}
\label{section:interventions}
 
 The strong temporal correlation between sharpness and calibration established in the previous section raises a causal question: \emph{does calibration improve because optimization converges to flatter minima, or because training dynamics suppress movement along high-curvature directions?} To disentangle these mechanisms, we formulate two competing hypotheses, and empirically find that suppressing directions of steep descent leads to improved in-sample calibration.

\begin{hyp}[Flat Minima for Calibration] \label{hyp:flat_minima}
Training procedures that bias optimization toward flat minima lead to lower in-sample calibration error.
\end{hyp}

\begin{hyp}[Directional Flatness for Calibration] \label{hyp:sharp_directions}
Training procedures that suppress updates along directions of steep curvature lead to lower in-sample calibration error, even if the final solution is not globally flat.
\end{hyp}

We test Hypothesis~\ref{hyp:flat_minima} using SAM \cite{zhou2024sharpness}, which explicitly penalizes worst-case loss perturbations within a local neighborhood, and is known to bias optimization toward flatter minima. To test Hypothesis~\ref{hyp:sharp_directions}, we use optimizers that directly suppress high-curvature directions during training. Muon  \cite{jordan2024muon} rescales gradient components to equalize their magnitudes, effectively clamping updates along sharp directions while amplifying flatter ones. BulkSGD \cite{song2024does} achieves a more extreme intervention by projecting gradients onto the subspace orthogonal to the top Hessian eigenvectors, thereby removing the steepest descent directions entirely. A more detailed analysis of the optimizers and their benefits in this experimental setting can be found in Appendix \ref{app:sam_muon_bulk}. 

Figures~\ref{fig:sam_vs_muon} and \ref{fig:bulk-sgd} show the training dynamics. Although SAM consistently maintains lower sharpness than GD and SGD, its calibration trajectory closely mirrors that of standard training, with a comparable peak ECE and slower convergence. In contrast, both Muon and BulkSGD achieve substantially lower peak calibration error and faster ECE decay, despite exhibiting markedly different sharpness profiles.

\begin{figure}[t]
  \centering
  \includegraphics[width=\columnwidth]{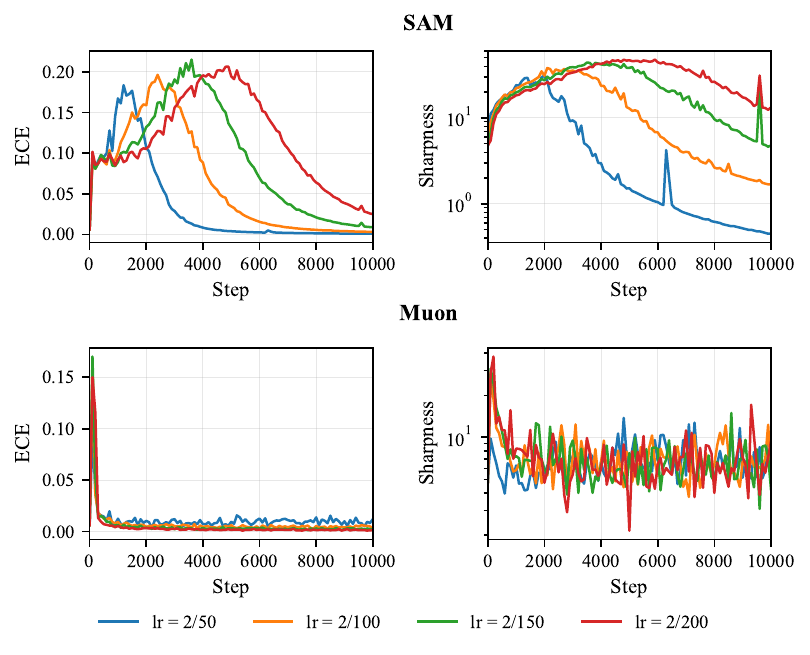}
  \caption{Training dynamics for SAM and Muon across learning rates on CIFAR-10.}
  \label{fig:sam_vs_muon}
\end{figure}

\begin{figure*}
\centering

\begin{subfigure}{0.48\linewidth}
    \centering
    \includegraphics[width=\linewidth]{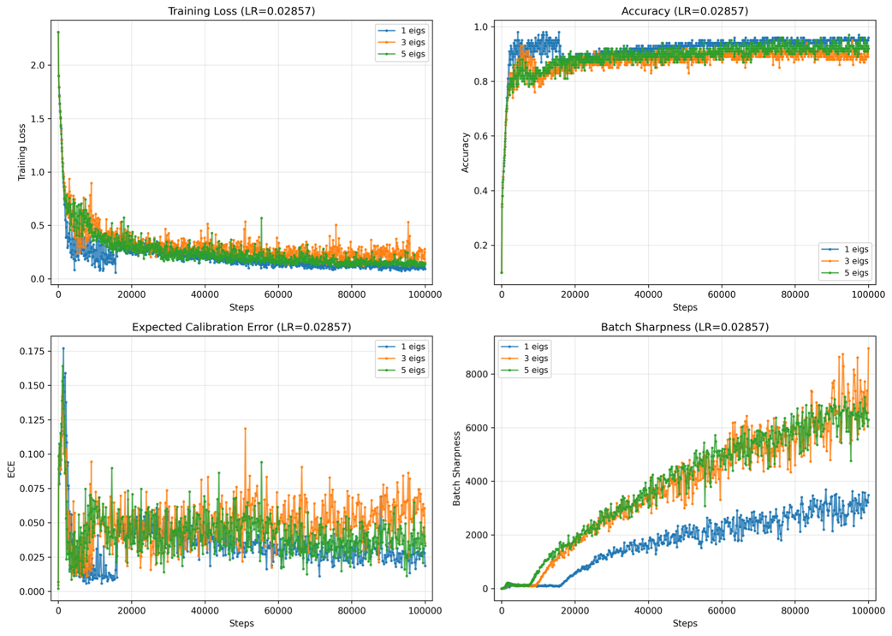}
    \caption{BulkSGD with learning rate $\frac{2}{70} $}
    \label{fig:bulk_sgd_2_70}
\end{subfigure}
\hfill
\begin{subfigure}{0.48\linewidth}
    \centering
    \includegraphics[width=\linewidth]{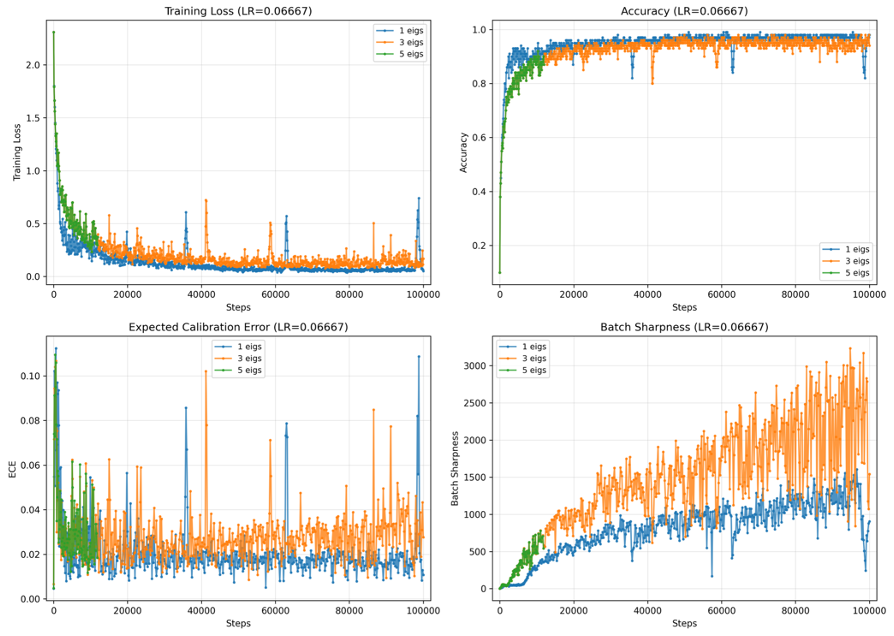}
    \caption{BulkSGD with learning rate $\frac{2}{30} $}
    \label{fig:bulk_sgd_2_30}
\end{subfigure}

\caption{Training dynamics for BulkSGD across different learning rates and number of projected-out gradients on CIFAR10.}
\label{fig:bulk-sgd}
\end{figure*}

Notably, Muon maintains low calibration error while operating in regimes that are not globally flat, and BulkSGD improves calibration even in the presence of pronounced instability. 
This suggests that calibration is sensitive to how optimization traverses sharp directions, rather than to the absolute flatness of the loss landscape. 
However, BulkSGD induces oscillatory dynamics and a sharpness divergence when too many dominant directions are projected out, making it impractical as a standalone optimizer.
Together, these results support Hypothesis~\ref{hyp:sharp_directions} over Hypothesis~\ref{hyp:flat_minima}: suppressing updates along high-curvature directions during training leads to improved in-sample calibration, whereas convergence to flat minima alone does not.  

\subsection{Out-of-Sample Behavior}
The sharpness--calibration coupling is less consistent out of sample (Table~\ref{tab:ece_sharpness_corr}). Across optimizers, test ECE does not consistently decrease alongside training ECE---in some cases it worsens as training progresses, even after sharpness and training calibration improve (Appendix~\ref{sec:correlation_plots}). 
Muon is an extreme example: training ECE drops to near zero while test ECE remains high, yielding a negative test correlation despite strong in-sample alignment. This reflects the calibration generalization gap in overparameterized models \citep{carrell2022calibration, berta2025rethinking, wu2025benefits}: a model that fits training data well can become overconfident on misclassified test examples, causing test ECE to increase and decouple from sharpness. In Muon's case, the near-zero training ECE is consistent with the large training margins it induces (Section~\ref{sec:math}): once these are extreme, the model becomes overconfident on test examples near the decision boundary, precisely where the margin functional is most sensitive.

Together with the findings from Section~\ref{section:interventions}, these results point to an important distinction:
on the one hand, directional interventions yield better in-sample calibration than flat-minima methods, suggesting that \emph{how} optimization traverses curvature matters more than where it converges; on the other hand, in-sample improvements do not automatically transfer to test data, pointing to a fundamental train--test gap. In the following section, we formalize this train--test gap and use it to design a training-time intervention that improves out-of-sample calibration.

%% file: sections/math.tex
%%%%%%%%%%%%%%%%%%%%%%%%%%%%%%%%%%%%%%%%%%%%%%%%%%%%%%%%%%%%%%%%%%%%%%%%%%%%%%
% Section 4
%%%%%%%%%%%%%%%%%%%%%%%%%%%%%%%%%%%%%%%%%%%%%%%%%%%%%%%%%%%%%%%%%%%%%%%%%%%%%%
\section{Curvature and Calibration in the Separable and Non-separable Regimes}
\label{sec:math}

In this section we explain the temporal alignment between calibration error and curvature observed in Section~\ref{sec:sharpness-calibration}
through a common underlying quantity: the (robust) \emph{true logit margin}. Our central claim is that, across training,
both ECE and Gauss--Newton sharpness respond to the evolution of the same margin-dependent
tail functional. This provides a concrete mechanism linking predictive confidence to loss-landscape geometry.

The analysis naturally separates into two regimes. Early in training, the data behave as
\emph{overlap-dominated}: a nontrivial fraction of examples exhibit small or negative true margins, and no uniform separability holds. 
The same regime persists at test time whenever accuracy is below $1$, since any misclassified example has $m_\theta(x,y)<0$ by definition---and, unlike on training, no cross-entropy mechanism pushes those margins to grow. 
In either case, neither calibration error nor curvature is forced to be small, and both can be dominated by a few hard or ambiguous points.
This perspective aligns with observations that loss and curvature are controlled by a small set of strongly opposing examples
\citep{rosenfeld2023outliers}. Later in training on the \emph{training set}, models trained with cross-entropy typically enter an
\emph{interpolating} regime in which all true margins become strictly positive. In this regime, calibration error and curvature
become tightly coupled: once the margin tail contracts, both quantities are forced to decrease together.

Together, these results provide a mechanism-level explanation for the observed co-evolution of calibration and curvature during
training, and clarify why training and test calibration can diverge even as accuracy improves. We conclude the section by
connecting these theoretical regimes to the empirical training dynamics observed in Section~\ref{sec:sharpness-calibration}. 

\paragraph{Setup and Notation.}
Let $(X,Y)\sim\pi$ with $Y\in\{1,\dots,K\}$. A model $\theta\in\mathbb{R}^d$ produces logits $z_\theta(x)\in\mathbb{R}^K$ and probabilities
$p_\theta(x)=\softmax(z_\theta(x))$. Let $\hat y(x)=\arg\max_k z_\theta(x)_k$ (deterministic tie-break) and confidence $\hat P(x)=\max_k p_\theta(x)_k$.
Define the \emph{true (logit) margin}
\[
m_\theta(x,y):=z_\theta(x)_y-\max_{j\neq y}z_\theta(x)_j,
\]
and the \emph{robust true margin} at radius $\varepsilon>0$,
\[
m_{\varepsilon,\theta}(x,y):=\inf_{\|\delta\|\le\varepsilon} m_\theta(x+\delta,y).
\]
Let $\mathrm{ECE}_M$ denote the population $\pi$/sample $\mathcal{D}$ binned ECE computed by binning $\hat P(X)$ into $M$ bins.
Let $J_\theta(x):=\partial z_\theta(x)/\partial\theta\in\mathbb{R}^{K\times d}$ and, for cross-entropy, $H_z(p):=\diag(p)-pp^\top$.
Define the population Gauss--Newton matrix and its curvature proxy
\[
H_{\mathrm{GN}}(\theta;\pi):=\mathbb{E}_{\pi}\!\big[J_\theta(X)^\top H_z(p_\theta(X))J_\theta(X)\big],
\]
\[
\lambda_{\max}:=\lambda_{\max}\!\big(H_{\mathrm{GN}}(\theta;\pi)\big).
\]

\subsection{Regime I: overlap-dominated (non-separable) behavior}
\label{sec:regime_overlap}

In this subsection all the quantities (ECE, GN matrix, robust margin, robust margin moment) are considered at a population level. See details in Appendix \ref{app:math_proofs}. Define the robust exponential margin moment $ Q(\theta):=\mathbb{E}_{(X,Y)\sim\pi}\!\big[e^{-m_{\varepsilon,\theta}(X,Y)}\big]$.

\begin{theorem}[Overlap regime: robust-margin upper bounds]
\label{thm:overlap_bounds}
For any $\theta$ and any distribution $\pi$,
\[
\mathrm{ECE}_M\ \le\ (K-1)\,Q(\theta).
\]
If additionally $\|J_\theta(X)\|_{\mathrm{op}}\le C_J$ holds $\pi$-a.s., then
\[
\lambda_{\max}\ \le\ 2C_J^2\,(K-1)\,Q(\theta).
\]
\end{theorem}
Proof is in Appendix \ref{app:proof_overlap}.

\paragraph{Interpretation (two bottlenecks).}
Theorem~\ref{thm:overlap_bounds} exposes two multiplicative controls:
\begin{itemize}[leftmargin=1em,itemsep=0.15em,topsep=0.15em,parsep=0pt]
\item a \emph{probability bottleneck} $Q(\theta)$, dominated by the tail of small/negative robust margins,
\item a \emph{geometry bottleneck} $C_J^2$ (how parameter perturbations move logits).
\end{itemize}
In overlap-dominated regimes (early training or test-time), a persistent set of small robust margins can keep $Q(\theta)$ bounded away from $0$,
so these bounds need not certify vanishing calibration error or curvature even if loss continues to decrease.

%%%%%%%%%%%%%%%%%%%%%%%%%%%%%%%%%%%%%%%%%%%%%%%%%%%%%%%%%%%%%%%%%%%%%%%%%%%%%%
% Regime II: interpolating training set bounds (two-sided + coupling)
%%%%%%%%%%%%%%%%%%%%%%%%%%%%%%%%%%%%%%%%%%%%%%%%%%%%%%%%%%%%%%%%%%%%%%%%%%%%%%
\subsection{Regime II: Interpolating (separable) behavior on the training set}
\label{sec:regime_separable}

In this subsection all the quantities (ECE, GN matrix, robust margin, robust margin moment) are considered at a finite-sample level. See details in Appendix \ref{app:math_proofs}. Let $\gamma(\theta;\mathcal{D})\;:=\;\min_{1\le i\le n} m_\theta(x_i,y_i)$ and the empirical exponential margin moment $Q_{\mathcal{D}}(\theta)
:=\frac{1}{n}\sum_{i=1}^n e^{-m_{\theta_t}(x_i,y_i)}$.

% In overparameterized networks trained with CE, once the training data are interpolated, gradient descent drives
% $p_\theta(x_i)_{y_i}\to 1$, so the logit-level curvature $\diag(p)-pp^\top$ collapses and training ECE vanishes.

\begin{theorem}[Interpolating regime: two-sided ECE--margin control and coupling to $\lambda_{\max}$]
\label{thm:separable_bounds}
Assume $\gamma(\theta;\mathcal{D})>0$ (all training points correctly classified with strictly positive true margin).
Then
\begin{equation*}
\begin{split}
    \frac{1}{K}\,Q_{\mathcal{D}}(\theta)
    \ &\le\
    \mathrm{ECE}_M
    \ \le\
    (K-1)\,Q_{\mathcal{D}}(\theta)
    \\& \le\
    (K-1)\,e^{-\gamma(\theta;\mathcal{D})}.
\end{split}
\end{equation*}
If additionally $\max_{i\in[n]}\|J_\theta(x_i)\|_{\mathrm{op}}\le C_J$, then
\[
\lambda_{\max}
\ \le\
2C_J^2\,(K-1)\,Q_{\mathcal{D}}(\theta)
\ \le\
2C_J^2\,K(K-1)\,\mathrm{ECE}_M,
\]
equivalently $\ \mathrm{ECE}_M
\ge \lambda_{\max}/(2C_J^2K(K-1))$.
\end{theorem}
Proof is in Appendix \ref{app:proof_separable}.

\paragraph{Interpretation.}
In the interpolating regime, $\mathrm{ECE}_M$ is equivalent up to constants to the exponential margin moment $Q_{\mathcal{D}}(\theta)$,
and $\lambda_{\max}$ is controlled by the \emph{same} moment (under bounded Jacobians).
This implies that once the training set is correctly classified, GN sharpness cannot be large without in-sample ECE also being large.
Moreover, in this regime, empirical binning becomes immaterial: 
$\mathrm{ECE}_M$ reduces to the mean misconfidence (formalized in Appendix~\ref{app:math_proofs}).

%%%%%%%%%%%%%%%%%%%%%%%%%%%%%%%%%%%%%%%%%%%%%%%%%%%%%%%%%%%%%%%%%%%%%%%%%%%%%%
% Discussion: mapping to experiments (train vs test and the "two-regime proxy")
%%%%%%%%%%%%%%%%%%%%%%%%%%%%%%%%%%%%%%%%%%%%%%%%%%%%%%%%%%%%%%%%%%%%%%%%%%%%%%
\subsection{Discussion: how the two-regime proxy matches the observed train/test split}
\label{sec:proxy_matches_experiments}

Section~\ref{sec:sharpness-calibration} shows strong co-evolution of training ECE and sharpness, well before convergence.
Asymptotic interpolation alone is insufficient to explain this observation.
Theorem~\ref{thm:overlap_bounds}--\ref{thm:separable_bounds} provide a mechanism-level lens: throughout training,
both quantities respond to the evolution of the same margin-dependent tail functional, and in the interpolating regime
this coupling becomes two-sided.

A remaining question is: why test ECE can increase while the logged (predicted) margin increases.
Two caveats explain this:
\begin{itemize}[leftmargin=1em,itemsep=0.15em,topsep=0.15em,parsep=0pt]
\item The bounds depend on the \emph{true} margin $m_\theta(x,y)$, not the predicted margin $z_{\hat y}-z_{(2)}$.
A model can become \emph{more confidently wrong} on a subset of test points: predicted margins increase, accuracy plateaus, and test ECE increases.
\item The sharpness control involves a geometry term (Jacobians). Even if $H_z(p)$ contracts as predictions become more one-hot,
large Jacobian norms (or failure of uniform Jacobian control) can keep curvature proxies large.
\end{itemize}

%% file: sections/calmo.tex
\section{From Margin Theory to Calibrated Training}
  \label{sec:CalMO}

Section~\ref{sec:math} establishes that calibration error is controlled by the margin
functional $Q(\theta) = \mathbb{E}[e^{-m_{\varepsilon,\theta}(X,Y)}]$, with the bound
$\mathrm{ECE} \le (K-1)\,Q(\theta)$ holding on any distribution. This motivates directly
targeting $Q(\theta)$ during training. We propose an objective that enforces robust margins.

% Section~\ref{sec:math} shows that both ECE and GN sharpness are controlled by the margin functional $Q(\theta) = \mathbb{E}[e^{-m_{\varepsilon,\theta}(X,Y)}]$. Under cross-entropy, training naturally drives $Q_{\mathrm{train}} \to 0$ as margins grow, yet test ECE remains high. The issue is not the absence of large margins, but their fragility: clean margins at training points need not survive small perturbations, and fragile margins provide no guarantee for nearby test points. This motivates an objective that explicitly targets \emph{robust} margins, not just large ones.

\subsection{Calibration with Margin Objective}
To minimize the ECE bound, we want large robust margins. For this, we combine two strategies: (i)~directly raising the margin at adversarial points, and (ii)~ensuring clean margins do not collapse under perturbation:
\begin{equation}\label{eq:CalMO-loss}
\begin{aligned}
L_{\mathrm{CalMO}}(\theta)
= \mathbb{E}_{(x,y)}\Big[
    &\ell_{\mathrm{CE}}(\theta;x,y)
    + \lambda_r\,R_{\mathrm{rob}}(\theta;x,y) \\
    &+ \lambda_s\,R_{\mathrm{smooth}}(\theta;x,y)
\Big].
\end{aligned}
\end{equation}
  where $\lambda_r, \lambda_s \ge 0$ are hyperparameters and $R_{\mathrm{rob}}$ and $R_{\mathrm{smooth}}$ are defined below.

\paragraph{Robustness regularizer.}
  Following TRADES~\citep{zhang2019theoretically}, we encourage consistent predictions
  between clean and adversarial inputs:
  \[
  R_{\mathrm{rob}}(\theta;x,y) = D_{\mathrm{KL}}\big(p_\theta(x) \,\|\, p_\theta(x_{\mathrm{adv}})\big),
  \]
  where $x_{\mathrm{adv}} \in \arg\max_{\|x'-x\|\le\varepsilon} \ell_{\mathrm{CE}}(\theta;x',y)$.
  Combined with cross-entropy at $x$, this raises the margin at the worst point in the
  $\varepsilon$-neighborhood, directly targeting $m_{\varepsilon,\theta}(x,y)$.

  \paragraph{Smoothness regularizer.}
  By Lemma~\ref{lem:robust_exp_compare}, if the margin has local Lipschitz constant $L_m(x,y)$, then
  \begin{equation}\label{eq:margin-lipschitz}
  m_{\varepsilon,\theta}(x,y) \ge m_\theta(x,y) - \varepsilon L_m(x,y),
  \end{equation}
  implying $e^{-m_{\varepsilon,\theta}} \le e^{-m_\theta} \cdot e^{\varepsilon L_m}$.
  When $L_m$ is large, this bound becomes vacuous even for large clean margins.
  To prevent this, we penalize:
  \[
  R_{\mathrm{smooth}}(\theta;x,y) = \|\nabla_x m_\theta(x,y)\|_2^2.
  \]
  For neural networks with Lipschitz activations, $m_\theta$ is locally Lipschitz
  with $L_m(x,y) = \|\nabla_x m_\theta(x,y)\|$ almost everywhere.
  This keeps $m_\theta - \varepsilon L_m$ close to $m_\theta$, ensuring large clean
  margins translate to large robust margins.

\paragraph{Why naive margin maximization fails.}
As training progresses, clean margins $m_\theta(x_i, y_i)$ on training points grow, $Q_{\mathrm{train}} \to 0$, and training ECE vanishes. Yet test ECE remains high.
The issue is that large clean margins need not imply large \emph{robust} margins: a model can achieve $m_\theta(x,y) \gg 0$ while the margin collapses at $x + \delta$ for small perturbations. The robust margin $m_{\varepsilon,\theta}(x,y) = \inf_{\|\delta\|\le\varepsilon} m_\theta(x+\delta,y)$ captures this distinction. If robust margins at training points are large, nearby test points---which lie within a $\varepsilon$-neighborhood under typical data distributions---inherit reasonable margins. Fragile margins, large only at training locations, provide no such transfer.

\subsection{Empirical Results with CalMO}
\label{sec:results}

\paragraph{Setup.}
We run a \emph{controlled fixed-budget experiment}: within each optimizer, only the loss function varies; all other training choices are held fixed. Concretely, we train ResNet-20~\citep{he2016deep} on CIFAR-10.
We compare four gradient-based optimizers spanning different training dynamics: SGD, AdamW, Muon, and SAM. Learning rates are tuned per optimizer: $\eta = 0.01$ for SGD, Muon and SAM, $\eta = 0.001$ for AdamW. All runs use batch size 128 for 10,000 steps. We compare standard cross-entropy against CalMO with hyperparameters $\lambda_r = 0.5$ (robustness) and $\lambda_s = 0.01$ (flatness). The adversarial perturbation radius is $\varepsilon = 8/255$, computed via 3-step PGD. Adversarial examples are initialized at the clean input $x$ (no random start) and updated with step size $\alpha = 2/255$. We report test accuracy and ECE computed with 15 bins. 

These values are fixed without tuning; our aim is to validate that the robust-margin regularization motivated by Section~\ref{sec:math} yields calibration gains over CE, and to demonstrate that training-time intervention on the loss landscape can improve calibration without relying on post-hoc corrections. We report test values in Table~\ref{tab:ablation}; training values are reported in Table~\ref{tab:train_test_gap} in the Appendix.

\paragraph{Performance.}
CalMO lowers test ECE across all four optimizers (Table~\ref{tab:ablation}), with reductions ranging from $0.003$ (SAM: $0.020 \to 0.017$) to $0.046$ (Muon: $0.065 \to 0.019$); test accuracy is preserved or improves in every case, by between $+0.2$ and $+4.9$ points, consistent with regularization helping under the fixed 10k-step budget. The Muon gap is the sharpest illustration of the mechanism of Section~\ref{sec:math}: directional optimization drives training margins to extreme values that collapse under input perturbation (fragile margins), keeping test ECE high, and CalMO's robust-margin term directly targets this tail.

    \begin{table}[t]
    \centering
    \footnotesize
    \begin{tabular}{@{}l@{\hspace{4pt}}l@{\hspace{6pt}}c@{\hspace{6pt}}c@{}}
      \toprule
      \textbf{Optimizer} & \textbf{Method} & \textbf{Acc (\%)} & \textbf{ECE} $\downarrow$ \\
      \midrule
      \multirow{4}{*}{SGD}
         & CE & $75.2 \pm 1.2$ & $0.081 \pm 0.021$ \\
         & \cellcolor{ourscolor}Flat.\ ($\lambda_r{=}0$) & \cellcolor{ourscolor}$\mathbf{80.8 \pm 0.1}$ & \cellcolor{ourscolor}$\mathbf{0.053 \pm 0.003}$ \\
         & \cellcolor{ourscolor}Rob.\ ($\lambda_s{=}0$) & \cellcolor{ourscolor}$78.2 \pm 0.9$ & \cellcolor{ourscolor}$0.062 \pm 0.011$ \\
         & \cellcolor{ourscolor}CalMO & \cellcolor{ourscolor}$80.1 \pm 1.4$ & \cellcolor{ourscolor}$0.056 \pm 0.001$ \\
      \midrule
      \multirow{4}{*}{AdamW}
         & CE & $80.7 \pm 0.4$ & $0.061 \pm 0.007$ \\
         & \cellcolor{ourscolor}Flat.\ ($\lambda_r{=}0$) & \cellcolor{ourscolor}$\mathbf{83.3 \pm 0.9}$ & \cellcolor{ourscolor}$0.047 \pm 0.002$ \\
         & \cellcolor{ourscolor}Rob.\ ($\lambda_s{=}0$) & \cellcolor{ourscolor}$82.4 \pm 0.2$ & \cellcolor{ourscolor}$\mathbf{0.039 \pm 0.009}$ \\
         & \cellcolor{ourscolor}CalMO & \cellcolor{ourscolor}$83.2 \pm 1.1$ & \cellcolor{ourscolor}$0.045 \pm 0.007$ \\
      \midrule
      \multirow{4}{*}{Muon}
         & CE & $80.3 \pm 0.3$ & $0.065 \pm 0.016$ \\
         & \cellcolor{ourscolor}Flat.\ ($\lambda_r{=}0$) & \cellcolor{ourscolor}$81.0 \pm 0.2$ & \cellcolor{ourscolor}$0.046 \pm 0.012$ \\
         & \cellcolor{ourscolor}Rob.\ ($\lambda_s{=}0$) & \cellcolor{ourscolor}$\mathbf{81.9 \pm 0.4}$ & \cellcolor{ourscolor}$0.052 \pm 0.010$ \\
         & \cellcolor{ourscolor}CalMO & \cellcolor{ourscolor}$81.7 \pm 0.9$ & \cellcolor{ourscolor}$\mathbf{0.019 \pm 0.002}$ \\
      \midrule
      \multirow{4}{*}{SAM}
         & CE & $85.0 \pm 0.3$ & $0.020 \pm 0.005$ \\
         & \cellcolor{ourscolor}Flat.\ ($\lambda_r{=}0$) & \cellcolor{ourscolor}$84.1 \pm 0.3$ & \cellcolor{ourscolor}$\mathbf{0.016 \pm 0.004}$ \\
         & \cellcolor{ourscolor}Rob.\ ($\lambda_s{=}0$) & \cellcolor{ourscolor}$\mathbf{85.5 \pm 0.4}$ & \cellcolor{ourscolor}$0.021 \pm 0.006$ \\
         & \cellcolor{ourscolor}CalMO & \cellcolor{ourscolor}$85.2 \pm 0.4$ & \cellcolor{ourscolor}$0.017 \pm 0.005$ \\
      \bottomrule
    \end{tabular}
    \vspace{0.2cm}
    \caption{CalMO vs CE. ResNet-20 on CIFAR-10, test set.
    Flat.: flatness only ($\lambda_r{=}0, \lambda_s{=}0.01$); Rob.: robustness only ($\lambda_r{=}0.5, \lambda_s{=}0$); CalMO: $\lambda_r{=}0.5, \lambda_s{=}0.01$. The relative importance of each term varies with optimizer, but CalMO strikes a balance across all settings.}
    \label{tab:ablation}
  \end{table}

\paragraph{Computational cost.}
CalMO incurs additional cost from the robustness and flatness terms. The robustness term requires a 3-step PGD attack per iteration; however, these compute gradients with respect to the input rather than the model parameters, making each PGD step relatively cheap. The flatness term adds one forward pass and one input-gradient computation. This cost can be reduced by using only the flatness term ($\lambda_r = 0$). All methods are trained for 10k steps, and we report performance at the best validation step. Matching compute by training CE longer does not close the gap: CE reaches its validation optimum well before 10k steps, after which performance degrades.

\paragraph{Ablation.}
To isolate the roles of robustness and flatness in shaping calibration, we ablate the CalMO objective into its components and evaluate their effects across optimizers. The results reveal a consistent but optimizer-dependent pattern. 
For SGD, enforcing flatness alone yields most of the calibration gains, while robustness provides limited benefit. 
For AdamW the two terms split roles: robustness yields the largest ECE reduction, while flatness yields the largest accuracy gain, suggesting that its adaptive preconditioning stabilizes parameter-space geometry along the directions relevant for accuracy, while leaving calibration sensitive
to robust-margin control. 
Muon exhibits a different behavior: neither robustness nor flatness alone is sufficient to substantially reduce ECE, but their combination leads to the greatest calibration. This supports the theoretical mechanism developed in Section~\ref{sec:math}, where calibration is governed jointly by the robust-margin tail and geometric sensitivity. From this perspective, CalMO should not be viewed as a universally superior training objective, but rather as a targeted intervention that controls these two terms. 

%% file: sections/conclusion.tex
\section{Conclusion}

We studied calibration as a \emph{training-time phenomenon} rather than a static property of a converged model, and showed that calibration and sharpness are tightly coupled along the optimization trajectory across multiple optimizers. This coupling arises from a shared dependence on margin growth, explaining both the temporal co-evolution of Expected Calibration Error and curvature during training, as well as the frequent divergence between train and test calibration in overlap-dominated regimes.

Building on this perspective, we distinguished between two competing mechanisms for improving calibration: convergence to flat minima and suppression of updates along high-curvature directions. Empirically, optimizers that implement directional control, such as Muon and BulkSGD, yielded consistently better in-sample calibration than methods targeting flat minima alone. Guided by the margin-based bounds, we proposed CalMO, a training objective that jointly regulates robust margins and local smoothness, and showed that their combined regulation can substantially improve out-of-sample calibration---most notably for directionally amplified optimizers such as Muon---while preserving predictive accuracy.

\paragraph{Limitations.}
Our empirical study relies on explicit curvature diagnostics (Gauss--Newton / Hessian-based measurements), which are computationally expensive and constrain the scale of architectures and datasets we can probe. Our theoretical results identify a margin-tail mediator that upper-bounds both calibration error and curvature under a Jacobian-control assumption; we do not claim that these certificates are tight or that the Jacobian bounds hold uniformly in all deep networks. A natural next step is to develop scalable, distributionally robust proxies for the mediator (for example, low-rank spectral estimators, mini-batch surrogates, or input-space stability measurements) and to characterize when they preserve the qualitative regime predictions we derive.

Our experiments also cover small-scale image classification only; extending the trajectory-level analysis to language-model training, where calibration is a central open question, is a natural follow-up. More broadly, we see the trajectory perspective itself---tracking how calibration, curvature, and margins co-evolve rather than inspecting them only at convergence---as a lens applicable to other training-time phenomena in deep networks.

% \section*{Impact Statement}
% % Authors are required to include a statement of the potential broader impact of their work.
% This paper studies the training-time relationship between loss-landscape geometry, margins, and calibration in neural networks. The contributions are theoretical and diagnostic in nature, focusing on understanding how curvature-related quantities evolve along optimization trajectories and how they relate to probabilistic confidence estimates.

% The work does not propose deployment-ready systems or application-specific methods, and therefore has no immediate or direct societal impact. Any downstream effects would arise only through subsequent use of these theoretical insights in applied settings, which are beyond the scope of this paper.

%% file: sections/appendix/further_related_work.tex
\section{Further Related Work}
\label{appendix:further_related_work}

\subsection{Mitigating Miscalibration: Extended Discussion}

\paragraph{Post-hoc calibration.}
Post-hoc methods learn a mapping from model scores to probabilities on a held-out set. Classical approaches include Platt scaling and isotonic/binning methods~\citep{platt2000probabilistic,Zadrozny2002}, with temperature scaling the de facto recipe for deep networks~\citep{guo2017calibration}. More expressive calibrators---such as Dirichlet calibration and compositional strategies like Mix-n-Match---correct class- or confidence-dependent distortions while preserving accuracy~\citep{kull2019dirichlet,zhang2020mixnmatch}. Because post-hoc methods do not influence training dynamics, they provide limited mechanistic insight and can degrade under distribution shift~\citep{Ovadia2019}.

\paragraph{Intrinsic methods.}
Intrinsic methods incorporate calibration objectives directly into training. These include entropy-based regularization~\citep{pereyra2017regularizing}, label smoothing~\citep{muller2019label}, augmentation schemes such as mixup that soften targets~\citep{thulasidasan2019mixup}, and focal loss---originally introduced for class imbalance---which yields better calibrated classifiers even before post-hoc scaling~\citep{lin2017focal,mukhoti2020calibrating}. Differentiable surrogates of calibration metrics further enable joint training for accuracy and calibration~\citep{kumar2018trainable,bohdal2023metacalibration}.

\paragraph{High-dimensional perspectives on miscalibration.}
Recent theoretical work highlights that miscalibration can arise intrinsically from high-dimensional statistical effects, even in well-specified problems. \citet{li2025optimal} analyze confidence estimates in high-dimensional classification and show that predictive probabilities can systematically deviate from true correctness likelihoods due to margin concentration and estimation noise, suggesting that miscalibration need not stem from optimization failures alone. Our perspective is complementary: rather than asymptotic statistical limits, we study the finite-sample training-time evolution of margins and curvature.

\subsection{Robust Margins, Sharpness, and Calibration}

\paragraph{Robust margins and calibration.}
Cross-entropy training pushes predictions toward extreme softmax outputs. On linearly separable data, gradient descent drives margins to infinity, converging to a max-margin classifier~\citep{Soudry2018}. While large margins aid classification, they also amplify overconfidence: \citet{qin2021improving} showed that inputs with small robust margin are more likely to be miscalibrated, and proposed adaptive label smoothing on such points. Focal loss~\citep{mukhoti2020calibrating} and label smoothing~\citep{muller2019label} can likewise curb overconfidence on hard examples. \citet{foret2021sam}'s SAM, which biases optimization toward flatter minima, has been observed to lower calibration error~\citep{Zheng2021,Mollenhoff2023}. These results share a common theme: controlling the growth or fragility of margins tends to improve calibration. Achieving both robustness and calibration is nevertheless non-trivial---standard adversarial training can degrade calibration without targeted interventions~\citep{Stutz2020CCAt}.

\paragraph{Flat minima and margins.}
Loss-landscape geometry has long been linked to generalization, with flat minima hypothesized to be preferable to sharp ones~\citep{hochreiter1997flatminima,keskar2016largebatch}. \citet{Dinh2017} complicated this picture: scaling symmetries in deep networks allow arbitrarily sharp solutions with identical outputs, motivating scale-invariant sharpness measures~\citep{Tsuzuku2019,Liang2019}. Under such measures, flatter minima correlate with better generalization in CE-trained models~\citep{Jiang2019,Maddox2020}. A structural correlate is the classification margin: flat basins of the CE loss align with large training margins~\citep{Lengyel2021,Jiang2019}. Small-batch SGD, which implicitly enlarges margins~\citep{Hoffer2017}, also finds flatter solutions than large-batch training~\citep{keskar2016largebatch}. Adversarially robust models---which have larger input margins---exhibit lower curvature in weight space~\citep{Stutz2021}; conversely, weight-space flatness regularizers such as entropy-SGD~\citep{Chaudhari2019entropy} improve adversarial robustness as a side effect~\citep{Stutz2021}.

\paragraph{Linear vs.~non-linear caveats.}
In linear models trained with cross-entropy, the notion of ``flat vs.~sharp'' is less meaningful: on separable data the weight norm grows without bound as margins maximize, driving the Hessian to zero while the classifier becomes arbitrarily confident. Meaningful cross-setting flatness comparisons therefore require correcting for reparameterization invariances~\citep{Dinh2017,Neyshabur2017pac,Tsuzuku2019}. In deep non-linear networks with such corrections, large margins correspond to flatter minima~\citep{Lengyel2021}. This distinction explains why margin-based analyses~\citep{Bartlett2017,Nagarajan2019} are often preferred for theoretical guarantees in linear settings.

\subsection{Positioning Relative to Prior Work}

The four literatures above---miscalibration under cross-entropy, sharpness/flatness and optimization stability, margin maximization via implicit bias, and robustness--calibration connections---together with recent work on calibration benefits of sharpness-aware optimizers~\citep{tan2025samcalibration}, provide the backdrop for our contribution.

Our work extends these literatures in three directions:

\begin{itemize}[leftmargin=1em,itemsep=0.25em,topsep=0.25em,parsep=0pt]
    \item \textbf{Trajectory-level analysis.} Most prior work compares final solutions; we track calibration and curvature \emph{pathwise} across training and show that they co-evolve, peaking together near the edge of stability and decaying together.
    \item \textbf{A shared margin-tail mediator.} We prove that a single exponential margin moment and its robust variant simultaneously upper-bound ECE and Gauss--Newton sharpness, with a two-sided sandwich for ECE in the interpolating regime. The triangulation ECE\,$\leftrightarrow$\,GN sharpness\,$\leftrightarrow$\,robust margins is, to our knowledge, new.
    \item \textbf{Directional vs.\ flat-minima interventions.} We distinguish optimizers that bias toward flat minima (SAM) from those that suppress steep descent directions along the trajectory (Muon, BulkSGD), and show empirically that the latter yield more reliable in-sample calibration gains. This motivates CalMO as a principled intervention on the mediator.
\end{itemize}

%% file: sections/appendix/logistic_proof.tex
\section{Proofs for Section~\ref{sec:math}}
\label{app:math_proofs}

\subsection{Definitions and basic reductions}
\label{app:ece_defs}

\paragraph{Fixed binning.}
Fix $M\in\mathbb{N}$ and deterministic bin edges $0=a_0<a_1<\cdots<a_M=1$.
Define bins $I_m:=(a_{m-1},a_m]$ for $m=1,\dots,M$.

\paragraph{Population ECE with fixed bins.}
Let $(X,Y)\sim\pi$ with $Y\in\{1,\dots,K\}$.
For fixed $\theta$, define the predicted label
\[
\widehat{Y}:=\arg\max_{k} z_\theta(X)_k
\]
(using the deterministic tie-break rule from the main text),
and the confidence
\[
\widehat{P}:=\max_k p_\theta(X)_k = p_\theta(X)_{\widehat{Y}}.
\]
Let $B_m:=\{\widehat{P}\in I_m\}$. Define binwise accuracy and confidence by
\[
\mathrm{acc}(B_m):=
\begin{cases}
\mathbb{P}(\widehat{Y}=Y\mid B_m), & \mathbb{P}(B_m)>0,\\
0,& \mathbb{P}(B_m)=0,
\end{cases}
\qquad
\mathrm{conf}(B_m):=
\begin{cases}
\mathbb{E}[\widehat{P}\mid B_m], & \mathbb{P}(B_m)>0,\\
0,& \mathbb{P}(B_m)=0.
\end{cases}
\]
The population binned calibration error is
\[
\mathrm{ECE}_M(\theta;\pi)
:=\sum_{m=1}^M \mathbb{P}(B_m)\,|\mathrm{acc}(B_m)-\mathrm{conf}(B_m)|.
\]
Equivalently, with $Z:=\mathbf{1}\{\widehat{Y}=Y\}-\widehat{P}$,
\[
\mathrm{ECE}_M(\theta;\pi)
=\sum_{m=1}^M \mathbb{P}(B_m)\,\big|\mathbb{E}[Z\mid B_m]\big|.
\]

\paragraph{Empirical ECE.}
Given a dataset $\mathcal{D}=\{(x_i,y_i)\}_{i=1}^n$, define
\[
\widehat{Y}_i:=\arg\max_k z_\theta(x_i)_k
\quad\text{(same deterministic tie-break)},\qquad
\widehat{P}_i:=\max_k p_\theta(x_i)_k,
\]
and bins $B_m:=\{i:\widehat{P}_i\in I_m\}$.
Let $Z_i:=\mathbf{1}\{\widehat{Y}_i=y_i\}-\widehat{P}_i$.
Then
\[
\mathrm{ECE}_M(\theta;\mathcal{D})
:=\sum_{m=1}^M \frac{|B_m|}{n}\,
\left|\frac{1}{|B_m|}\sum_{i\in B_m} Z_i\right|,
\]
with the convention that the inner average is $0$ when $|B_m|=0$.

% \newpage

\subsection{Core lemmas}
\label{app:core_lemmas}

\begin{lemma}[ECE is bounded by the mean absolute correctness--confidence gap]
\label{lem:ece_by_abs_gap}
\textbf{(Population).}
Let $Z:=\mathbf{1}\{\widehat{Y}=Y\}-\widehat{P}$ and let
$\mathcal{G}:=\sigma(B_1,\dots,B_M)$ be the $\sigma$-algebra generated by the bin events
$B_m:=\{\widehat{P}\in I_m\}$.
Note that $Z\in[-1,1]$, hence $Z$ is integrable. Then
\[
\mathrm{ECE}_M(\theta;\pi)
=\mathbb{E}\Big[\,\big|\,\mathbb{E}[Z\mid \mathcal{G}]\,\big|\,\Big]
\;\le\;\mathbb{E}[|Z|].
\]

\textbf{(Empirical).}
For any dataset $\mathcal{D}$,
\[
\mathrm{ECE}_M(\theta;\mathcal{D})
\;\le\;
\frac{1}{n}\sum_{i=1}^n \big|\mathbf{1}\{\widehat{Y}_i=y_i\}-\widehat{P}_i\big|.
\]
\end{lemma}

\begin{proof}
\textbf{Population.}
For each bin $B_m$ with $\mathbb{P}(B_m)>0$,
\[
\mathbb{E}[Z\mid B_m]
=
\mathbb{P}(\widehat{Y}=Y\mid B_m)-\mathbb{E}[\widehat{P}\mid B_m]
=
\mathrm{acc}(B_m)-\mathrm{conf}(B_m),
\]
and we set $\mathbb{E}[Z\mid B_m]:=0$ when $\mathbb{P}(B_m)=0$.
Therefore,
\[
\mathrm{ECE}_M(\theta;\pi)
=\sum_{m=1}^M \mathbb{P}(B_m)\,\big|\mathbb{E}[Z\mid B_m]\big|
=\mathbb{E}\Big[\,\big|\,\mathbb{E}[Z\mid \mathcal{G}]\,\big|\,\Big].
\]
By Jensen's inequality for the convex function $u\mapsto |u|$,
\[
\mathbb{E}\Big[\,\big|\,\mathbb{E}[Z\mid \mathcal{G}]\,\big|\,\Big]
\le
\mathbb{E}\big[\mathbb{E}[|Z|\mid \mathcal{G}]\big]
=
\mathbb{E}[|Z|].
\]

\textbf{Empirical.}
For each bin $B_m$ with $|B_m|>0$, let $Z_i:=\mathbf{1}\{\widehat{Y}_i=y_i\}-\widehat{P}_i$.
Then
\[
\big|\mathrm{acc}(B_m)-\mathrm{conf}(B_m)\big|
=
\left|\frac{1}{|B_m|}\sum_{i\in B_m} Z_i\right|
\le
\frac{1}{|B_m|}\sum_{i\in B_m} |Z_i|
\]
by the triangle inequality. Multiplying by $|B_m|/n$ and summing over $m$ yields the claim.
\end{proof}

\begin{lemma}[Correctness--confidence gap is controlled by the true-class probability]
\label{lem:gap_by_trueprob}
For any $(x,y)$ and $\theta$, where $\widehat{y}(x)$ and $\widehat{P}(x)$ are defined as above,
\[
\big|\mathbf{1}\{\widehat{y}(x)=y\}-\widehat{P}(x)\big|
\;\le\;
1-p_\theta(x)_y.
\]
\end{lemma}

\begin{proof}
Let $\widehat{y}=\widehat{y}(x)$ and $\widehat{P}=\widehat{P}(x)=\max_k p_\theta(x)_k=p_\theta(x)_{\widehat{y}}$.
If $\widehat{y}=y$, then $|\mathbf{1}\{\widehat{y}=y\}-\widehat{P}|=|1-p_y|=1-p_y$.
If $\widehat{y}\neq y$, then $|\mathbf{1}\{\widehat{y}=y\}-\widehat{P}|=\widehat{P}=p_{\widehat{y}}
\le \sum_{j\neq y} p_j = 1-p_y$,
since $p_{\widehat{y}}$ is one of the nonnegative summands in $\sum_{j\neq y}p_j$.
\end{proof}

\begin{lemma}[Softmax tail bound: $1-p_y$ is exponentially controlled by the true margin]
\label{lem:tail_by_margin}
Let $p=\mathrm{softmax}(z)\in\Delta^{K-1}$ and fix a label $y$.
Define the true margin $m:=z_y-\max_{j\neq y} z_j$.
Then
\begin{equation}
1-p_y \;\le\; \sum_{j\neq y} e^{z_j-z_y}\;\le\;(K-1)e^{-m}.
\label{eq:tail_upper}
\end{equation}
Moreover,
\begin{equation}
\frac{e^{-m}}{1+(K-1)e^{-m}}
\;\le\;
1-p_y.
\label{eq:tail_lower}
\end{equation}
In particular, if $m\ge 0$ then
\begin{equation}
\frac{1}{K}e^{-m}\;\le\;1-p_y\;\le\;(K-1)e^{-m}.
\label{eq:tail_sandwich_mpos}
\end{equation}
\end{lemma}

\begin{proof}
Write
\[
p_y=\frac{e^{z_y}}{\sum_{k=1}^K e^{z_k}}
=\frac{1}{1+\sum_{j\neq y} e^{z_j-z_y}},
\qquad
1-p_y=\frac{\sum_{j\neq y} e^{z_j-z_y}}{1+\sum_{j\neq y} e^{z_j-z_y}}.
\]
Let $S:=\sum_{j\neq y} e^{z_j-z_y}\ge 0$. Then $1-p_y=S/(1+S)\le S$, proving the first inequality in~\eqref{eq:tail_upper}.
For each $j\neq y$, $z_j-z_y \le \max_{k\neq y} z_k - z_y = -m$, hence $e^{z_j-z_y}\le e^{-m}$ and $S\le (K-1)e^{-m}$,
proving the second inequality in~\eqref{eq:tail_upper}.
For~\eqref{eq:tail_lower}, pick $j^\star\in\arg\max_{j\neq y} z_j$ so that $z_{j^\star}-z_y=-m$ and hence $S\ge e^{-m}$.
Therefore
\[
1-p_y=\frac{S}{1+S}\ge \frac{e^{-m}}{1+S}\ge \frac{e^{-m}}{1+(K-1)e^{-m}},
\]
where the last step uses $S\le (K-1)e^{-m}$.
If $m\ge 0$ then $e^{-m}\le 1$ and thus $1+(K-1)e^{-m}\le 1+(K-1)=K$, giving~\eqref{eq:tail_sandwich_mpos}.
\end{proof}

\begin{lemma}[Cross-entropy logit Hessian top eigenvalue is controlled by $1-p_{\max}$]
\label{lem:hz_eig_upper}
Let $p\in\Delta^{K-1}$ and define $H_z(p):=\mathrm{diag}(p)-pp^\top$.
Then
\[
\lambda_{\max}\big(H_z(p)\big)\;\le\; 2\bigl(1-p_{\max}\bigr),
\qquad p_{\max}:=\max_k p_k.
\]
\end{lemma}

\begin{proof}
We use Gershgorin's circle theorem for symmetric matrices.
Write $A:=H_z(p)$, so that for each $i$,
\[
A_{ii}=p_i(1-p_i),\qquad A_{ij}=-p_ip_j\quad(i\neq j).
\]
Let $R_i:=\sum_{j\neq i}|A_{ij}|=\sum_{j\neq i}p_ip_j=p_i(1-p_i)$.
Gershgorin implies every eigenvalue $\lambda$ of $A$ lies in at least one interval
\[
\lambda \in [A_{ii}-R_i,\,A_{ii}+R_i]=[0,\,2p_i(1-p_i)]\quad\text{for some }i.
\]
Hence
\[
\lambda_{\max}(A)\le \max_i 2p_i(1-p_i).
\]
Now fix $k^\star\in\arg\max_k p_k$ so that $p_{k^\star}=p_{\max}$.
If $i=k^\star$, then $p_i(1-p_i)=p_{\max}(1-p_{\max})\le 1-p_{\max}$.
If $i\neq k^\star$, then $p_i\le 1-p_{\max}$ and $p_i(1-p_i)\le p_i\le 1-p_{\max}$.
Therefore $\max_i p_i(1-p_i)\le 1-p_{\max}$ and consequently
\[
\lambda_{\max}(H_z(p))\le 2(1-p_{\max}),
\]
as claimed.
\end{proof}

\begin{lemma}[Robust margin comparisons (trivial upper bound; Lipschitz lower bound)]
\label{lem:robust_exp_compare}
For all $(x,y)$,
\begin{equation}
\label{eq:robust_trivial_upper}
m_{\varepsilon,\theta}(x,y)\le m_\theta(x,y)
\qquad\Longrightarrow\qquad
e^{-m_\theta(x,y)}\le e^{-m_{\varepsilon,\theta}(x,y)}.
\end{equation}
If moreover there exists $L_m(x,y)\in[0,\infty)$ such that
\[
|m_\theta(x+\delta,y)-m_\theta(x,y)|\le L_m(x,y)\,\|\delta\|
\qquad\forall\,\|\delta\|\le \varepsilon,
\]
then
\begin{equation}
\label{eq:robust_Lip_lower}
m_{\varepsilon,\theta}(x,y)\ge m_\theta(x,y)-\varepsilon L_m(x,y)
\qquad\Longrightarrow\qquad
e^{-m_\theta(x,y)}\ge e^{-\varepsilon L_m(x,y)}\,e^{-m_{\varepsilon,\theta}(x,y)}.
\end{equation}
\end{lemma}

\begin{proof}
\textbf{Trivial robust-vs-clean comparison.}
By definition of the infimum and because $\delta=0$ is feasible, we have
\[
m_{\varepsilon,\theta}(x,y)=\inf_{\|\delta\|\le\varepsilon} m_\theta(x+\delta,y)\le m_\theta(x,y).
\]
Since the map $t\mapsto e^{-t}$ is \emph{decreasing}, this implies
\[
e^{-m_\theta(x,y)} \le e^{-m_{\varepsilon,\theta}(x,y)},
\]
which is~\eqref{eq:robust_trivial_upper}.

\textbf{Lipschitz lower bound.}
Assume the stated local Lipschitz condition at $(x,y)$. Then for any $\|\delta\|\le \varepsilon$,
\[
m_\theta(x+\delta,y)\ge m_\theta(x,y)-L_m(x,y)\,\|\delta\|
\ge m_\theta(x,y)-\varepsilon L_m(x,y).
\]
Taking the infimum over all $\|\delta\|\le\varepsilon$ yields
\[
m_{\varepsilon,\theta}(x,y)\ge m_\theta(x,y)-\varepsilon L_m(x,y),
\]
equivalently
\[
m_\theta(x,y)\le m_{\varepsilon,\theta}(x,y)+\varepsilon L_m(x,y).
\]
Multiply by $-1$ (which flips the inequality) to get
\[
-m_\theta(x,y)\ge -m_{\varepsilon,\theta}(x,y)-\varepsilon L_m(x,y),
\]
and exponentiate to obtain
\[
e^{-m_\theta(x,y)}\ge e^{-m_{\varepsilon,\theta}(x,y)}\,e^{-\varepsilon L_m(x,y)},
\]
which is~\eqref{eq:robust_Lip_lower}.
\end{proof}

% \newpage

\begin{remark}[Label-free GN bound via predicted margin]
\label{rem:label_free_gn_alt}
Because $H_z(p_\theta(X))$ depends only on $X$, one can avoid the label $Y$ in the GN bound.
Let $\widehat{y}(x)\in\arg\max_k z_\theta(x)_k$ (with the deterministic tie-break rule) and define the \emph{predicted margin}
\[
\widehat{m}_\theta(x):=z_\theta(x)_{\widehat{y}(x)}-\max_{j\neq \widehat{y}(x)} z_\theta(x)_j\;\ge 0.
\]
Applying Lemma~\ref{lem:tail_by_margin} with $y=\widehat{y}(x)$ yields
\[
1-p_{\max}(x)=1-p_\theta(x)_{\widehat{y}(x)}\le (K-1)e^{-\widehat{m}_\theta(x)}.
\]
Combining with Lemma~\ref{lem:hz_eig_upper} gives
\[
\lambda_{\max}\!\big(H_z(p_\theta(x))\big)
\le 2\bigl(1-p_{\max}(x)\bigr)
\le 2(K-1)e^{-\widehat{m}_\theta(x)}.
\]
cConsequently, under $\|J_\theta(X)\|_{\mathrm{op}}\le C_J$ $\pi$-a.s.,
\[
\lambda_{\max}\!\big(H_{\mathrm{GN}}(\theta;\pi)\big)
\le 2C_J^2(K-1)\,\mathbb{E}\big[e^{-\widehat{m}_\theta(X)}\big].
\]
This can be substantially tighter than bounds routing through $Y$ when the model is confidently incorrect.
\end{remark}

% \newpage

\subsection{Rigorous restatement of the main theorems}
\label{sec:theorems}

\paragraph{Notation alignment with Section~\ref{sec:math}.}
Fix a robust radius $\varepsilon>0$.
To match the main-text notation, we use
\[
Q(\theta;\pi)\;:=\;\mathbb{E}_{(X,Y)\sim\pi}\!\big[e^{-m_{\varepsilon,\theta}(X,Y)}\big]
\qquad\text{and}\qquad
Q_{\mathcal D}(\theta)\;:=\;\frac{1}{n}\sum_{i=1}^n e^{-m_\theta(x_i,y_i)}.
\]
When $\pi$ (or $\mathcal D$) is clear from context, we may drop it from the notation.
For comparison with alternative functionals used in some intermediate lemmas, note that
$Q(\theta;\pi)$ coincides with the quantity previously denoted $\Psi_\varepsilon^{0}(\theta;\pi)$,
and $Q_{\mathcal D}(\theta)$ coincides with the quantity previously denoted $\mu(\theta;\mathcal D)$.
If a pointwise margin Lipschitz constant $L_m(\cdot,\cdot)$ is available, we also define the (generally looser) population functional
\[
Q^{+}(\theta;\pi)\;:=\;\mathbb{E}_{(X,Y)\sim\pi}\!\big[e^{\varepsilon L_m(X,Y)}\,e^{-m_{\varepsilon,\theta}(X,Y)}\big],
\]
and the finite-sample robust moments
\[
Q^{0}_{\varepsilon,\mathcal D}(\theta)\;:=\;\frac{1}{n}\sum_{i=1}^n e^{-m_{\varepsilon,\theta}(x_i,y_i)},
\qquad
Q^{-}_{\varepsilon,\mathcal D}(\theta)\;:=\;\frac{1}{n}\sum_{i=1}^n e^{-\varepsilon L_m(x_i,y_i)}\,e^{-m_{\varepsilon,\theta}(x_i,y_i)}.
\]
The proofs appear in Subsections~\ref{app:proof_overlap} and~\ref{app:proof_separable}.

\medskip
\noindent\textbf{Theorem~\ref{thm:overlap_bounds}} \textit{(Overlap regime: simultaneous robust-margin upper bounds)}\textbf{.}\\
Let $\pi$ be any distribution on $\mathcal{X}\times\{1,\dots,K\}$ and let $\theta$ be any parameter vector.

\paragraph{(i) Calibration upper bound.}
For the population binned calibration error $\mathrm{ECE}_M(\theta;\pi)$,
\[
\mathrm{ECE}_M(\theta;\pi)
\;\le\;
(K-1)\,\mathbb{E}\big[e^{-m_\theta(X,Y)}\big]
\;\le\;
(K-1)\,Q(\theta;\pi).
\]
If $L_m$ is defined, then also $\mathrm{ECE}_M(\theta;\pi)\le (K-1)\,Q^{+}(\theta;\pi)$, but this is never tighter than the
$Q(\theta;\pi)$ bound since $Q^{+}(\theta;\pi)\ge Q(\theta;\pi)$.

\paragraph{(ii) Gauss--Newton curvature (top eigenvalue) upper bound.}
Assume additionally that the logit Jacobian is uniformly bounded in operator norm,
\[
\|J_\theta(X)\|_{\mathrm{op}}\le C_J
\qquad \text{$\pi$-a.s.}
\]
Then the population Gauss--Newton matrix
\[
H_{\mathrm{GN}}(\theta;\pi)
:=\mathbb{E}_{(X,Y)\sim\pi}\!\big[J_\theta(X)^\top H_z(p_\theta(X))J_\theta(X)\big]
\]
satisfies
\[
\lambda_{\max}\!\big(H_{\mathrm{GN}}(\theta;\pi)\big)
\;\le\;
2C_J^2\,(K-1)\,\mathbb{E}\big[e^{-m_\theta(X,Y)}\big]
\;\le\;
2C_J^2\,(K-1)\,Q(\theta;\pi).
\]
If $L_m$ is defined, then also
$\lambda_{\max}(H_{\mathrm{GN}}(\theta;\pi))\le 2C_J^2\,(K-1)\,Q^{+}(\theta;\pi)$, again a looser bound than the one via $Q(\theta;\pi)$.

\paragraph{(iii) What the bound can (and cannot) certify.}
If along a training trajectory $\{\theta_t\}$ the robust moment $Q(\theta_t;\pi)$ fails to converge to $0$,
then the bounds in (i)--(ii) do not certify that $\mathrm{ECE}_M(\theta_t;\pi)\to 0$ or
$\lambda_{\max}(H_{\mathrm{GN}}(\theta_t;\pi))\to 0$.

\paragraph{(iv) Remarks (label dependence and trivial clamping).}
\begin{itemize}\setlength{\itemsep}{1pt}
\item \textbf{Label dependence vs.\ label-free curvature.}
$\lambda_{\max}(H_{\mathrm{GN}}(\theta;\pi))$ depends only on the marginal law of $X$ (since $H_z(p_\theta(X))$ is label-free),
whereas the bound above routes through $Y$ via $m_\theta(X,Y)$ (equivalently $1-p_\theta(X)_Y$). This is valid but can be loose,
especially when the model is confidently wrong.
A label-free alternative (via the predicted margin) is given in Remark~\ref{rem:label_free_gn_alt}.
\item \textbf{Bounds may exceed $1$.} Since $\mathrm{ECE}_M(\theta;\pi)\in[0,1]$, any upper bound $U$ can be trivially tightened to $\min\{1,U\}$.
\end{itemize}

\medskip
\noindent\textbf{Theorem~\ref{thm:separable_bounds}} \textit{(Interpolating regime: two-sided ECE control and coupling to $\lambda_{\max}$)}\textbf{.}\\
Assume $\gamma(\theta;\mathcal{D})>0$, i.e.\ the training set $\mathcal{D}=\{(x_i,y_i)\}_{i=1}^n$ is correctly classified with strictly positive \emph{true} margin.
(Strictness ensures $\widehat{Y}_i=y_i$ without tie-breaking subtleties.)

\paragraph{(i) Two-sided control of in-sample ECE by the exponential margin moment.}
\[
\frac{1}{K}\,Q_{\mathcal{D}}(\theta)
\;\le\;
\mathrm{ECE}_M(\theta;\mathcal{D})
\;\le\;
(K-1)\,Q_{\mathcal{D}}(\theta)
\;\le\;
(K-1)\,e^{-\gamma(\theta;\mathcal{D})}.
\]
(As always, one may clamp the upper bound by $\min\{1,\cdot\}$.)

\paragraph{(ii) In-sample GN curvature bound in terms of the same moment.}
Assume additionally that $\|J_\theta(x_i)\|_{\mathrm{op}}\le C_J$ for all $i=1,\dots,n$. Then
\[
\lambda_{\max}\!\big(H_{\mathrm{GN}}(\theta;\mathcal{D})\big)
\;\le\;
2C_J^2\,(K-1)\,Q_{\mathcal{D}}(\theta)
\;\le\;
2C_J^2\,K(K-1)\,\mathrm{ECE}_M(\theta;\mathcal{D}).
\]

\paragraph{(iii) Consequence: in the interpolating regime, curvature and ECE are forced to co-vary.}
Under the same assumptions,
\[
\mathrm{ECE}_M(\theta;\mathcal{D})
\;\ge\;
\frac{\lambda_{\max}(H_{\mathrm{GN}}(\theta;\mathcal{D}))}{2C_J^2\,K(K-1)}.
\]
Thus, once the training set is correctly classified and Jacobians remain bounded, large GN curvature cannot occur without large in-sample ECE.

\paragraph{(iv) Robust-margin variant (optional; requires local Lipschitzness at $(x_i,y_i)$).}
Assume moreover that for each $i$ there exists $L_m(x_i,y_i)\in[0,\infty)$ such that
\[
|m_\theta(x_i+\delta,y_i)-m_\theta(x_i,y_i)|
\le
L_m(x_i,y_i)\,\|\delta\|
\qquad\forall\ \|\delta\|\le \varepsilon.
\]
Then, with the robust moments $Q^{0}_{\varepsilon,\mathcal D}(\theta)$ and $Q^{-}_{\varepsilon,\mathcal D}(\theta)$ defined above,
\[
\frac{1}{K}\,Q^{-}_{\varepsilon,\mathcal D}(\theta)
\;\le\;
\mathrm{ECE}_M(\theta;\mathcal{D})
\;\le\;
(K-1)\,Q^{0}_{\varepsilon,\mathcal D}(\theta),
\]
and, under $\max_i\|J_\theta(x_i)\|_{\mathrm{op}}\le C_J$,
\[
\lambda_{\max}\!\big(H_{\mathrm{GN}}(\theta;\mathcal{D})\big)
\;\le\;
2C_J^2\,(K-1)\,Q^{0}_{\varepsilon,\mathcal D}(\theta).
\]

\paragraph{(v) Remark (binning irrelevance under perfect accuracy).}
Under $\gamma(\theta;\mathcal{D})>0$, every nonempty bin has empirical accuracy $1$,
so $\mathrm{ECE}_M(\theta;\mathcal{D})$ reduces to the \emph{mean misconfidence} and becomes independent of the choice of bins.

\subsection{Proof of Theorem~\ref{thm:overlap_bounds}}
\label{app:proof_overlap}

% =========================
% Adversarial proofread + notation cleanup for Appendix E.4--E.5
% (Target: remove legacy \Psi/\mu/s_{\rm GN} notation; use Q and Q_{\mathcal D} consistently.)
% =========================

% -------------------------------------------------
% E.4 Proof of Theorem 4.1 (Overlap regime)
% -------------------------------------------------
\begin{proof}[Proof of Theorem~4.1]
\textbf{(i) Calibration bound.}
By Lemma~\ref{lem:ece_by_abs_gap} (population version),
\[
\mathrm{ECE}_M(\theta;\pi)\le 
\mathbb{E}_{(X,Y)\sim\pi}\Big[\big|\mathbf{1}\{\widehat{Y}=Y\}-\widehat{P}\big|\Big].
\]
By Lemma~\ref{lem:gap_by_trueprob},
\[
\big|\mathbf{1}\{\widehat{Y}=Y\}-\widehat{P}\big|
\le 1-p_\theta(X)_Y,
\]
hence
\[
\mathrm{ECE}_M(\theta;\pi)\le \mathbb{E}_{(X,Y)\sim\pi}\big[1-p_\theta(X)_Y\big].
\]
Applying Lemma~\ref{lem:tail_by_margin} with $z=z_\theta(X)$ and $y=Y$ yields
\[
1-p_\theta(X)_Y \le (K-1)e^{-m_\theta(X,Y)}
\qquad \pi\text{-a.s.},
\]
so
\[
\mathrm{ECE}_M(\theta;\pi)\le (K-1)\,\mathbb{E}_{(X,Y)\sim\pi}\big[e^{-m_\theta(X,Y)}\big].
\]
Finally, by Lemma~\ref{lem:robust_exp_compare} (the trivial robust-vs-clean comparison),
\[
e^{-m_\theta(X,Y)}\le e^{-m_{\varepsilon,\theta}(X,Y)}
\qquad \pi\text{-a.s.},
\]
and therefore
\[
\mathrm{ECE}_M(\theta;\pi)
\le (K-1)\,\mathbb{E}_{(X,Y)\sim\pi}\big[e^{-m_{\varepsilon,\theta}(X,Y)}\big]
=(K-1)\,Q(\theta;\pi).
\]
If the local Lipschitz condition in Lemma~\ref{lem:robust_exp_compare} holds so that $Q^+(\theta;\pi)$ is defined, then
$Q(\theta;\pi)\le Q^+(\theta;\pi)$ since $e^{\varepsilon L_m(X,Y)}\ge 1$.

\textbf{(ii) GN curvature bound.}
Define the random PSD matrix
\[
A(X):=J_\theta(X)^\top H_z(p_\theta(X))J_\theta(X)\succeq 0.
\]
Then $H_{\mathrm{GN}}(\theta;\pi)=\mathbb{E}_{(X,Y)\sim\pi}[A(X)]$.
Since $\lambda_{\max}$ is convex on the PSD cone (equivalently, by the variational characterization),
\[
\lambda_{\max}\big(H_{\mathrm{GN}}(\theta;\pi)\big)
=\lambda_{\max}\!\big(\mathbb{E}[A(X)]\big)
\le
\mathbb{E}\big[\lambda_{\max}(A(X))\big].
\]
For each realization $X$,
\[
\lambda_{\max}(A(X))
\le
\|J_\theta(X)\|_{\mathrm{op}}^2\,\lambda_{\max}\big(H_z(p_\theta(X))\big).
\]
Under $\|J_\theta(X)\|_{\mathrm{op}}\le C_J$ $\pi$-a.s.,
\[
\lambda_{\max}(A(X))\le C_J^2\,\lambda_{\max}\big(H_z(p_\theta(X))\big).
\]
By Lemma~\ref{lem:hz_eig_upper},
\[
\lambda_{\max}(H_z(p_\theta(X)))\le 2(1-p_{\max}(X)),
\qquad p_{\max}(X):=\max_k p_\theta(X)_k.
\]
Since $p_{\max}(X)\ge p_\theta(X)_Y$, we have $1-p_{\max}(X)\le 1-p_\theta(X)_Y$, hence
\[
\lambda_{\max}(H_z(p_\theta(X)))
\le 2\bigl(1-p_\theta(X)_Y\bigr)
\le 2(K-1)e^{-m_\theta(X,Y)}
\qquad \pi\text{-a.s.}
\]
Therefore,
\[
\lambda_{\max}\big(H_{\mathrm{GN}}(\theta;\pi)\big)
\le
2C_J^2(K-1)\,\mathbb{E}_{(X,Y)\sim\pi}\big[e^{-m_\theta(X,Y)}\big]
\le
2C_J^2(K-1)\,\mathbb{E}_{(X,Y)\sim\pi}\big[e^{-m_{\varepsilon,\theta}(X,Y)}\big]
=
2C_J^2(K-1)\,Q(\theta;\pi),
\]
where the last inequality uses Lemma~\ref{lem:robust_exp_compare}.
If the local Lipschitz condition in Lemma~\ref{lem:robust_exp_compare} holds, then also
$Q(\theta;\pi)\le Q^+(\theta;\pi)$.

\textbf{(iii) Certification statement.}
Immediate from (i)--(ii): if $Q(\theta_t;\pi)\not\to 0$ (or likewise $Q^+(\theta_t;\pi)\not\to 0$),
then the corresponding right-hand sides do not converge to $0$ and therefore cannot certify
$\mathrm{ECE}_M(\theta_t;\pi)\to 0$ nor $\lambda_{\max}(H_{\mathrm{GN}}(\theta_t;\pi))\to 0$.
\end{proof}

% -------------------------------------------------
% E.5 Proof of Theorem 4.2 (Interpolating regime on the training set)
% -------------------------------------------------

\subsection{Proof of Theorem~\ref{thm:separable_bounds}}
\label{app:proof_separable}

\begin{proof}[Proof of Theorem~4.2]
Assume $\gamma(\theta;\mathcal{D})>0$, i.e.\ $m_\theta(x_i,y_i)>0$ for all $i$.
Hence $\widehat{Y}_i=y_i$ for all $i$ (no tie-breaking occurs).

\textbf{(i) Two-sided ECE--moment bounds.}
Because $\widehat{Y}_i=y_i$, every nonempty bin $B_m$ has empirical accuracy $\mathrm{acc}(B_m)=1$.
Therefore, for each nonempty bin,
\[
\bigl|1-\mathrm{conf}(B_m)\bigr|=1-\mathrm{conf}(B_m)
\qquad\text{since }\mathrm{conf}(B_m)\in[0,1].
\]
Hence
\[
\mathrm{ECE}_M(\theta;\mathcal{D})
=\sum_{m=1}^M \frac{|B_m|}{n}\,\Bigl(1-\mathrm{conf}(B_m)\Bigr)
=1-\frac{1}{n}\sum_{i=1}^n \widehat{P}_i.
\]
Since $\widehat{P}_i=\max_k p_\theta(x_i)_k=p_\theta(x_i)_{\widehat{Y}_i}$ and $\widehat{Y}_i=y_i$, we have
$\widehat{P}_i=p_\theta(x_i)_{y_i}$, and thus
\begin{equation}
\mathrm{ECE}_M(\theta;\mathcal{D})
=\frac{1}{n}\sum_{i=1}^n \bigl(1-p_\theta(x_i)_{y_i}\bigr).
\label{eq:ece_equals_mean_misconf_clean}
\end{equation}
Recalling $Q_{\mathcal D}(\theta):=\frac{1}{n}\sum_{i=1}^n e^{-m_\theta(x_i,y_i)}$, Lemma~\ref{lem:tail_by_margin}
implies (since $m_\theta(x_i,y_i)\ge 0$ for all $i$) that
\[
\frac{1}{K}\,e^{-m_\theta(x_i,y_i)}
\le
1-p_\theta(x_i)_{y_i}
\le
(K-1)\,e^{-m_\theta(x_i,y_i)}.
\]
Averaging over $i$ and using~\eqref{eq:ece_equals_mean_misconf_clean} yields
\[
\frac{1}{K}\,Q_{\mathcal D}(\theta)\le \mathrm{ECE}_M(\theta;\mathcal{D})\le (K-1)\,Q_{\mathcal D}(\theta).
\]
Finally, $m_\theta(x_i,y_i)\ge \gamma(\theta;\mathcal{D})$ implies $Q_{\mathcal D}(\theta)\le e^{-\gamma(\theta;\mathcal{D})}$.

\textbf{(ii) GN curvature bound.}
For each $i$, define $H_{z,i}:=H_z(p_\theta(x_i))$ and $J_i:=J_\theta(x_i)$.
Then
\[
H_{\mathrm{GN}}(\theta;\mathcal{D})=\frac{1}{n}\sum_{i=1}^n J_i^\top H_{z,i}J_i \succeq 0.
\]
Since $\lambda_{\max}$ is convex on the PSD cone,
\[
\lambda_{\max}\big(H_{\mathrm{GN}}(\theta;\mathcal{D})\big)
\le
\frac{1}{n}\sum_{i=1}^n \lambda_{\max}(J_i^\top H_{z,i}J_i)
\le
\frac{1}{n}\sum_{i=1}^n \|J_i\|_{\mathrm{op}}^2\,\lambda_{\max}(H_{z,i}).
\]
Under $\|J_i\|_{\mathrm{op}}\le C_J$ for all $i$,
\[
\lambda_{\max}\big(H_{\mathrm{GN}}(\theta;\mathcal{D})\big)
\le
C_J^2\cdot \frac{1}{n}\sum_{i=1}^n \lambda_{\max}(H_{z,i}).
\]
By Lemma~\ref{lem:hz_eig_upper}, $\lambda_{\max}(H_{z,i})\le 2(1-p_{\max,i})$.
Since $\widehat{Y}_i=y_i$, we have $p_{\max,i}=p_\theta(x_i)_{y_i}$, hence by Lemma~\ref{lem:tail_by_margin},
\[
\lambda_{\max}(H_{z,i})
\le 2\bigl(1-p_\theta(x_i)_{y_i}\bigr)
\le 2(K-1)e^{-m_\theta(x_i,y_i)}.
\]
Therefore
\[
\lambda_{\max}\big(H_{\mathrm{GN}}(\theta;\mathcal{D})\big)
\le
2C_J^2(K-1)\cdot \frac{1}{n}\sum_{i=1}^n e^{-m_\theta(x_i,y_i)}
=
2C_J^2(K-1)\,Q_{\mathcal D}(\theta).
\]

\textbf{(iii) Coupling to $\mathrm{ECE}_M$ (rearranged lower bound).}
Combining the bound in (ii) with $\mathrm{ECE}_M(\theta;\mathcal{D})\ge \frac{1}{K}Q_{\mathcal D}(\theta)$ from (i) yields
\[
\lambda_{\max}\big(H_{\mathrm{GN}}(\theta;\mathcal{D})\big)
\le 2C_J^2K(K-1)\,\mathrm{ECE}_M(\theta;\mathcal{D}),
\qquad\text{equivalently}\qquad
\mathrm{ECE}_M(\theta;\mathcal{D})
\ge \frac{\lambda_{\max}(H_{\mathrm{GN}}(\theta;\mathcal{D}))}{2C_J^2K(K-1)}.
\]

\textbf{(iv) Robust-margin variant.}
Assume that for each $(x_i,y_i)$ there exists $L_m(x_i,y_i)\in[0,\infty)$ such that
\[
|m_\theta(x_i+\delta,y_i)-m_\theta(x_i,y_i)|\le L_m(x_i,y_i)\|\delta\|
\qquad \forall \|\delta\|\le \varepsilon.
\]
Define the robust moments (as in Appendix~E.3)
\[
Q^0_{\varepsilon,\mathcal D}(\theta):=\frac{1}{n}\sum_{i=1}^n e^{-m_{\varepsilon,\theta}(x_i,y_i)},
\qquad
Q^-_{\varepsilon,\mathcal D}(\theta):=\frac{1}{n}\sum_{i=1}^n e^{-\varepsilon L_m(x_i,y_i)}e^{-m_{\varepsilon,\theta}(x_i,y_i)}.
\]
By Lemma~\ref{lem:robust_exp_compare},
\[
e^{-m_\theta(x_i,y_i)} \ge e^{-\varepsilon L_m(x_i,y_i)}e^{-m_{\varepsilon,\theta}(x_i,y_i)}
\qquad\text{and}\qquad
e^{-m_\theta(x_i,y_i)} \le e^{-m_{\varepsilon,\theta}(x_i,y_i)}.
\]
Insert these bounds into $\frac{1}{K}Q_{\mathcal D}(\theta)\le \mathrm{ECE}_M(\theta;\mathcal{D})\le (K-1)Q_{\mathcal D}(\theta)$ to obtain
\[
\frac{1}{K}Q^-_{\varepsilon,\mathcal D}(\theta)
\le
\mathrm{ECE}_M(\theta;\mathcal{D})
\le
(K-1)Q^0_{\varepsilon,\mathcal D}(\theta).
\]
For $\lambda_{\max}(H_{\mathrm{GN}}(\theta;\mathcal D))$, repeat the argument in (ii) and apply
$e^{-m_\theta(x_i,y_i)}\le e^{-m_{\varepsilon,\theta}(x_i,y_i)}$ in the final step to get
\[
\lambda_{\max}\big(H_{\mathrm{GN}}(\theta;\mathcal{D})\big)\le 2C_J^2(K-1)\,Q^0_{\varepsilon,\mathcal D}(\theta).
\]
\end{proof}

%% file: sections/appendix/all_plots.tex
\section{Additional Sharpness--Calibration Experiments}
\label{sec:additional_plots}

\subsection{Sharpness--Calibration Correlation Analysis}
\label{sec:correlation_plots}

We present detailed training dynamics for each optimizer on CIFAR-10 and CIFAR-100, showing the co-evolution of loss, accuracy, ECE, margin, and sharpness throughout training. Since computing the full Hessian eigenvalue is expensive, these experiments use an MLP with a 5K/5K train/validation split. For each dataset, a scatter summary visualizes the temporal coupling across optimizers in a single view; per-optimizer figures report training (left) and validation (right) metrics across learning rates.

\begin{figure}[htbp]
    \centering
    \includegraphics[width=0.8\linewidth]{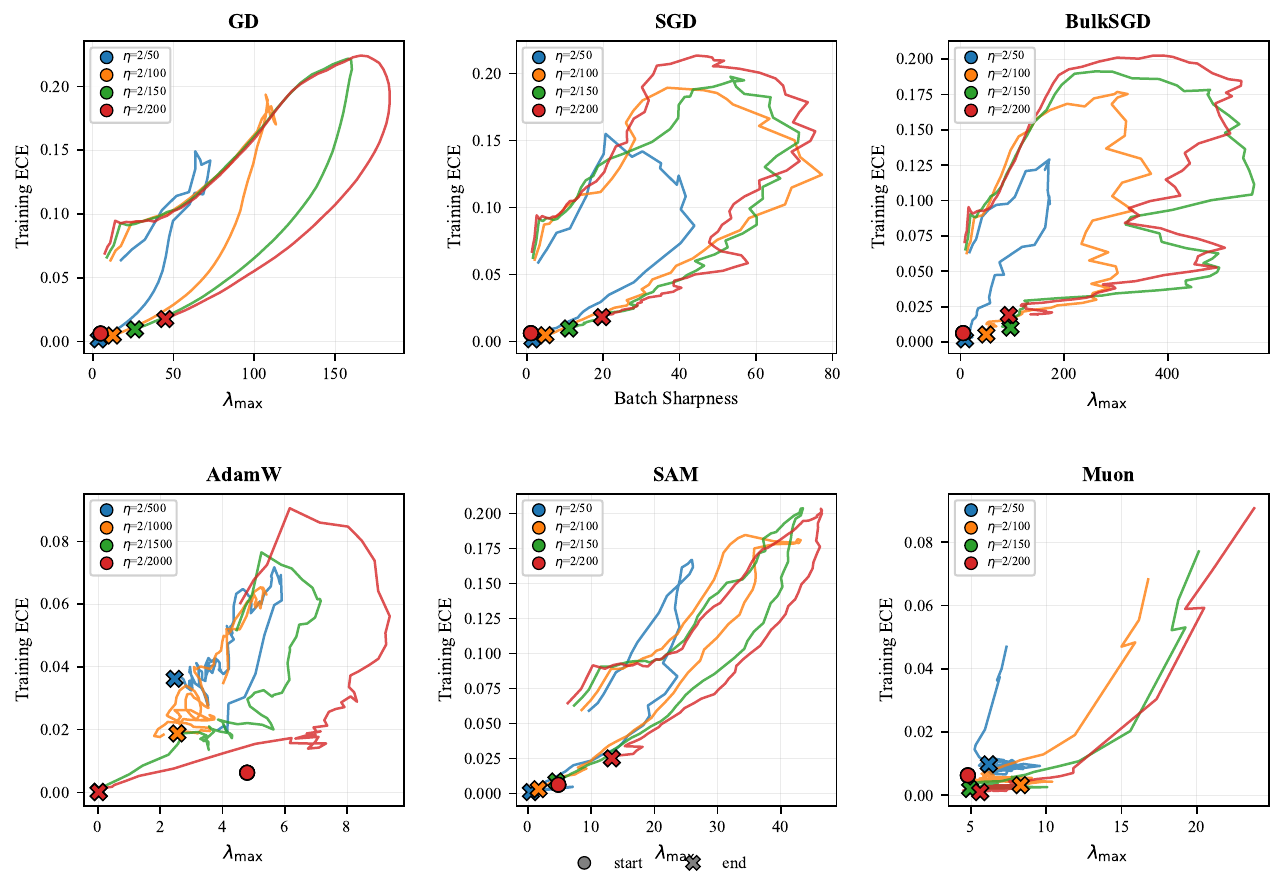}
    \caption{\textbf{ECE vs.\ GN sharpness trajectories (CIFAR-10).} Each curve traces the joint evolution of ECE and GN sharpness ($\lambda_{\max}$) across training steps for one optimizer and learning rate, with a filled circle marking the first training step and a cross ($\times$) the last; color encodes learning rate. Trajectories lie near the diagonal, visualizing the temporal coupling between the two quantities.}
    \label{fig:ece_sharpness_scatter}
\end{figure}

% GD
\begin{figure}[htbp]
    \centering
    \begin{subfigure}[b]{0.48\textwidth}
        \centering
        \includegraphics[width=\textwidth]{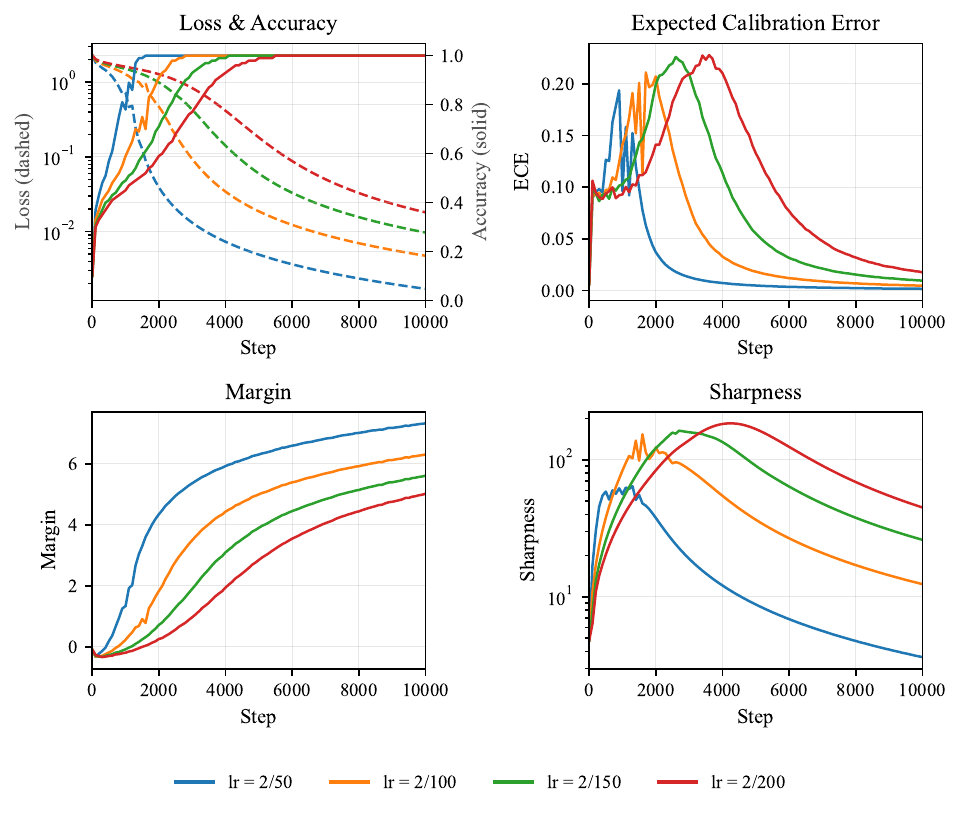}
        \caption{Training metrics}
    \end{subfigure}
    \hfill
    \begin{subfigure}[b]{0.48\textwidth}
        \centering
        \includegraphics[width=\textwidth]{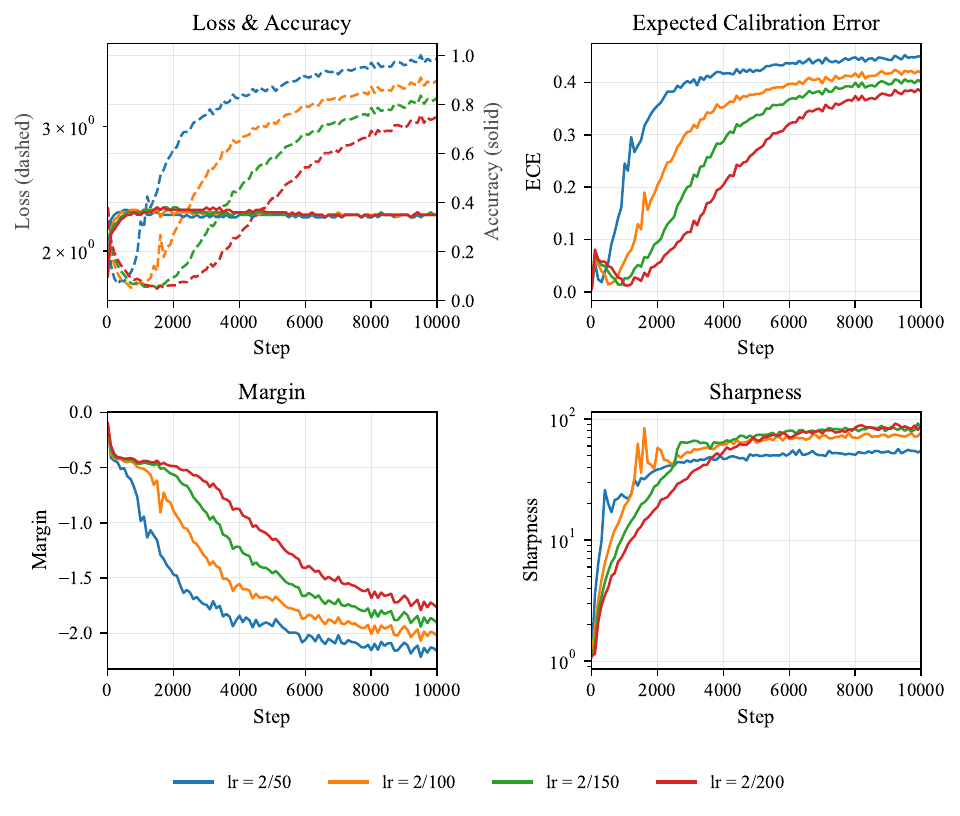}
        \caption{Validation metrics}
    \end{subfigure}
    \caption{\textbf{Gradient Descent (GD).} Training dynamics (loss, accuracy, ECE, margin, sharpness) across four learning rates on CIFAR-10; training (left) and validation (right).}
    \label{fig:gd}
\end{figure}

% SGD
\begin{figure}[htbp]
    \centering
    \begin{subfigure}[b]{0.48\textwidth}
        \centering
        \includegraphics[width=\textwidth]{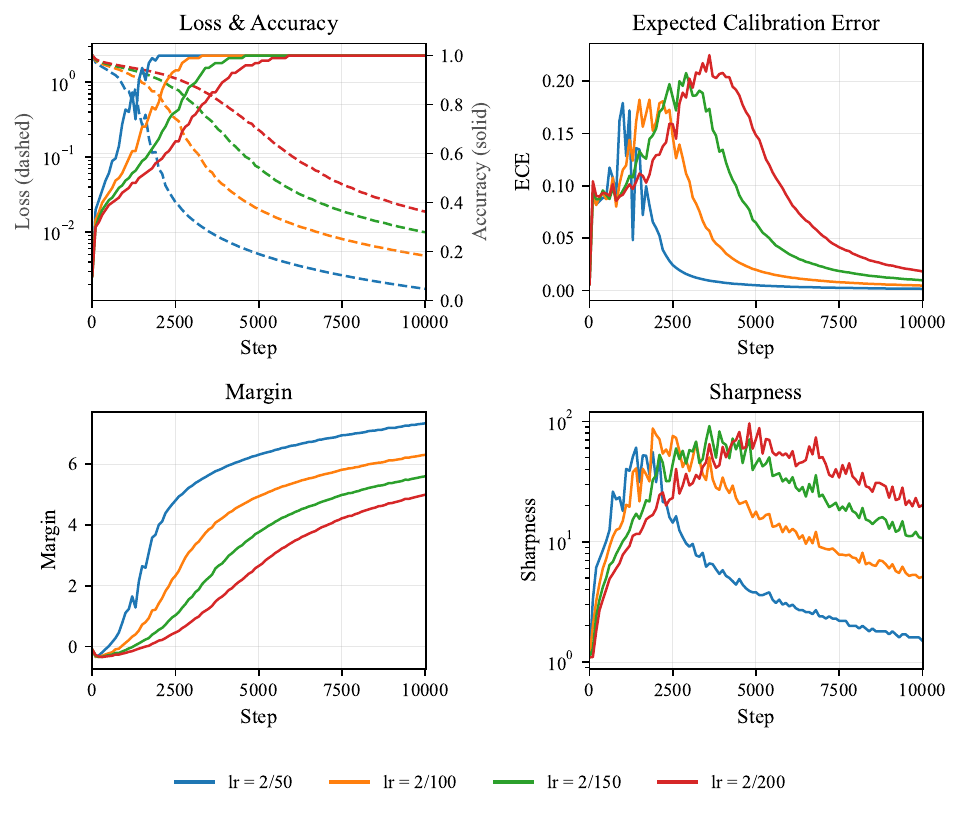}
        \caption{Training metrics}
    \end{subfigure}
    \hfill
    \begin{subfigure}[b]{0.48\textwidth}
        \centering
        \includegraphics[width=\textwidth]{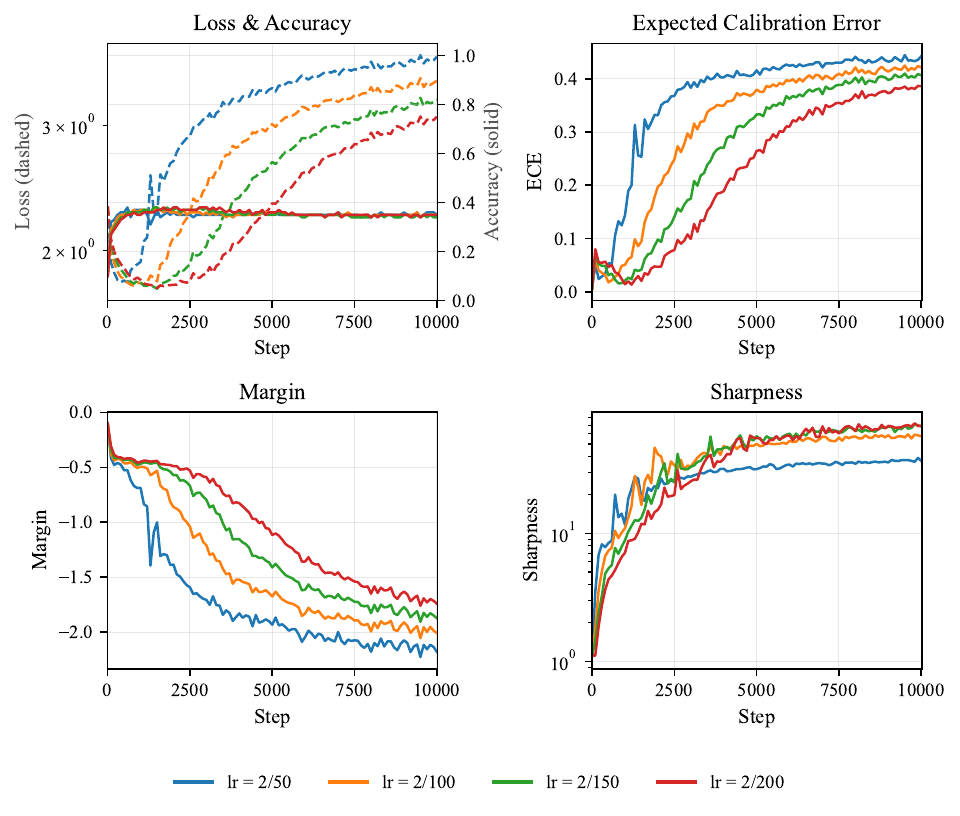}
        \caption{Validation metrics}
    \end{subfigure}
    \caption{\textbf{Stochastic Gradient Descent (SGD).} Training dynamics (loss, accuracy, ECE, margin, sharpness) across four learning rates on CIFAR-10; training (left) and validation (right).}
    \label{fig:sgd}
\end{figure}

% SAM
\begin{figure}[htbp]
    \centering
    \begin{subfigure}[b]{0.48\textwidth}
        \centering
        \includegraphics[width=\textwidth]{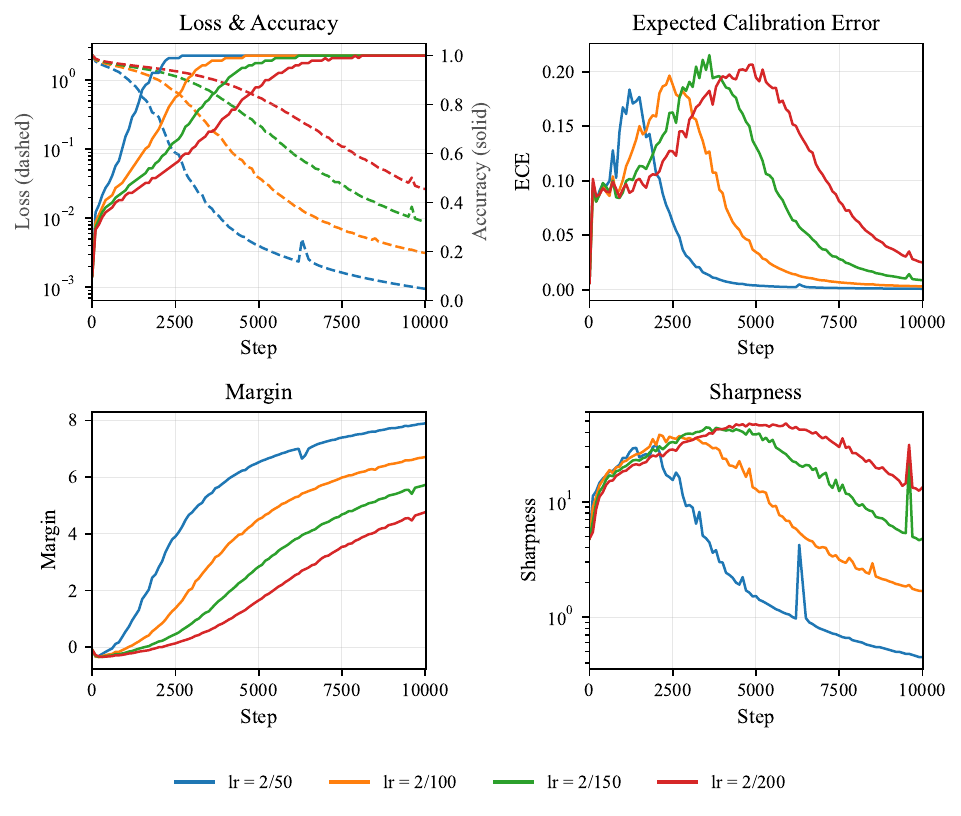}
        \caption{Training metrics}
    \end{subfigure}
    \hfill
    \begin{subfigure}[b]{0.48\textwidth}
        \centering
        \includegraphics[width=\textwidth]{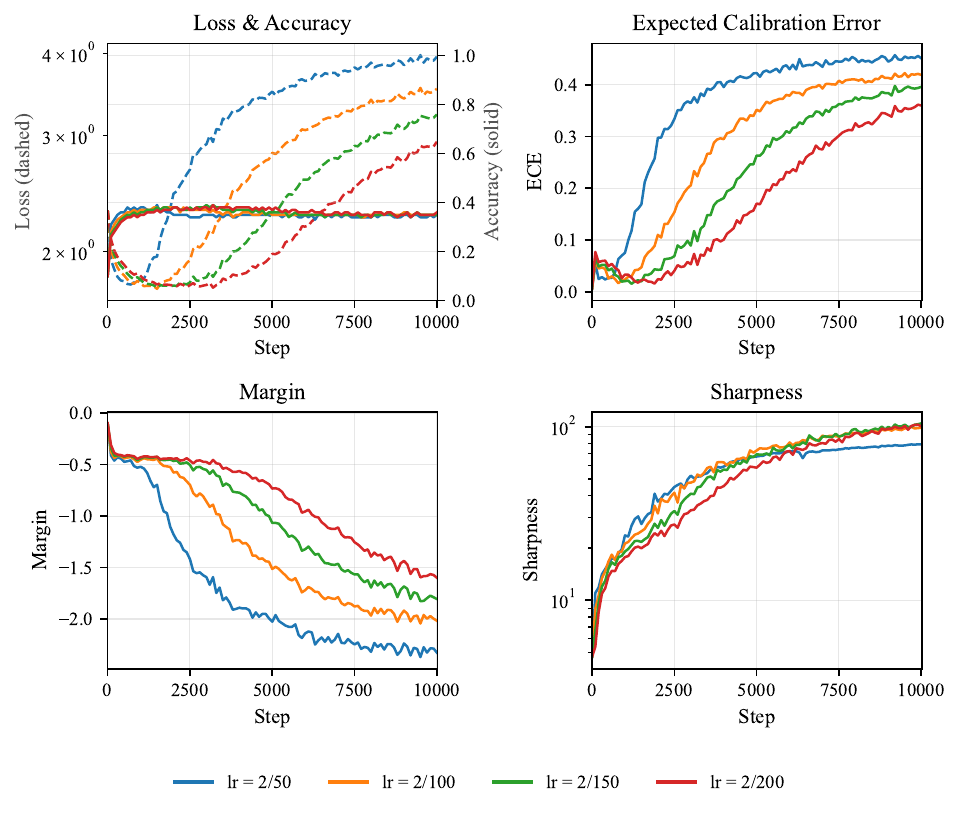}
        \caption{Validation metrics}
    \end{subfigure}
    \caption{\textbf{Sharpness-Aware Minimization (SAM).} Training dynamics (loss, accuracy, ECE, margin, sharpness) across four learning rates on CIFAR-10; training (left) and validation (right).}
    \label{fig:sam}
\end{figure}

% Muon
\begin{figure}[htbp]
    \centering
    \begin{subfigure}[b]{0.48\textwidth}
        \centering
        \includegraphics[width=\textwidth]{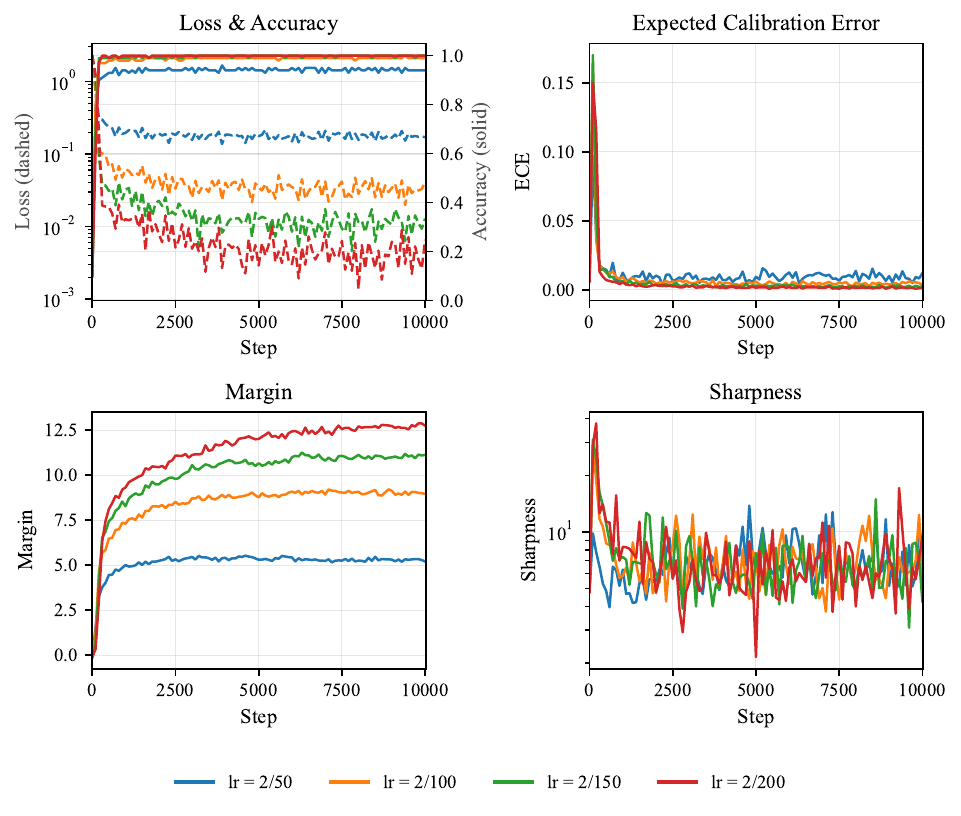}
        \caption{Training metrics}
    \end{subfigure}
    \hfill
    \begin{subfigure}[b]{0.48\textwidth}
        \centering
        \includegraphics[width=\textwidth]{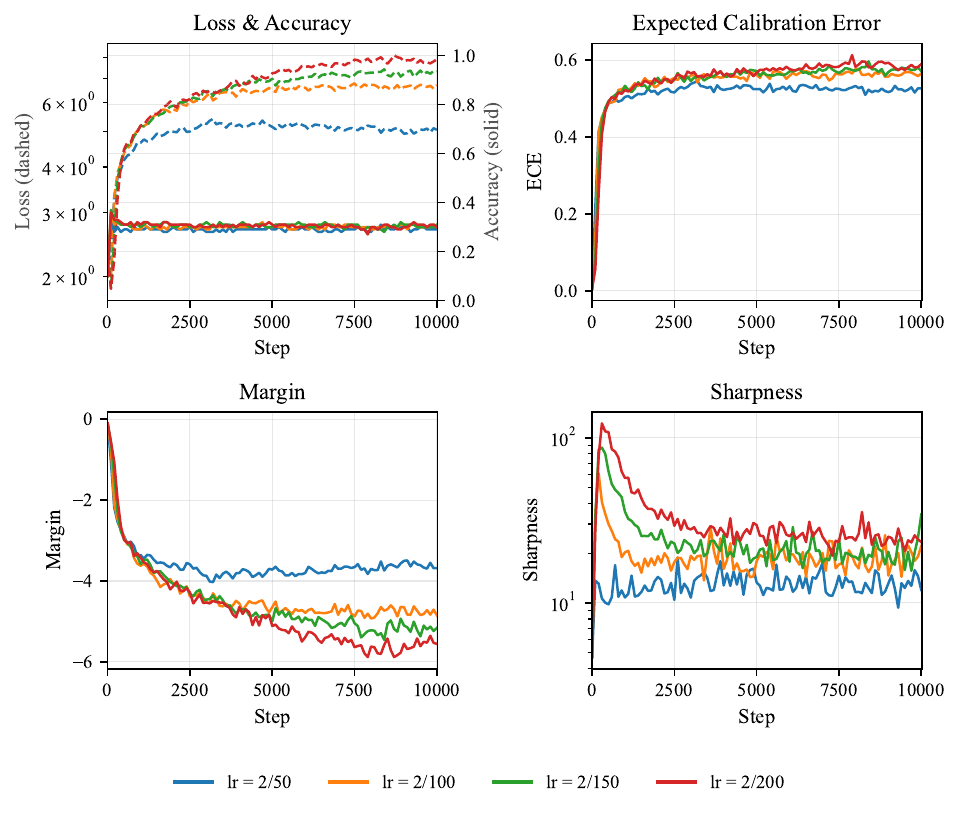}
        \caption{Validation metrics}
    \end{subfigure}
    \caption{\textbf{Muon.} Training dynamics (loss, accuracy, ECE, margin, sharpness) across four learning rates on CIFAR-10; training (left) and validation (right).}
    \label{fig:muon}
\end{figure}

% AdamW
\begin{figure}[htbp]
    \centering
    \begin{subfigure}[b]{0.48\textwidth}
        \centering
        \includegraphics[width=\textwidth]{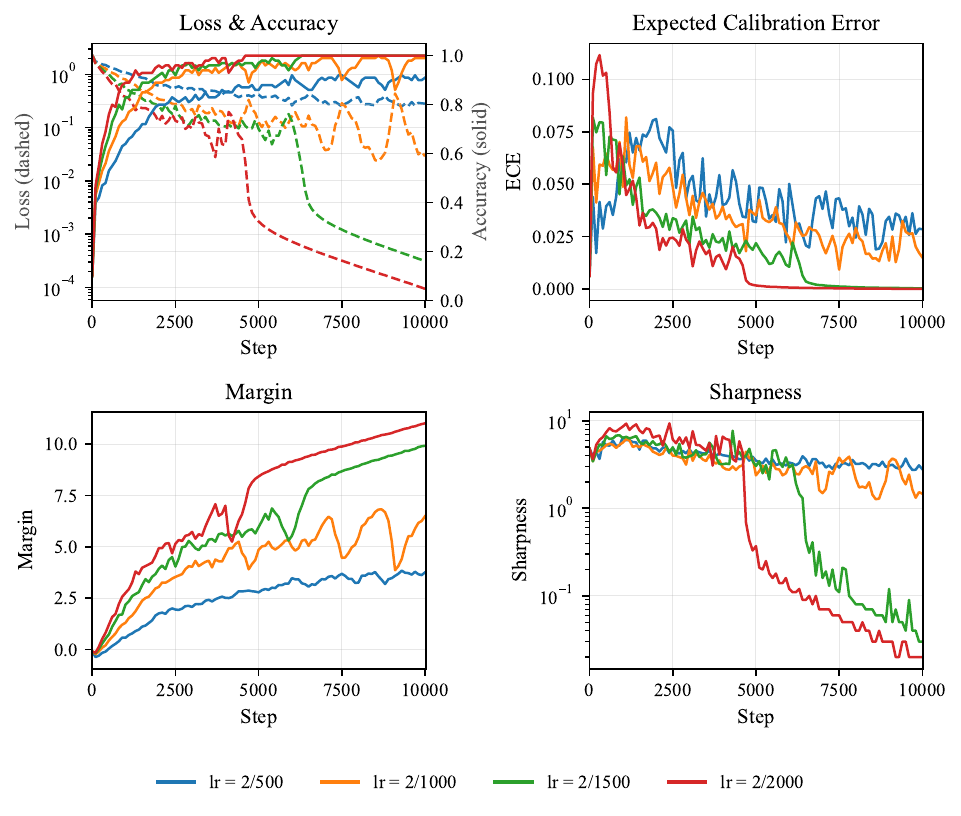}
        \caption{Training metrics}
    \end{subfigure}
    \hfill
    \begin{subfigure}[b]{0.48\textwidth}
        \centering
        \includegraphics[width=\textwidth]{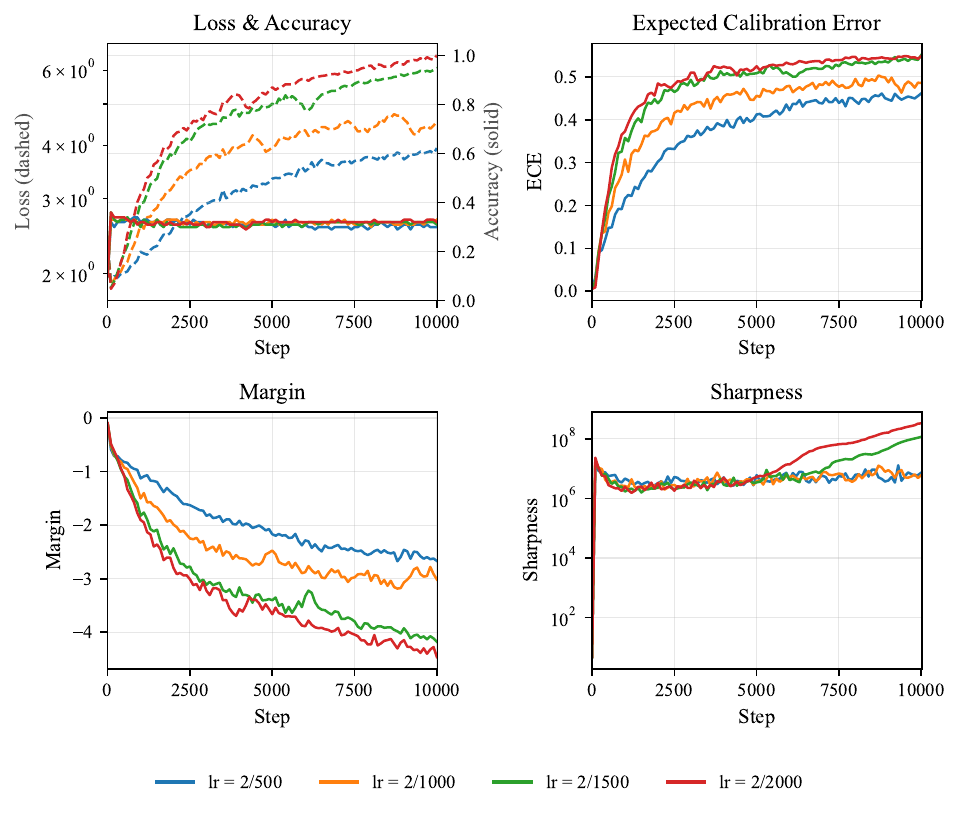}
        \caption{Validation metrics}
    \end{subfigure}
    \caption{\textbf{AdamW.} Training dynamics (loss, accuracy, ECE, margin, sharpness) across four learning rates on CIFAR-10; training (left) and validation (right).}
    \label{fig:adamw}
\end{figure}

% BulkSGD
\begin{figure}[htbp]
    \centering
    \begin{subfigure}[b]{0.48\textwidth}
        \centering
        \includegraphics[width=\textwidth]{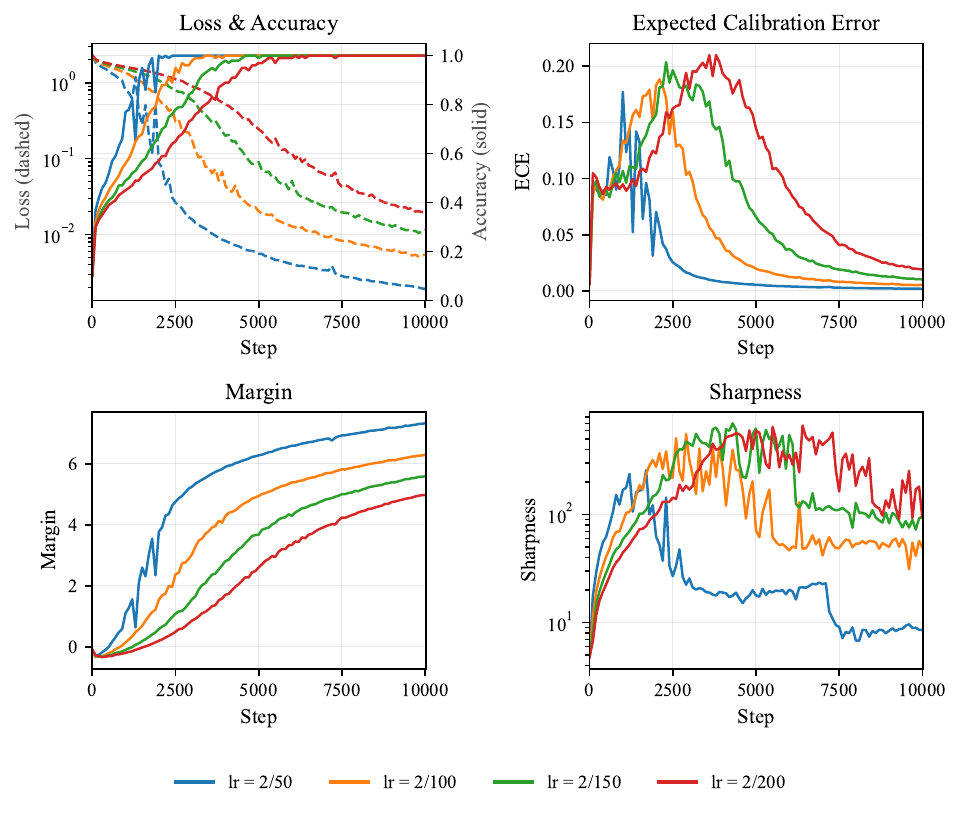}
        \caption{Training metrics}
    \end{subfigure}
    \hfill
    \begin{subfigure}[b]{0.48\textwidth}
        \centering
        \includegraphics[width=\textwidth]{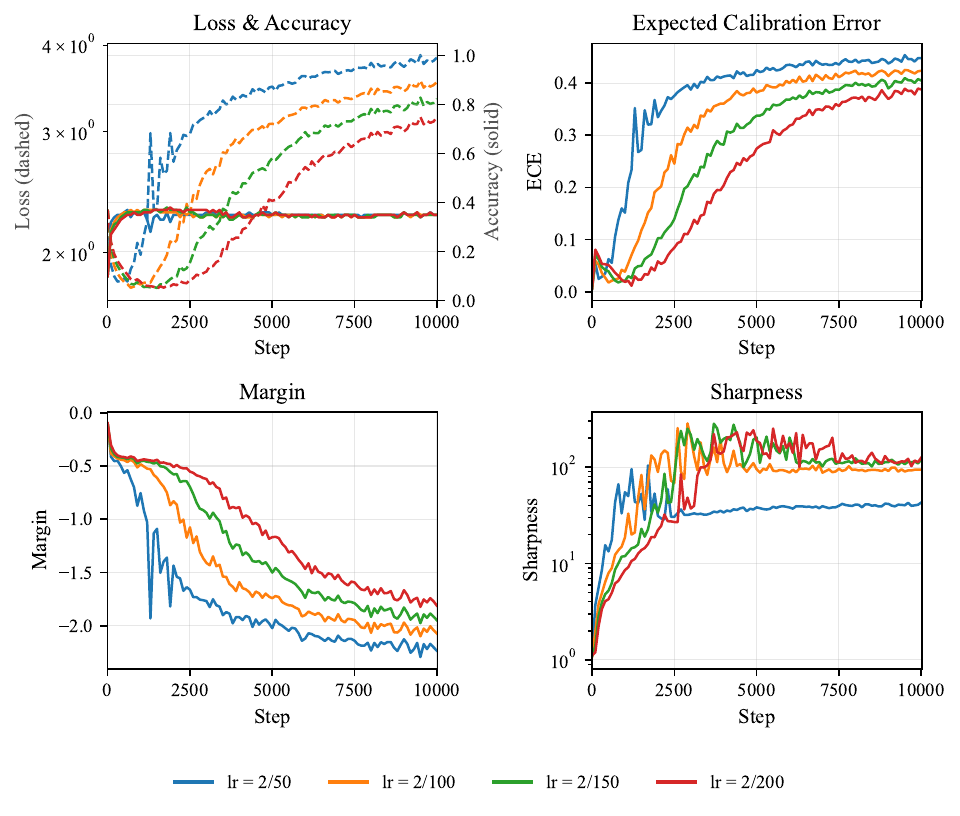}
        \caption{Validation metrics}
    \end{subfigure}
    \caption{\textbf{BulkSGD.} Training dynamics (loss, accuracy, ECE, margin, sharpness) across four learning rates on CIFAR-10; training (left) and validation (right).}
    \label{fig:bulksgd}
\end{figure}

\begin{figure}[htbp]
    \centering
    \includegraphics[width=0.8\linewidth]{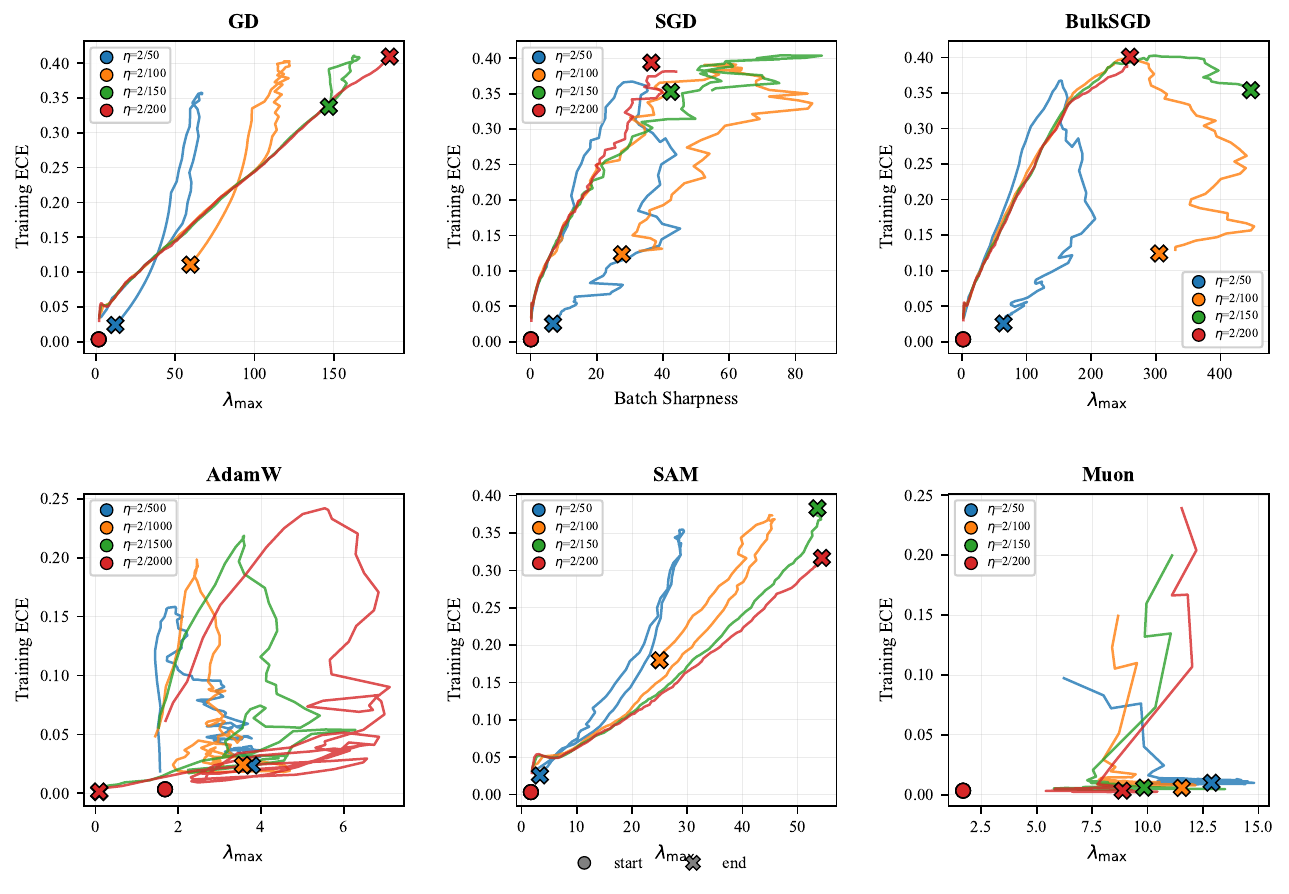}
    \caption{\textbf{ECE vs.\ GN sharpness trajectories (CIFAR-100).} Same format as Figure~\ref{fig:ece_sharpness_scatter}.}
    \label{fig:ece_sharpness_scatter_cifar100}
\end{figure}

% GD (CIFAR-100)
\begin{figure}[htbp]
    \centering
    \begin{subfigure}[b]{0.48\textwidth}
        \centering
        \includegraphics[width=\textwidth]{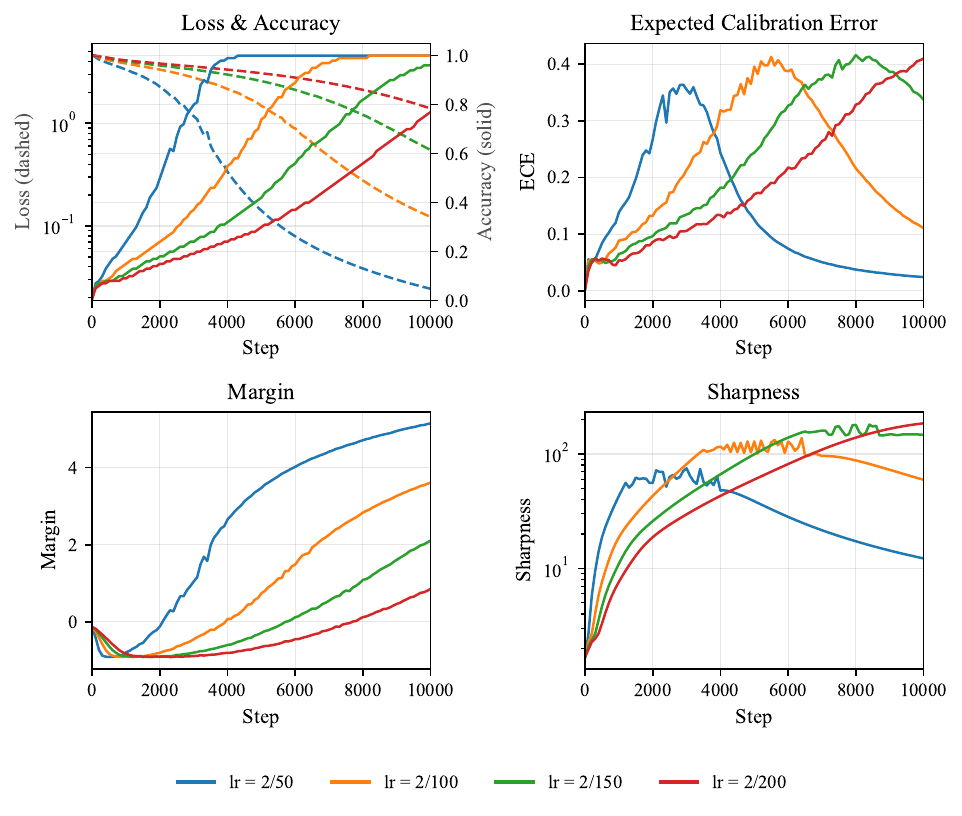}
        \caption{Training metrics}
    \end{subfigure}
    \hfill
    \begin{subfigure}[b]{0.48\textwidth}
        \centering
        \includegraphics[width=\textwidth]{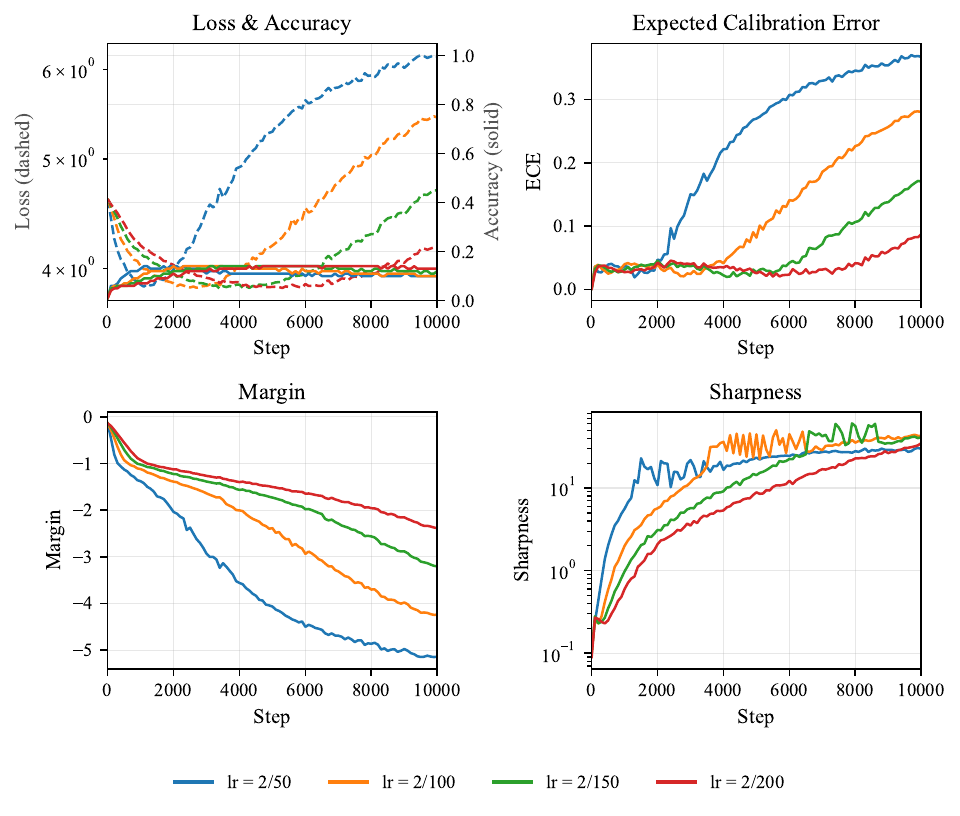}
        \caption{Validation metrics}
    \end{subfigure}
    \caption{\textbf{Gradient Descent (GD) --- CIFAR-100.} Training dynamics (loss, accuracy, ECE, margin, sharpness) across learning rates on CIFAR-100; training (left) and validation (right).}
    \label{fig:gd_cifar100}
\end{figure}

% SGD (CIFAR-100)
\begin{figure}[htbp]
    \centering
    \begin{subfigure}[b]{0.48\textwidth}
        \centering
        \includegraphics[width=\textwidth]{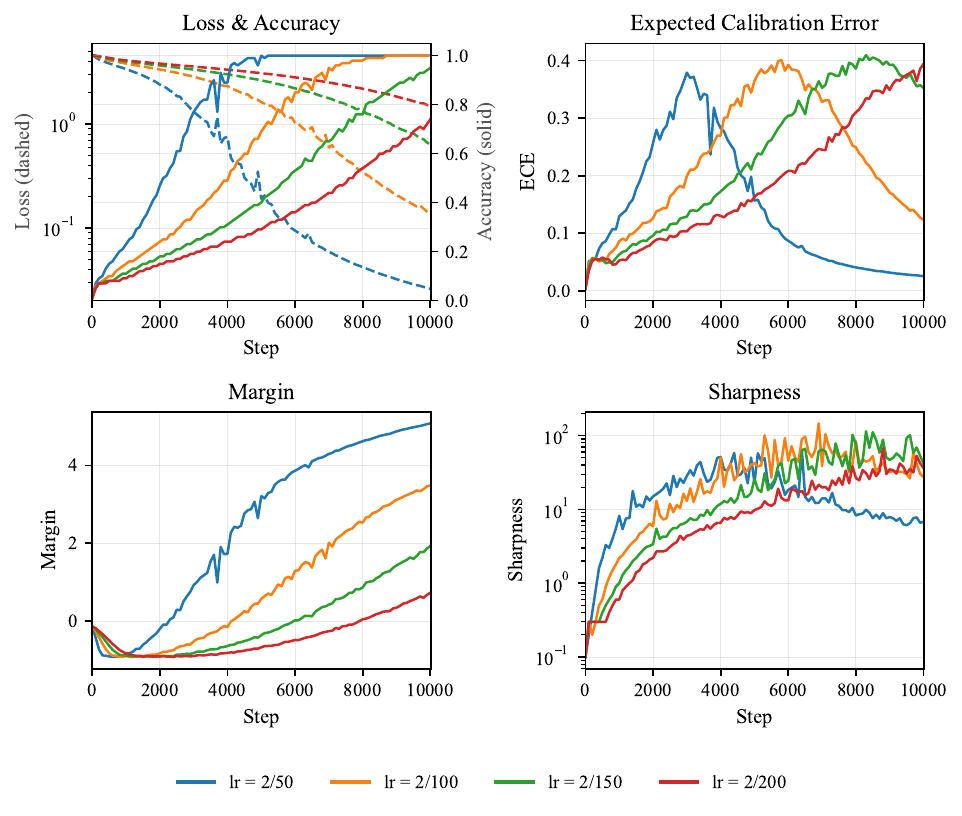}
        \caption{Training metrics}
    \end{subfigure}
    \hfill
    \begin{subfigure}[b]{0.48\textwidth}
        \centering
        \includegraphics[width=\textwidth]{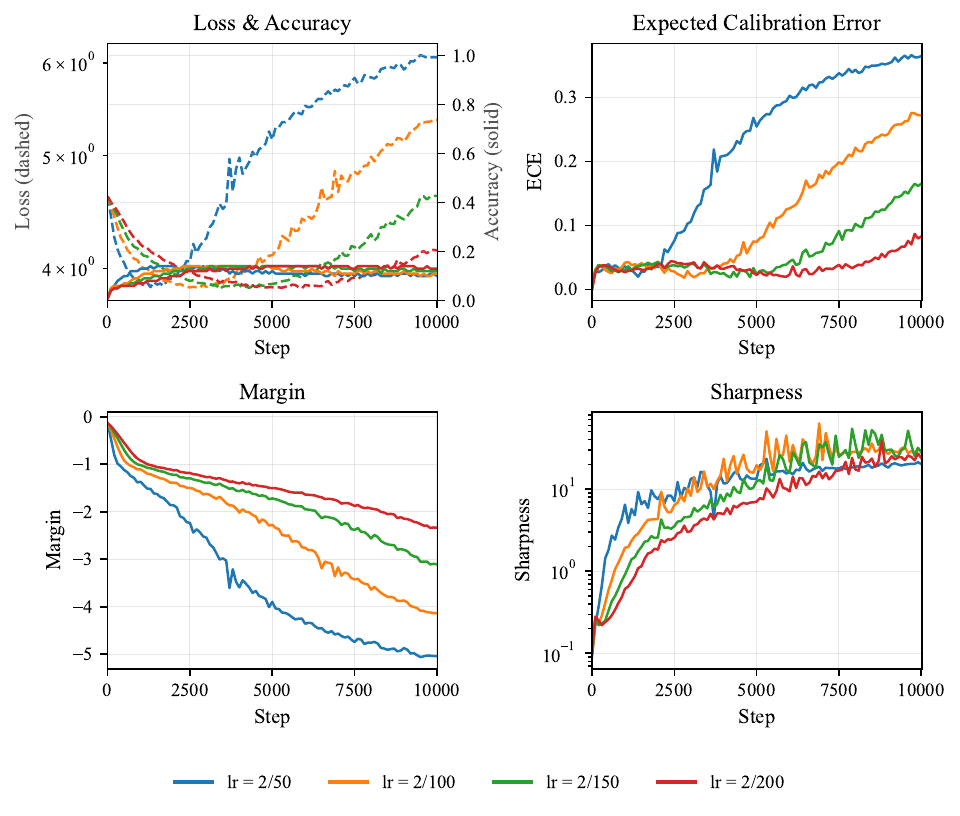}
        \caption{Validation metrics}
    \end{subfigure}
    \caption{\textbf{Stochastic Gradient Descent (SGD) --- CIFAR-100.} Training dynamics (loss, accuracy, ECE, margin, sharpness) across learning rates on CIFAR-100; training (left) and validation (right).}
    \label{fig:sgd_cifar100}
\end{figure}

% SAM (CIFAR-100)
\begin{figure}[htbp]
    \centering
    \begin{subfigure}[b]{0.48\textwidth}
        \centering
        \includegraphics[width=\textwidth]{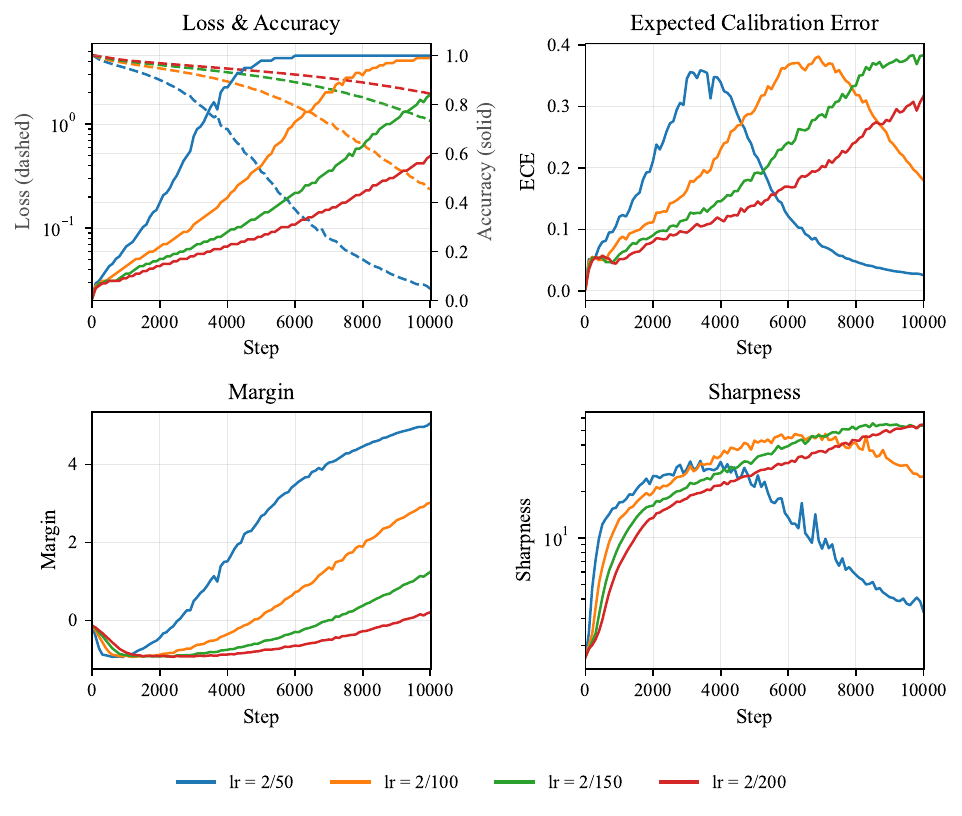}
        \caption{Training metrics}
    \end{subfigure}
    \hfill
    \begin{subfigure}[b]{0.48\textwidth}
        \centering
        \includegraphics[width=\textwidth]{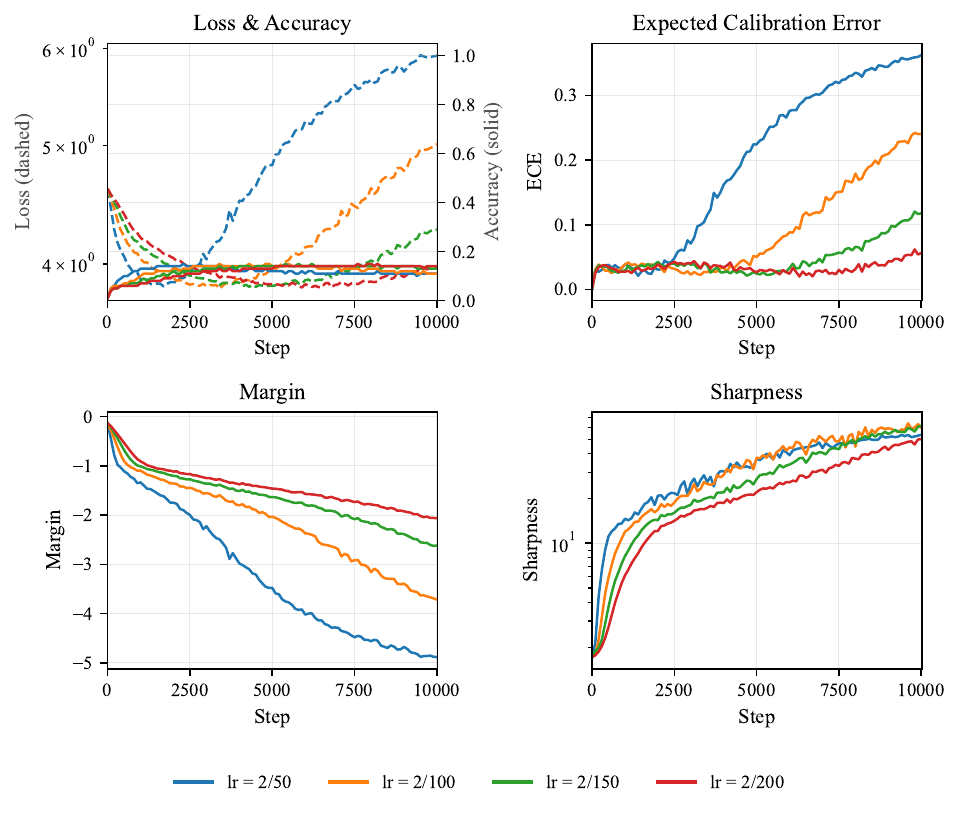}
        \caption{Validation metrics}
    \end{subfigure}
    \caption{\textbf{Sharpness-Aware Minimization (SAM) --- CIFAR-100.} Training dynamics (loss, accuracy, ECE, margin, sharpness) across learning rates on CIFAR-100; training (left) and validation (right).}
    \label{fig:sam_cifar100}
\end{figure}

% Muon (CIFAR-100)
\begin{figure}[htbp]
    \centering
    \begin{subfigure}[b]{0.48\textwidth}
        \centering
        \includegraphics[width=\textwidth]{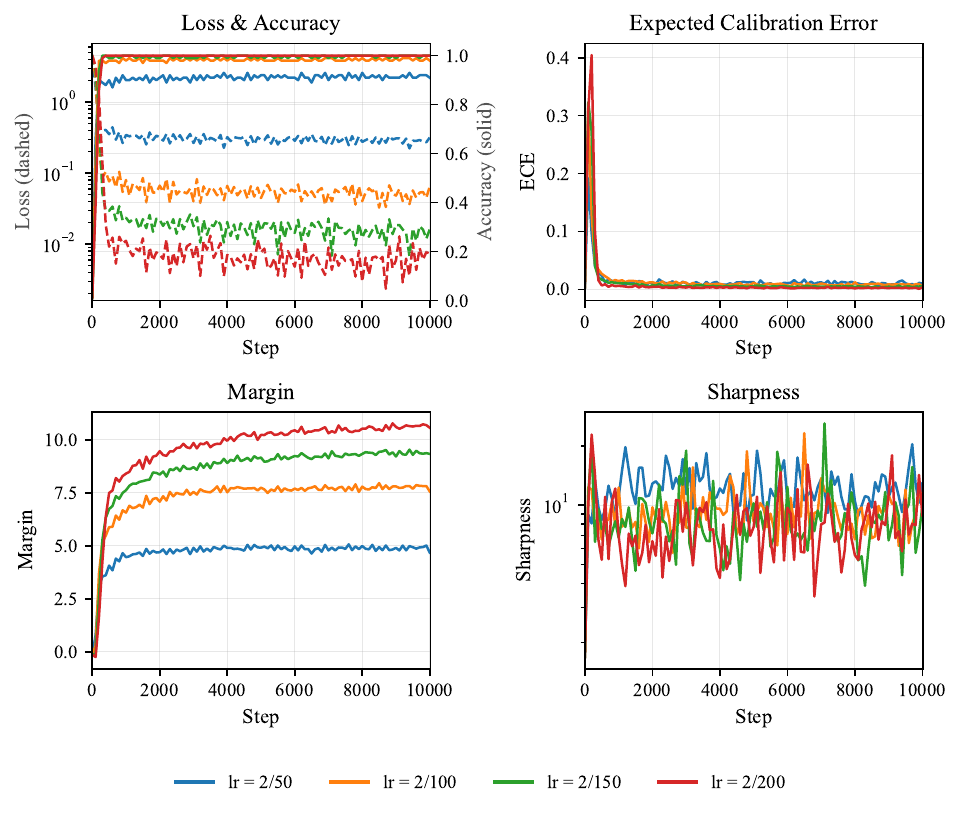}
        \caption{Training metrics}
    \end{subfigure}
    \hfill
    \begin{subfigure}[b]{0.48\textwidth}
        \centering
        \includegraphics[width=\textwidth]{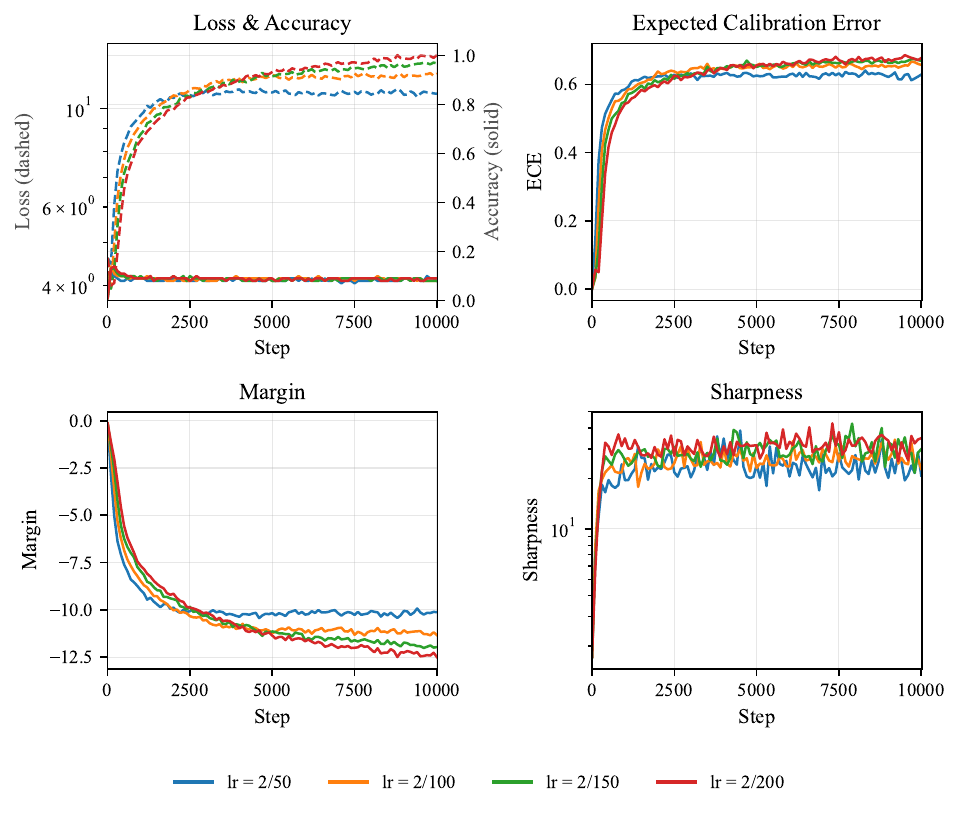}
        \caption{Validation metrics}
    \end{subfigure}
    \caption{\textbf{Muon --- CIFAR-100.} Training dynamics (loss, accuracy, ECE, margin, sharpness) across learning rates on CIFAR-100; training (left) and validation (right).}
    \label{fig:muon_cifar100}
\end{figure}

% AdamW (CIFAR-100)
\begin{figure}[htbp]
    \centering
    \begin{subfigure}[b]{0.48\textwidth}
        \centering
        \includegraphics[width=\textwidth]{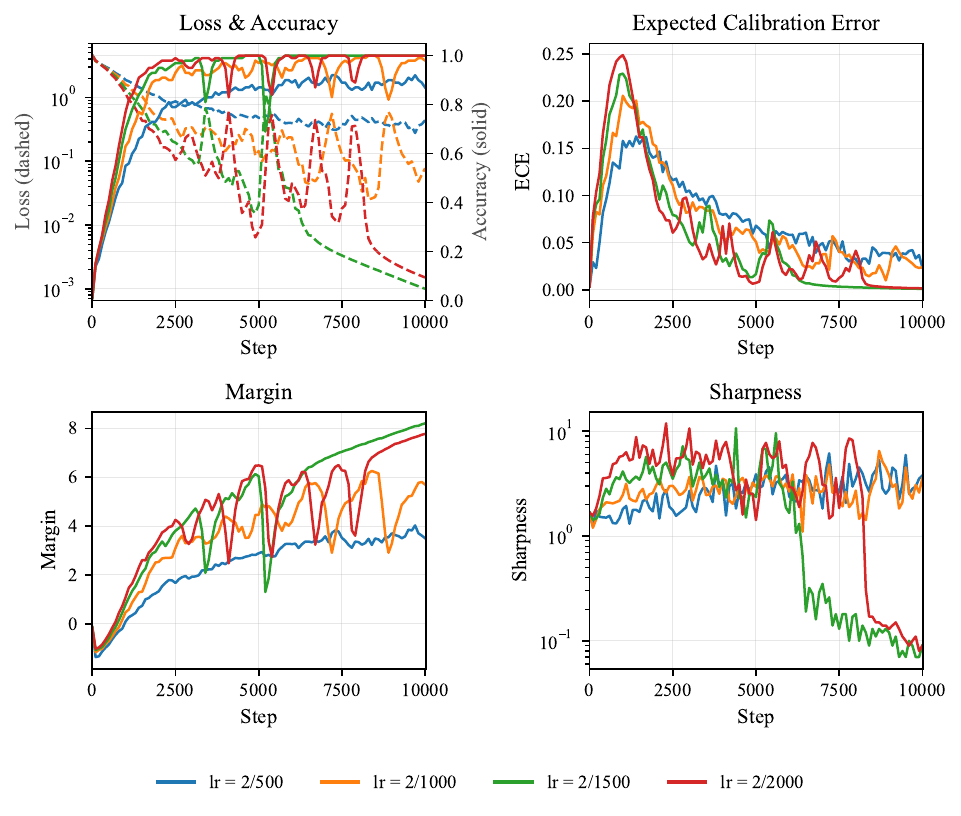}
        \caption{Training metrics}
    \end{subfigure}
    \hfill
    \begin{subfigure}[b]{0.48\textwidth}
        \centering
        \includegraphics[width=\textwidth]{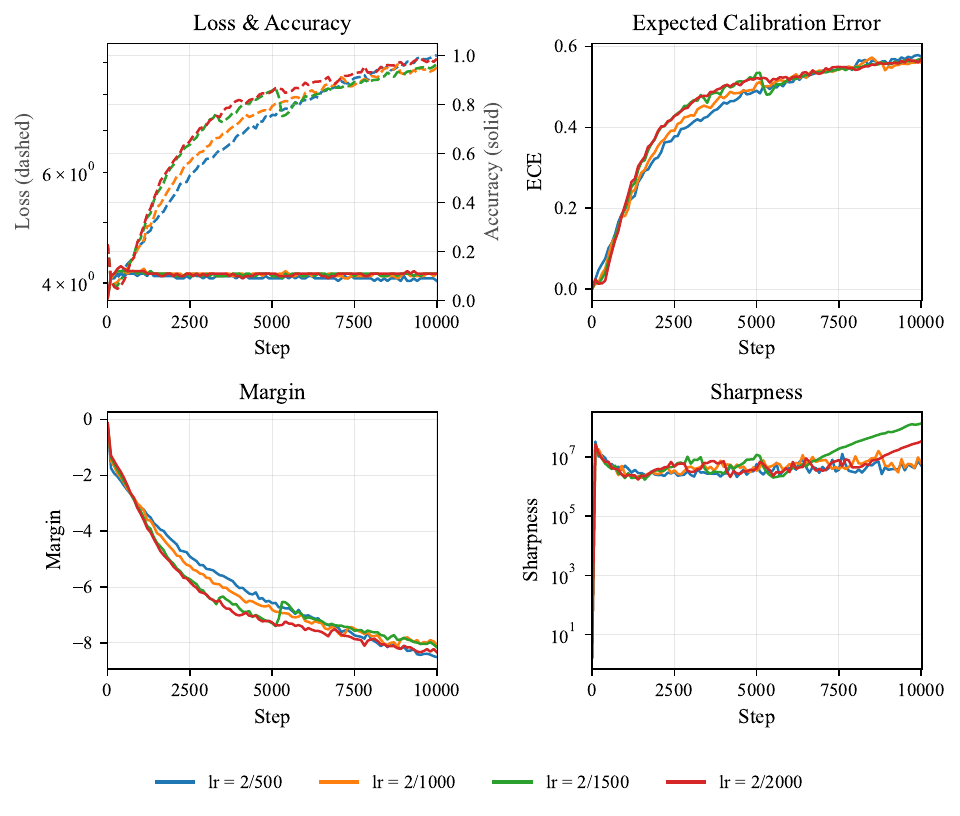}
        \caption{Validation metrics}
    \end{subfigure}
    \caption{\textbf{AdamW --- CIFAR-100.} Training dynamics (loss, accuracy, ECE, margin, sharpness) across learning rates on CIFAR-100; training (left) and validation (right).}
    \label{fig:adamw_cifar100}
\end{figure}

% BulkSGD (CIFAR-100)
\begin{figure}[htbp]
    \centering
    \begin{subfigure}[b]{0.48\textwidth}
        \centering
        \includegraphics[width=\textwidth]{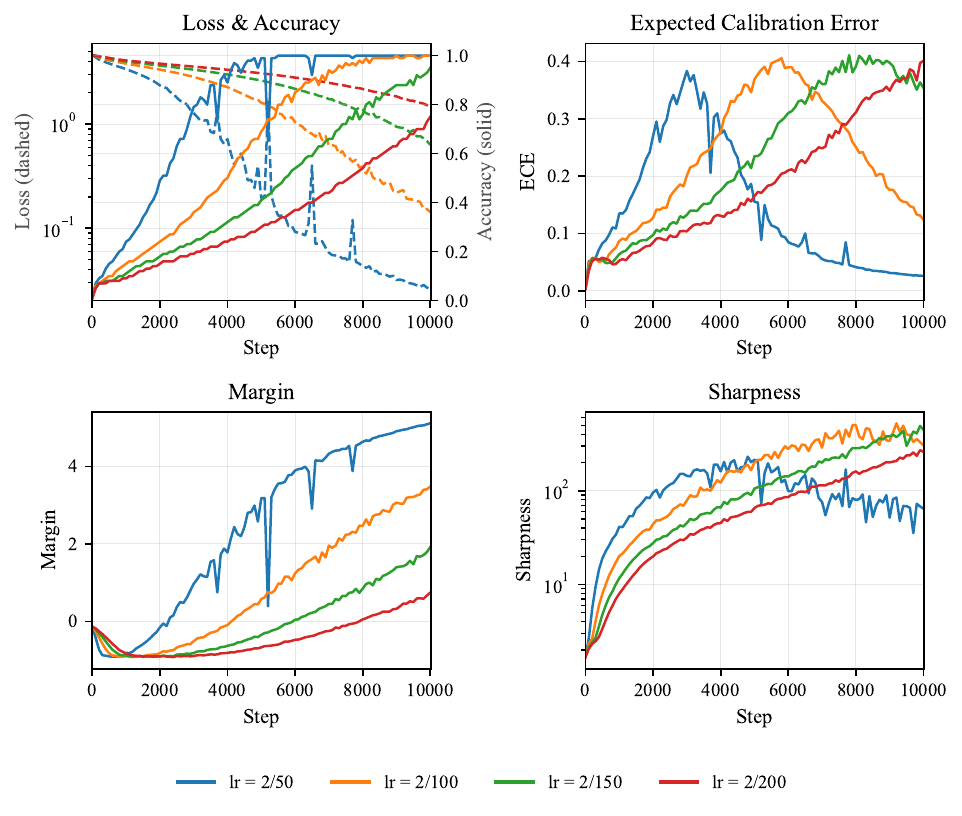}
        \caption{Training metrics}
    \end{subfigure}
    \hfill
    \begin{subfigure}[b]{0.48\textwidth}
        \centering
        \includegraphics[width=\textwidth]{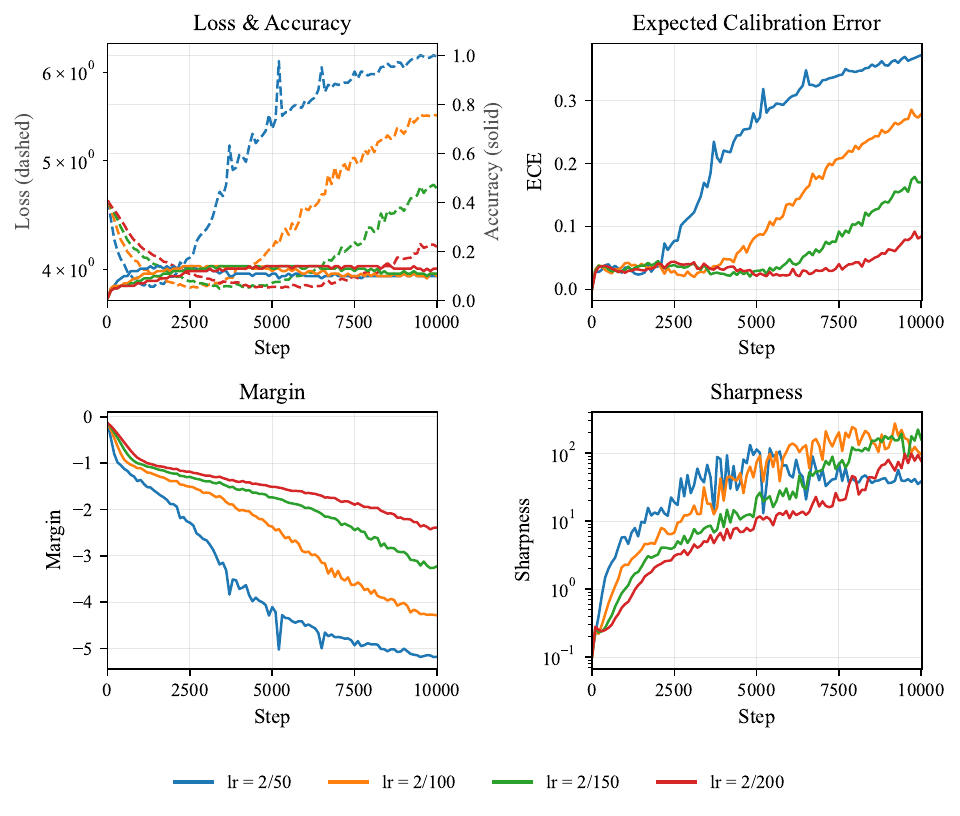}
        \caption{Validation metrics}
    \end{subfigure}
    \caption{\textbf{BulkSGD --- CIFAR-100.} Training dynamics (loss, accuracy, ECE, margin, sharpness) across learning rates on CIFAR-100; training (left) and validation (right).}
    \label{fig:bulksgd_cifar100}
\end{figure}

\subsection{Optimizer Details: SAM, Muon, and BulkSGD}
\label{app:sam_muon_bulk}
There is literature to support the notion that SAM may lead to improved calibration metrics, specifically that SAM act as an implicit regularizer and therefore prevents overfitting during training \cite{tan2025samcalibration}. At every step, SAM solves $$\min_{w} \max_{||\epsilon|| \leq \rho}L(w+\epsilon)$$
which explicitly penalizes the sharpness of the Hessian and leads to convergence to flatter minima \cite{zhou2024sharpness}. We train networks using SAM to test the first hypothesis, looking to confirm that flat minima lead to lower calibration error.

To test the second hypothesis, we apply optimizers that explicitly suppress the contribution of eigenvectors associated with directions of steep descent, through Muon and BulkSGD. Muon rescales the gradient components at each update, so all directions contribute with comparable magnitude. This means directions of steep descent are clamped, while the flatter directions are amplified \cite{jordan2024muon}. Another method to suppress directions of steepest descent is using BulkSGD, which at each step projects the gradient to the space orthogonal to the subspace spanned by the top eigenvectors. We try projecting out the top eigenvector, as well as the top three and five eigenvectors \cite{song2024does}. We note that with BulkSGD, we entirely omit the directions of steepest descent, while with Muon we still allow small updates to be made in those directions.

For BulkSGD, training suffers from high levels of instability depending on the number of dominant eigenvectors that are projected out. We observe that training loss is still minimized over 100,000 steps, however the trajectory features steep oscillations. Similarly, sharpness explodes to values in the thousands, which has not been previously observed with other optimizers. This could be due to the fact that, with the dominant eigenvectors projected out, the gradient continues to remain in areas of high curvature, without following the directions of steepest descent.

%% file: sections/appendix/calmo_vs_ce.tex
\section{CalMO: Extended Results}
\label{sec:calmo_vs_ce}

\subsection{Per-Optimizer Training Dynamics}
  \label{sec:CalMO_plots}

  To evaluate out-of-sample calibration, we train ResNet-20 on CIFAR-10 using the full dataset (45K training / 5K validation split). We compare standard cross-entropy loss against CalMO. Each plot shows accuracy,
  ECE, the margin functional $Q(\theta) = \mathbb{E}[e^{-m}]$, and loss. For each optimizer, we report both training metrics (left) and validation metrics (right). Lines show
   mean over 3 seeds; shaded regions indicate $\pm 1$ standard deviation (multiplicative for log-scale plots).

  % SGD
  \begin{figure}[htbp]
      \centering
      \begin{subfigure}[b]{0.48\textwidth}
          \centering
          \includegraphics[width=\textwidth]{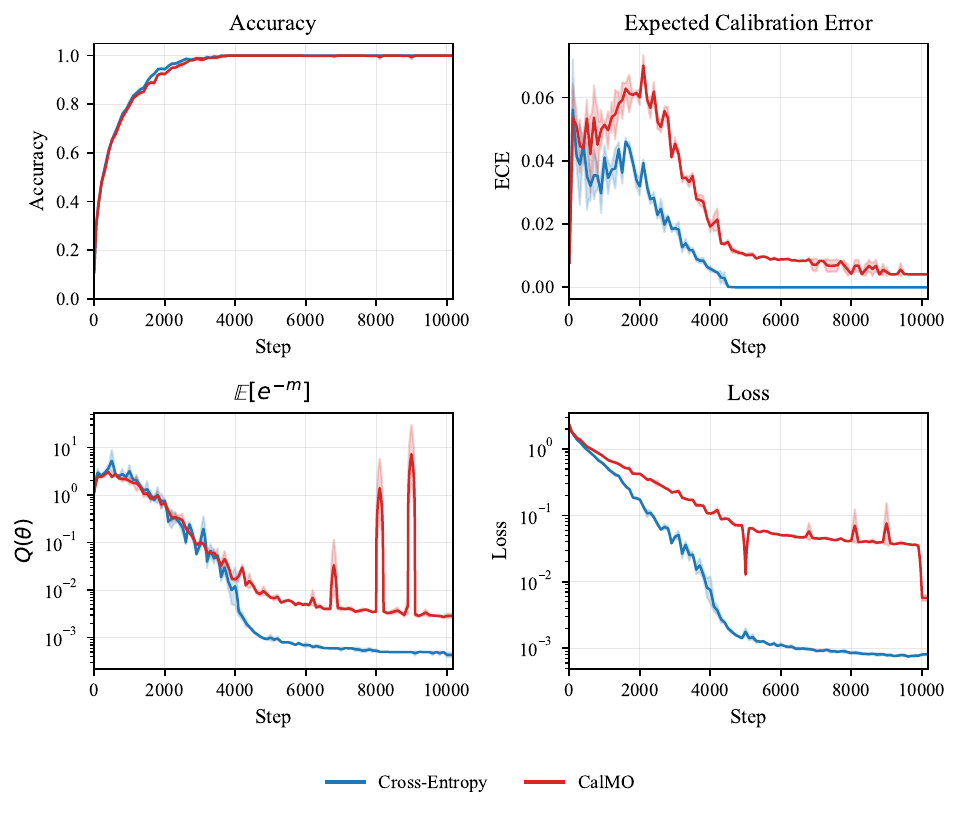}
          \caption{Training metrics}
      \end{subfigure}
      \hfill
      \begin{subfigure}[b]{0.48\textwidth}
          \centering
          \includegraphics[width=\textwidth]{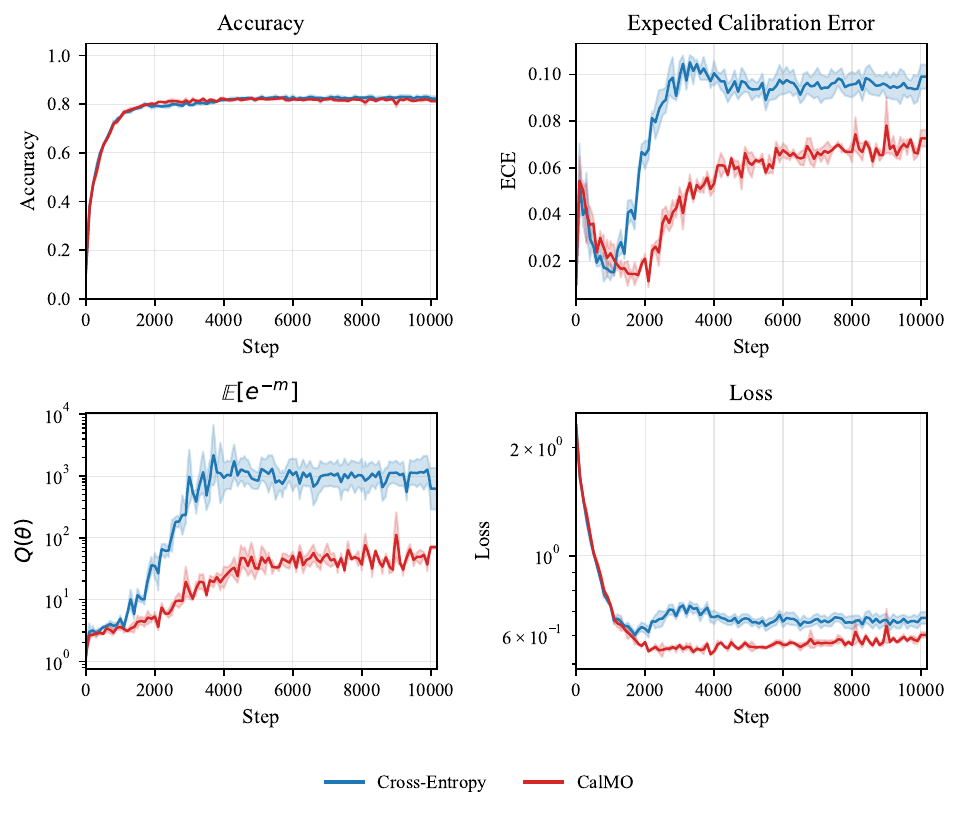}
          \caption{Validation metrics}
      \end{subfigure}
      \caption{\textbf{SGD: CE vs.\ CalMO.} Training dynamics on ResNet-20/CIFAR-10; training (left) and validation (right), mean $\pm 1$ std over 3 seeds.}
      \label{fig:CalMO_sgd}
  \end{figure}

  % AdamW
  \begin{figure}[htbp]
      \centering
      \begin{subfigure}[b]{0.48\textwidth}
          \centering
          \includegraphics[width=\textwidth]{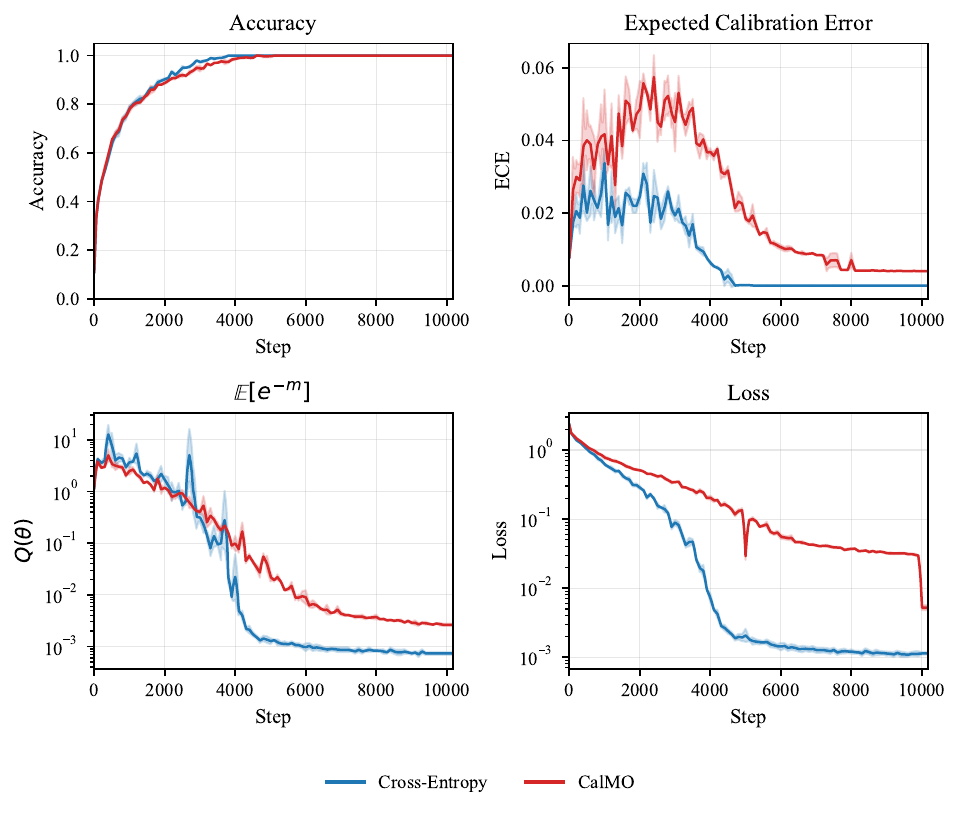}
          \caption{Training metrics}
      \end{subfigure}
      \hfill
      \begin{subfigure}[b]{0.48\textwidth}
          \centering
          \includegraphics[width=\textwidth]{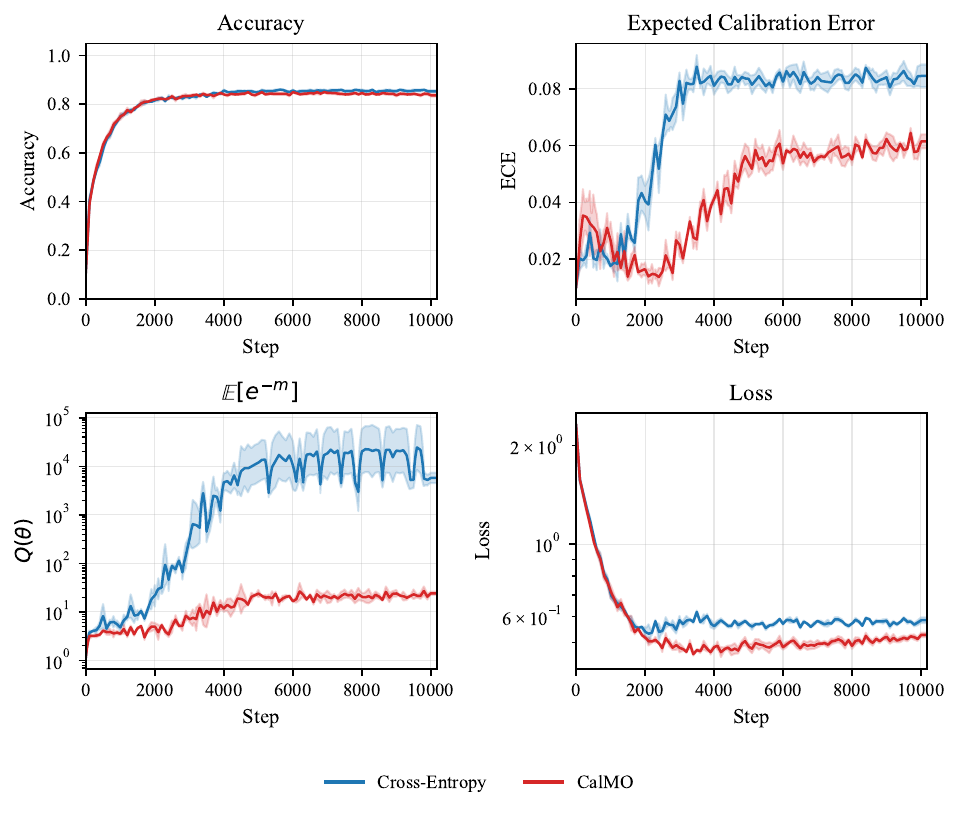}
          \caption{Validation metrics}
      \end{subfigure}
      \caption{\textbf{AdamW: CE vs.\ CalMO.} Training dynamics on ResNet-20/CIFAR-10; training (left) and validation (right), mean $\pm 1$ std over 3 seeds.}
      \label{fig:CalMO_adamw}
  \end{figure}

  % Muon
  \begin{figure}[htbp]
      \centering
      \begin{subfigure}[b]{0.48\textwidth}
          \centering
          \includegraphics[width=\textwidth]{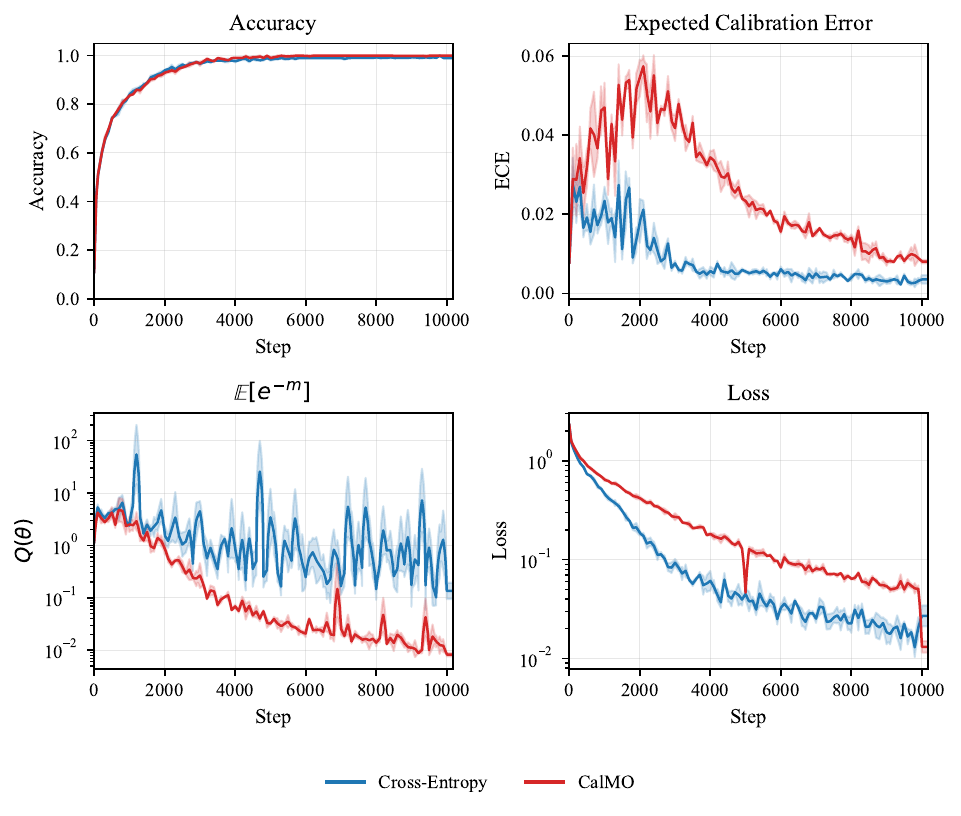}
          \caption{Training metrics}
      \end{subfigure}
      \hfill
      \begin{subfigure}[b]{0.48\textwidth}
          \centering
          \includegraphics[width=\textwidth]{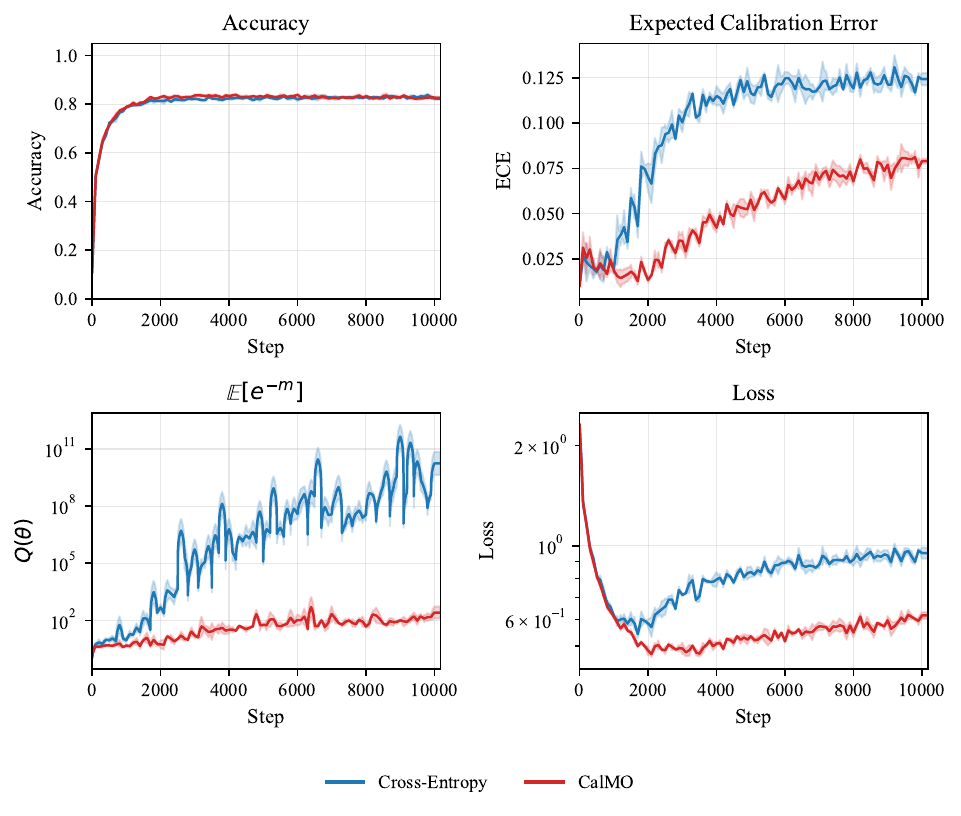}
          \caption{Validation metrics}
      \end{subfigure}
      \caption{\textbf{Muon: CE vs.\ CalMO.} Training dynamics on ResNet-20/CIFAR-10; training (left) and validation (right), mean $\pm 1$ std over 3 seeds.}
      \label{fig:CalMO_muon}
  \end{figure}

    % SAM
  \begin{figure}[htbp]
      \centering
      \begin{subfigure}[b]{0.48\textwidth}
          \centering
          \includegraphics[width=\textwidth]{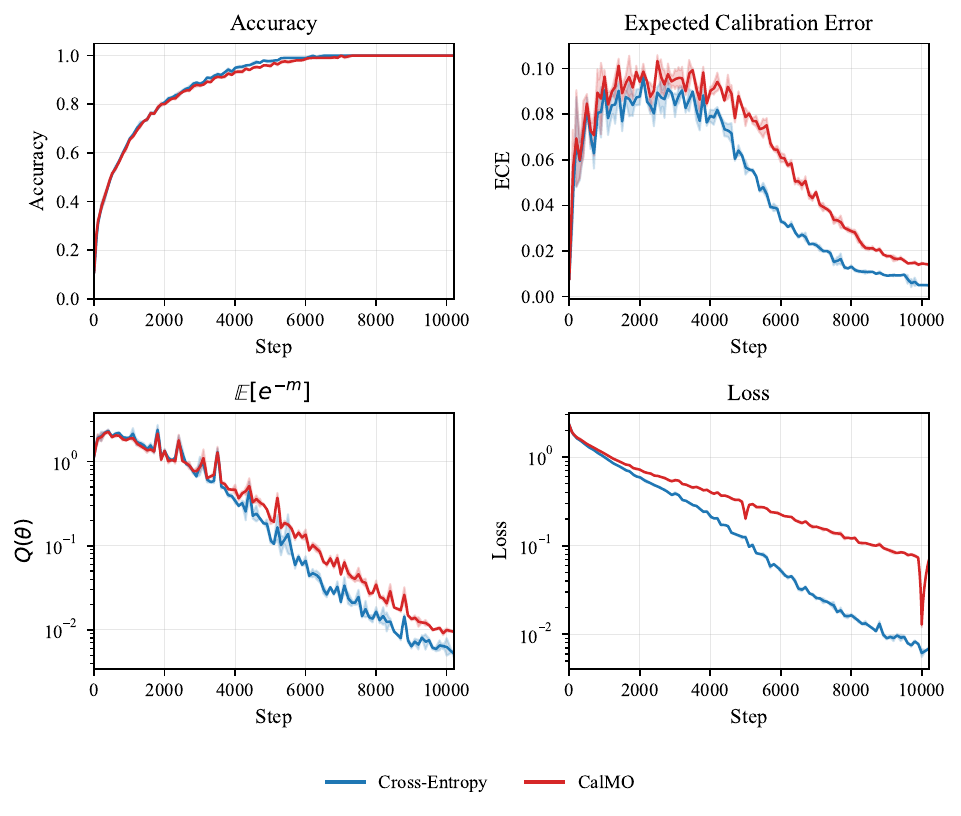}
          \caption{Training metrics}
      \end{subfigure}
      \hfill
      \begin{subfigure}[b]{0.48\textwidth}
          \centering
          \includegraphics[width=\textwidth]{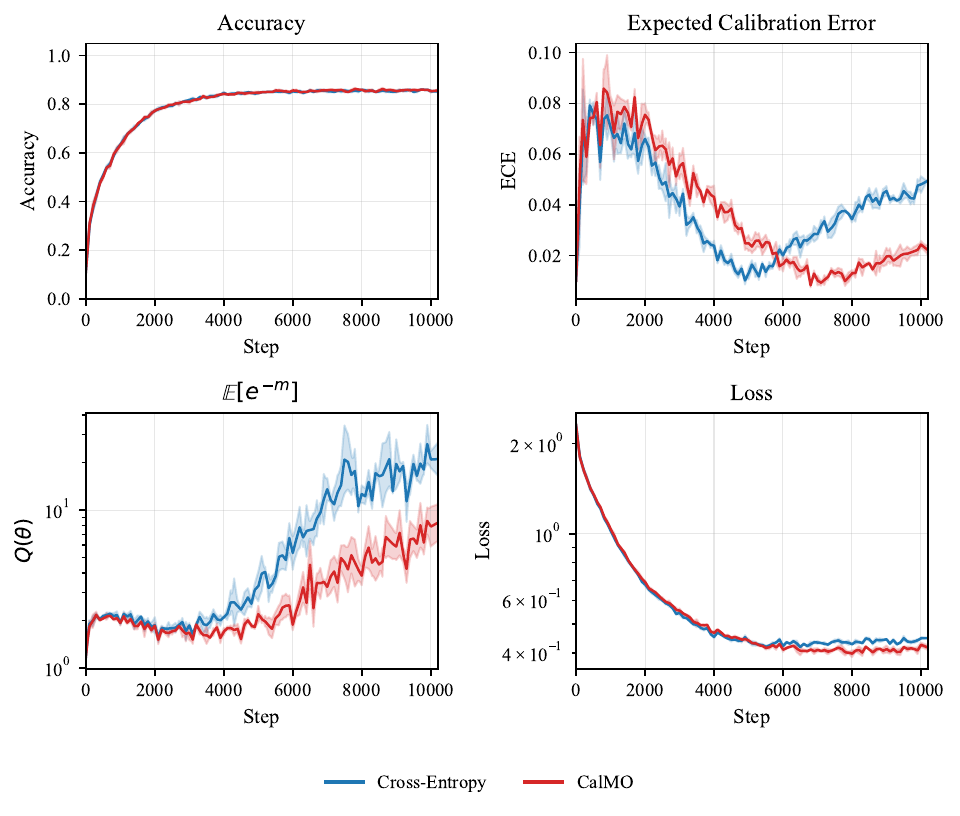}
          \caption{Validation metrics}
      \end{subfigure}
      \caption{\textbf{SAM: CE vs.\ CalMO.} Training dynamics on ResNet-20/CIFAR-10; training (left) and validation (right), mean $\pm 1$ std over 3 seeds.}
      \label{fig:CalMO_sam}
  \end{figure}

\subsection{Train--Test Calibration Gap}
\label{app:train_test_gap}

Table~\ref{tab:train_test_gap} extends Table~\ref{tab:ablation} with training accuracy and training ECE. All methods reach near-perfect training accuracy and near-zero training ECE after interpolation, confirming that the differences observed on the test set reflect generalization of calibration rather than training dynamics.

\begin{table}[htbp]
      \centering
      \footnotesize
      \setlength{\tabcolsep}{4pt}
      \begin{tabular}{@{}llcccc@{}}
        \toprule
        \textbf{Optimizer} & \textbf{Loss} & \textbf{Train Acc (\%)} & \textbf{Train ECE} & \textbf{Test Acc (\%)} $\uparrow$ & \textbf{Test ECE} $\downarrow$ \\
        \midrule
        \multirow{2}{*}{SGD}
           & CE    & $100.0 \pm 0.0$ & $0.000 \pm 0.000$ & $75.2 \pm 0.9$ & $0.081 \pm 0.017$ \\
           & \cellcolor{ourscolor}CalMO & \cellcolor{ourscolor}$100.0 \pm 0.0$ & \cellcolor{ourscolor}$0.004 \pm 0.000$ & \cellcolor{ourscolor}$80.1 \pm 1.1$ & \cellcolor{ourscolor}$0.056 \pm 0.001$ \\
        \midrule
        \multirow{2}{*}{AdamW}
           & CE    & $100.0 \pm 0.0$ & $0.000 \pm 0.000$ & $80.7 \pm 0.3$ & $0.061 \pm 0.005$ \\
           & \cellcolor{ourscolor}CalMO & \cellcolor{ourscolor}$100.0 \pm 0.0$ & \cellcolor{ourscolor}$0.004 \pm 0.000$ & \cellcolor{ourscolor}$83.2 \pm 0.9$ & \cellcolor{ourscolor}$0.045 \pm 0.005$ \\
        \midrule
        \multirow{2}{*}{SAM}
           & CE    & $100.0 \pm 0.0$ & $0.006 \pm 0.002$ & $85.0 \pm 0.2$ & $0.020 \pm 0.004$ \\
           & \cellcolor{ourscolor}CalMO & \cellcolor{ourscolor}$100.0 \pm 0.0$ & \cellcolor{ourscolor}$0.014 \pm 0.001$ & \cellcolor{ourscolor}$85.2 \pm 0.3$ & \cellcolor{ourscolor}$0.017 \pm 0.004$ \\
        \midrule
        \multirow{2}{*}{Muon}
           & CE    & $99.5 \pm 0.0$ & $0.003 \pm 0.001$ & $80.3 \pm 0.3$ & $0.065 \pm 0.013$ \\
           & \cellcolor{ourscolor}CalMO & \cellcolor{ourscolor}$99.9 \pm 0.0$ & \cellcolor{ourscolor}$0.008 \pm 0.000$ & \cellcolor{ourscolor}$81.7 \pm 0.7$ & \cellcolor{ourscolor}$0.019 \pm 0.002$ \\
        \bottomrule
      \end{tabular}
      \vspace{0.2cm}
\caption{\textbf{Train--test calibration gap.} CE vs.\ CalMO on ResNet-20/CIFAR-10; extends Table~\ref{tab:ablation} with training metrics.}
\label{tab:train_test_gap}
\end{table}

  \subsection{Benchmark vs.\ Intrinsic Calibration Methods}

    \begin{table}[htbp]
    \centering
    \footnotesize
    \begin{tabular}{@{}llcc@{}}
    \toprule
    \textbf{Optimizer} & \textbf{Method} & \textbf{Acc (\%)} & \textbf{ECE} $\downarrow$ \\
    \midrule
    \multirow{4}{*}{SGD} & CE & $75.2 \pm 1.2$ & $0.081 \pm 0.021$ \\
     & Label Smooth. & $\mathbf{81.0 \pm 1.0}$ & $0.082 \pm 0.009$ \\
     & Focal Loss & $78.0 \pm 2.9$ & $\mathbf{0.029 \pm 0.018}$ \\
     & CalMO & $80.1 \pm 1.4$ & $0.056 \pm 0.001$ \\
    \midrule
    \multirow{4}{*}{AdamW} & CE & $80.7 \pm 0.4$ & $0.061 \pm 0.007$ \\
     & Label Smooth. & $\mathbf{85.8 \pm 0.1}$ & $0.059 \pm 0.002$ \\
     & Focal Loss & $80.5 \pm 1.3$ & $\mathbf{0.042 \pm 0.019}$ \\
     & CalMO & $83.2 \pm 1.1$ & $0.045 \pm 0.007$ \\
    \midrule
    \multirow{4}{*}{Muon} & CE & $80.3 \pm 0.3$ & $0.065 \pm 0.016$ \\
     & Label Smooth. & $\mathbf{82.4 \pm 0.7}$ & $0.045 \pm 0.003$ \\
     & Focal Loss & $78.7 \pm 0.5$ & $0.025 \pm 0.017$ \\
     & CalMO & $81.7 \pm 0.9$ & $\mathbf{0.019 \pm 0.002}$ \\
    \midrule
    \multirow{4}{*}{SAM} & CE & $85.0 \pm 0.2$ & $0.020 \pm 0.004$ \\
     & Label Smooth. & $\mathbf{86.2 \pm 0.5}$ & $0.067 \pm 0.004$ \\
     & Focal Loss & $84.0 \pm 1.0$ & $0.089 \pm 0.010$ \\
     & CalMO & $85.8 \pm 0.3$ & $\mathbf{0.017 \pm 0.005}$ \\
    \bottomrule
    \end{tabular}
    \vspace{0.2cm}
    \caption{Accuracy and ECE at the best-validation step for CE, label smoothing, focal loss, and CalMO on ResNet-20/CIFAR-10 (90/10 train/validation split).}
    \label{tab:benchmark_best_val}
  \end{table}

Table \ref{tab:benchmark_best_val} compares CalMO to commonly used intrinsic calibration methods—label smoothing and focal loss—across optimizers on CIFAR-10 with ResNet-20 with a 90/10 train/validation split. For each model, we compute accuracy and ECE at the training step that minimizes loss on the validation set. The results highlight that the effectiveness of calibration interventions is optimizer-dependent. For SGD and AdamW, focal loss or label smoothing often achieve the lowest ECE, consistent with prior observations that these methods implicitly regularize confidence. In contrast, for Muon, CalMO yields the largest reduction in ECE while maintaining competitive accuracy. Similarly for SAM, the flatness and robustness terms alone do not control ECE; their combination leads to the lowest miscalibration while maintaining a high accuracy. Across optimizers, CalMO tends to strike a favorable balance between predictive accuracy and calibration error, avoiding the larger accuracy–ECE trade-offs exhibited by some single-mechanism baselines.

%% file: sections/appendix/mse_extension.tex
\section{Extension to Mean Squared Error}
\label{sec:mse_ece}
\subsection{GD and SGD Experiments}
\label{subsec:mse-gd-sgd}

We rerun the experimental setup of Section~\ref{sec:sharpness-calibration} with mean squared error (MSE) loss in place of cross-entropy, on 5000 CIFAR-10 training samples. Figures~\ref{fig:gd-mse} and~\ref{fig:sgd-mse} report the training dynamics for gradient descent and stochastic gradient descent, respectively.

\begin{figure}[htbp]
    \centering
    \includegraphics[width=0.8\linewidth]{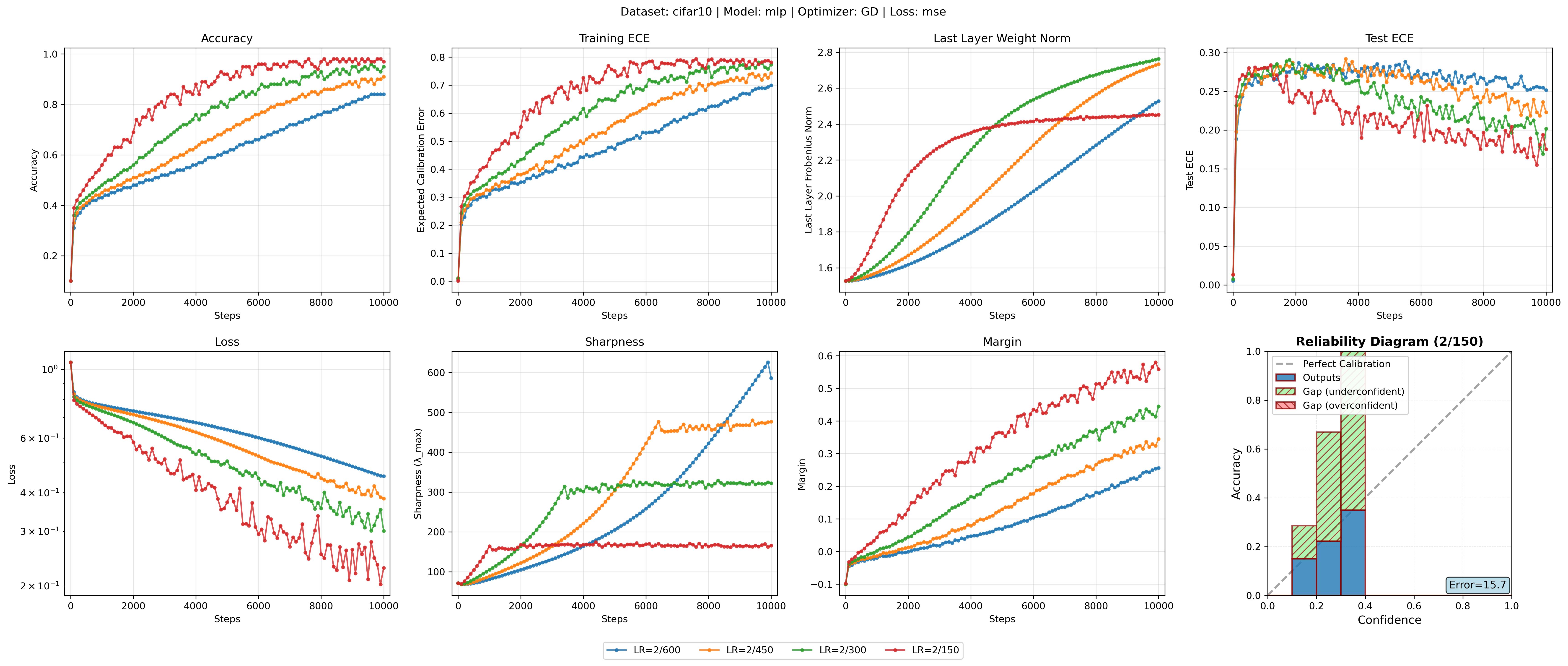}
    \caption{CIFAR-10 | Optimizer: Gradient Descent | Loss: Mean Squared Error}
\label{fig:gd-mse}
\end{figure}

\begin{figure}[htbp]
    \centering
    \includegraphics[width=0.8\linewidth]{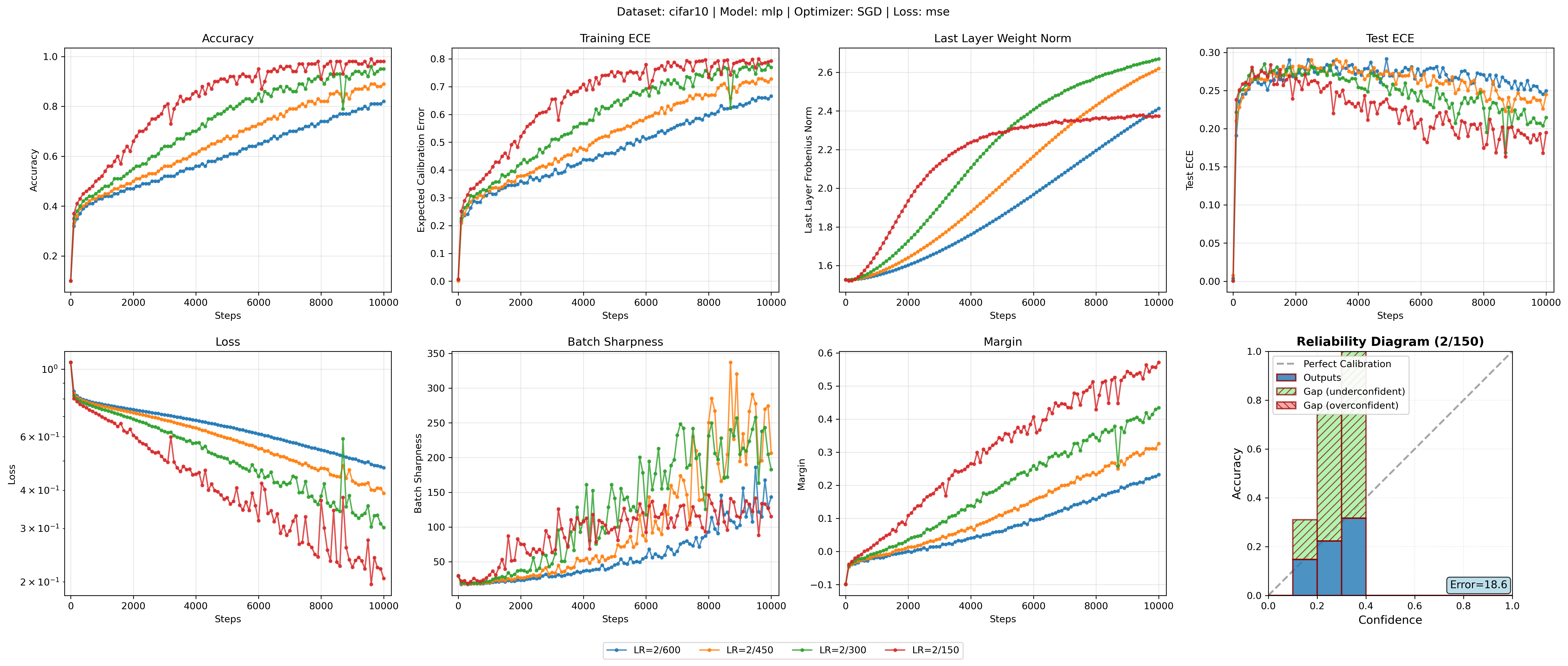}
    \caption{CIFAR-10 | Optimizer: Stochastic Gradient Descent | Loss: Mean Squared Error}
\label{fig:sgd-mse}
\end{figure}

The results show that MSE is extremely miscalibrated, resulting in severely underconfident models as evidenced by the reliability diagrams. This is a consequence of the fact that MSE is not a proper scoring rule.

For $L_{\mathrm{MSE}}$, treating $z_\theta(x_i)$ as free variables, the unique minimizer for a
single example is
\[
z_k^\star(x_i) =
\begin{cases}
1, & k = y_i,\\[3pt]
0, & k \neq y_i.
\end{cases}
\]
Thus the loss penalizes pushing $z_\theta(x_i)_{y_i}$ above $1$ or the other logits below $0$. However, when using the logits in the softmax before the ECE computation, this finite target pattern leads to underconfidence. In fact, at the optimum, the confidence is
\[
\widehat{P}(x_i)
= p_\theta(x_i)_{y_i}
= \frac{e^1}{e^1 + (K-1)e^0}
= \frac{e}{e + K - 1}
\approx 0.23 \quad\text{for } K=10.
\]
Consequently, in regimes where training accuracy is close to $1$ but logits are near this finite pattern, the
model is systematically underconfident on the training set (accuracy $\approx 1$ vs.\ confidence
$\approx 0.23$ in the main bin), and the training ECE remains large instead of decaying towards zero as in
the CE case.

\subsection{Asymptotics of MSE}

For MSE,
\[
H^{\mathrm{MSE}}_{z,i}(\theta)
\;=\; \nabla_{z_i}^2 L_{\mathrm{MSE}}(z_i,y_i)
\;=\; \frac{2}{K} I_K,
\]
so the logit-level Hessian is constant and does not depend on the predicted probabilities $p_\theta(x_i)$.
Hence, unlike CE, the Gauss--Newton curvature does not attenuate as the model improves its fit:
\[
H_{\mathrm{GN}}^{\mathrm{MSE}}(\theta)
= \frac{1}{n} \sum_i J_i(\theta)^\top H^{\mathrm{MSE}}_{z,i} J_i(\theta)
= \frac{2}{K}\cdot \frac{1}{n} \sum_i J_i(\theta)^\top J_i(\theta),
\]
so the eigenvalues of $H_{\mathrm{GN}}^{\mathrm{MSE}}(\theta)$ are governed entirely by the Jacobians $J_i(\theta)$, with no probability-dependent factor $\mathrm{diag}(p) - pp^\top$ to drive curvature toward zero. Combined with the underconfidence analysis of Appendix~\ref{subsec:mse-gd-sgd}, this explains why neither sharpness nor training ECE collapse under MSE in the interpolation regime, in contrast with the CE case.